\documentclass[final,onefignum,onetabnum]{siamonline250106}
\usepackage[compatibility=false]{caption}
\usepackage{subcaption}

\usepackage{lipsum}
\usepackage{xr}
\usepackage{hyperref}
\usepackage{xr-hyper}
\usepackage{cleveref}
\usepackage{amsfonts}
\usepackage{graphicx}
\usepackage{epstopdf}
\usepackage{algorithmic}
\usepackage{bbm}
\usepackage{mathtools}
\usepackage{bm}
\usepackage{comment}
\usepackage{amssymb}
\usepackage[textsize=tiny]{todonotes}
\usepackage[margin=1.2in]{geometry}

\ifpdf
  \DeclareGraphicsExtensions{.eps,.pdf,.png,.jpg}
\else
  \DeclareGraphicsExtensions{.eps}
\fi

\usepackage{enumitem}
\setlist[enumerate]{leftmargin=.5in}
\setlist[itemize]{leftmargin=.5in}

\newsiamremark{remark}{Remark}
\newsiamremark{hypothesis}{Hypothesis}
\crefname{hypothesis}{Hypothesis}{Hypotheses}
\newsiamthm{claim}{Claim}

\newtheorem{assumption}{Assumption}

\newcommand{\R}{\mathbb{R}}
\newcommand{\C}{\mathbb{C}}
\newcommand{\Ncal}{\mathcal{N}} 
\newcommand{\E}{\mathbb{E}} 
\newcommand{\Prob}{\mathbb{P}} 

\DeclareMathOperator{\rank}{rank}
\DeclareMathOperator*{\argmax}{arg\,max} 

\newcommand{\ml}[1]{{\color{cyan}{[ml: #1]}}}

\newcommand{\ymm}[1]{{\color{orange}{[ymm: #1]}}}
\newcommand{\ymtd}[1]{\todo{ymm: #1}}

\newcommand{\ab}[1]{{\color{red}{[ab: #1]}}}

\newcommand{\Thetacan}{\Theta_\text{c}}
\newcommand{\Thetastan}{\Theta_\text{s}}

\headers{Canonical Bayesian Linear System Identification}{A.~Bryutkin, M.~Levine, I.~Urteaga, Y.~Marzouk}

\title{Canonical Bayesian Linear System Identification\thanks{Submitted to the editors on the 08/22/2025.
    \funding{AB and YM are supported in part by the US Department of Energy, SciDAC-5 program, under contract DE-SC0012704, subcontract 425236. YM also acknowledges support from the US Department of Energy ASCR program, under contract DE-SC0023187.  MEL is supported by the Eric and Wendy Schmidt Center at the Broad Institute of MIT and Harvard and Basis Research Institute. IU is supported by ``la Caixa'' foundation’s LCF/BQ/PI22/11910028 award, MICIU/AEI/10.13039/501100011033, and the Basque Government through the BERC 2022-2025 program.}}}

\author{Andrey Bryutkin\thanks{Massachusetts Institute of Technology
    (\email{bryutkin@mit.edu, ymarz@mit.edu}).
  }
  \and Matthew E.~Levine\thanks{Eric and Wendy Schmidt Center, Broad Institute of MIT and Harvard (\email{levinema@mit.edu}) \& Basis Research Institute (\email{matt@basis.ai}).}
    \and I\~{n}igo Urteaga\thanks{BCAM (Basque Center for Applied Mathematics) \& Ikerbasque (Basque Foundation for Science)
    (\email{iurteaga@bcamath.org}, \url{https://iurteaga.github.io/})
}
\and Youssef Marzouk\footnotemark[2]
}

\usepackage{amsopn}

\ifpdf
\hypersetup{
  pdftitle={Canonical Bayesian LTI System Identification},
  pdfauthor={A. Bryutkin, M. Levine, I. Urteaga, Y. Marzouk}
}
\fi



\begin{document}

\maketitle


\begin{abstract}
  Standard Bayesian approaches for linear time-invariant (LTI) system identification are hindered by parameter non-identifiability; the resulting complex, multi-modal posteriors make inference inefficient and impractical. We solve this problem by embedding canonical forms of LTI systems within the Bayesian framework.
  We rigorously establish that inference in these minimal parameterizations fully captures all invariant system dynamics (e.g., transfer functions, eigenvalues, predictive distributions of system outputs) while resolving identifiability. This approach unlocks the use of meaningful, structure-aware priors (e.g., enforcing stability via eigenvalues) and ensures conditions for a Bernstein--von Mises theorem---a link between Bayesian and frequentist
large-sample asymptotics that is broken in standard forms. Extensive simulations with modern MCMC methods highlight advantages over standard parameterizations: canonical forms achieve higher computational efficiency, generate interpretable and well-behaved posteriors, and provide robust uncertainty estimates, particularly from limited data.
%
%
\end{abstract}


\section{Introduction}
\label{sec:intro}

Linear time-invariant (LTI) systems are a foundational class of dynamical systems, used extensively across control engineering, signal processing, econometrics, robotics, and the development of digital twins for complex physical systems \cite{massai2025free, stilgenbauer2024symbolic, duarte2021optimal, wang2023review}. Their defining characteristics are linearity, i.e., obeying the principle of superposition, and time invariance, meaning that their dynamical behavior is consistent over time. State-space models offer a powerful and versatile framework for describing and analyzing LTI systems. The theoretical underpinnings for these models originate in Kalman's seminal contributions \cite{Kalman1963} and have since been subject to extensive development; see \cite{chen1999,ljung1999system} for relevant background.

This work focuses on identifying discrete-time LTI systems from data. Our framework can be adapted to continuous-time systems with minor adjustments, but we concentrate here on the discrete-time case. Specifically, we consider systems described by the following stochastic state-space model:
\begin{equation}
\label{eq:general_lti_state_space_model_noise}
\begin{aligned}
x_{t+1} & =A x_t+B u_t+w_t \\
y_t & =C x_t+D u_t+z_t ,
\end{aligned}
\end{equation}
where $x_t \in \mathbb{R}^{d_x}$ represents the latent state of the system, $u_t \in \mathbb{R}^{d_u}$ is the known input, and $y_t \in \mathbb{R}^{d_y}$ is an (indirect) observation of the system, all at time $t \in \mathbb{N}_0$. The terms $w_t \in \mathbb{R}^{d_x}$ and $z_t \in \mathbb{R}^{d_y}$ form the process and measurement noise sequences, respectively, assumed iid over time. Moreover, for any $t$, we assume that $w_t$ is independent of $\{x_s\}_{s \leq t}$; $z_t$ is independent of $\{x_t\}_{t \in \mathbb{N}_0}$; and $\{w_t\}_{t \in \mathbb{N}_0}$ and $\{z_t\}_{t \in \mathbb{N}_0}$ are mutually independent, 
marginally distributed as $w_t \sim \pi_w$,  $z_t \sim \pi_z$. The state and observation dynamics are governed by the state $A \in \mathbb{R}^{d_x \times d_x}$, input $B \in \mathbb{R}^{d_x \times d_u}$, observation $C \in \mathbb{R}^{d_y \times d_x}$, and feedthrough $D \in \mathbb{R}^{d_y \times d_u}$ matrices.

The central problem addressed herein is inference of the matrices $(A, B, C, D)$ given one (or more) \textit{finite-length} trajectories of inputs $u_t$ and noisy outputs $y_t$. This is a version of the \textit{system identification} problem. Classical system identification tools include prediction error methods \cite{ljung1999system}, frequency-domain techniques \cite{pintelon2012system}, and subspace algorithms for direct state-space estimation \cite{vanOverschee1996subspace}. Statistical performance limits of these approaches are now well understood for both single-input single-output (SISO) and multi-input multi-output (MIMO) settings \cite{gevers2005identification,hjalmarsson2005experiments,wahlberg1986system,juditsky1995nonparametric,zheng2004parameter}. Yet these approaches provide only point estimates, without quantification of uncertainty \cite{campiljung2006}.
Moreover, transfer function models are prevalent in system identification, in large part because they compactly summarize input-output dynamics (relating $\{u_t\}_{t \in \mathbb{N}_0}$ to $\{y_t\}_{t \in \mathbb{N}_0}$), but that very compression can hide structural properties.  Note that the mapping from a state-space realization $(A, B, C, D)$ to a transfer function (even considering only \textit{minimal realizations} defined in \Cref{def:minimal_realization}) is many-to-one: i.e., identifying the transfer function does not specify the state-space model.
%
These issues intensify in high-dimensional MIMO problems, where the practical implementation of system identification is already challenging \cite{rojas2008input,qin2006overview}. 
Other approaches to system identification from input-output data, not limited to linear systems, include kernel-based regularization methods that incorporate system-theoretic insights \cite{Pillonetto2023} and sparsifying priors for dynamic systems \cite{Hirsh2022}, building on concepts from sparse learning \cite{tipping2001sparse, aravkin2012sparse}.

Bayesian inference offers an alternative to these classical identification approaches by incorporating rich specifications of prior information and naturally quantifying uncertainty in the learned system and its dynamics given finite data. As we argue below, Bayesian statistical modeling pairs naturally with state-space models, posing the inference problem directly on a full internal description of the system.
Bayesian approaches to system identification were first proposed in \cite{Peterka1981BayesianID} and have since evolved significantly \cite{NinnessHenriksen2010BayesMCMC,sarkka2013bayesian}. The central idea is to treat model parameters as random variables, whose probability distributions reflect one's state of knowledge about the true model. Newly observed data, in combination with a statistical model, are then used to systematically update these prior distributions into posterior distributions. Särkkä \cite{sarkka2013bayesian} notably bridges classical state estimation with modern Bayesian methods, while \cite{Svensson2016GPSSM} develops efficient Gaussian process methods suitable for potentially nonlinear state-space models. Other Bayesian approaches include particle methods for inferring state-space models from data \cite{Chopin_2013,Kantas_2015, urteaga_2017a, urteaga_2017b}. 
The recent work \cite{Galioto_2020} provides a unifying perspective encompassing various identification approaches, highlighting the flexibility of the Bayesian framework.

While attractive for these reasons, Bayesian inference of state-space models presents certain fundamental difficulties. A central difficulty is the lack of \emph{parameter identifiability}, stemming from the non-uniqueness of the state-space representation. Even for \textit{minimal} systems as described in \Cref{sec:equivalences}, an entire equivalence class of models, related by state transformations,
yield precisely the same input-output behavior. This intrinsic non-uniqueness, reflecting the arbitrary choice of internal state coordinates, makes the estimation of a specific matrix parameter set $(A, B, C, D)$ classically ill-posed.
Difficulties due to non-identifiability persist in the Bayesian framework. As we will demonstrate, non-identifiability yields strongly non-Gaussian and multi-modal posterior distributions that can be extremely difficult to sample.

While strong priors could in principle tame some of this structure,  defining meaningful and well-behaved prior distributions directly on redundant parameters leads to underidentified models \cite[Ch.\ 4]{gelman2013bayesian}. 
%
For instance, specifying priors on the system's eigenvalues (poles) is complicated by the eigenvalues' invariance to state permutations, and leads to additional challenges that we formalize below (see \Cref{remark_prior_comparison_with_matrix_entries}).
More generically, the computational cost of Bayesian inference in high dimensions can be significant, and hence methods for reducing the dimension of the inference problem are practically useful. In this paper we will use \emph{canonical forms} of LTI state-space models to \textit{resolve the non-identifiability problem} 
and simultaneously achieve parameter dimension reduction.


Canonical forms of state-space models resolve the parameter identifiability problem by providing a unique, minimal set of parameters corresponding to each distinct input-output behavior, thereby eliminating the ambiguity of an arbitrary choice of state basis; see \Cref{def:controller_form} for a concrete example in SISO systems and \cite{chen1999} for background. We argue that guaranteeing structural identifiability in this way is also a crucial prerequisite for \textit{specifying unambiguous prior distributions}. Indeed, we will show that priors can be coherently defined on the reduced parameter space of a canonical form to reflect genuine system properties (such as stability), without being confounded by parameter redundancy. The value of using canonical forms within non-Bayesian learning frameworks has been demonstrated by \cite{hardt2018}, which develops guarantees for maximum likelihood estimation via stochastic gradient descent. 
Previous work on Bayesian system identification has not exploited canonical forms.

Beyond the specification of the Bayesian system identification problem, it is useful to understand structural properties of the posterior. In particular, Bernstein--von Mises (BvM) theorems establish asymptotic normality and concentration of the posterior, and show when Bayesian credible sets can also be interpreted as frequentist confidence sets. They also open the door to fast computational approximations, e.g., Laplace approximations \cite{katsevich2024improveddimensiondependencebernstein}, for sufficiently large data sets.
We show in \Cref{sec:fim_bvm} that the posterior distribution over parameters in the canonical form satisfies a BvM and develop explicit formulas for the associated Fisher information.
We will also show that the posterior distribution over a standard parameterization of the state-space model \textit{cannot} satisfy a BvM, again due to the core issue of non-identifiability.
%

Our Bayesian framework provides a principled approach for inference from finite data. This perspective is complementary to the active field of \textit{non-asymptotic statistics} for system identification, which focuses on deriving finite-sample performance guarantees for point estimation procedures~\cite{oymak2019, sarkar2020}, establishing fundamental limits on sample complexity~\cite{bakshi2023newapproachlearninglinear}, and connecting statistical learning with control-theoretic stability~\cite{hu2022}. A valuable direction for future work is to bridge these perspectives by establishing non-asymptotic concentration guarantees for the full posterior distributions produced by our methods. But developing such guarantees almost certainly requires first establishing a BvM theorem, as we do here.


\subsection{Contributions and outline}
\label{sec:contributions_outline} 


The key contributions of this paper are as follows:
\begin{itemize} \itemsep0em 
\item \textbf{\textsf{Structurally identifiable Bayesian LTI parameterizations:}}
 We formulate Bayesian inference of LTI systems using minimal, canonical state-space parameterizations $\Thetacan$. Doing so yields an \textit{identifiable} problem with a typically simple posterior geometry that facilitates efficient computation.  Foundational details on equivalence classes, canonical forms, and identifiability proofs are given in \Cref{sec:equivalences}.

 
\item \textsf{\textbf{Principled stability-preserving priors:}} Leveraging the canonical representation, we design informative priors that enforce system stability and, if desired, shape other dynamical features of the system. This involves specifying distributions on intrinsic properties (stable eigenvalues) and
  mapping them onto $\Thetacan$ via
  a well-behaved transformation,
  overcoming the difficulties of imposing prior information on the redundant, standard state-space parameterization $\Thetastan$. The methodology is detailed in \Cref{sec:priors}.

\item \textsf{\textbf{Bayesian posterior guarantees:}} We characterize statistical properties of the canonical approach by deriving the Fisher information matrix for $\Thetacan$ and proving that a Bernstein--von Mises theorem holds under standard conditions. This establishes asymptotic posterior normality and efficiency, in contrast with non-identifiable parameterizations. This theoretical analysis is presented in \Cref{sec:fim_bvm}.

    \item \textsf{\textbf{Empirical evaluations:}} We provide comprehensive numerical comparisons between Bayesian inference in standard and canonical LTI representations. We evaluate posterior geometry and computational performance, the accuracy of posterior estimates, and the information provided by posterior predictives---for datasets of varying size and for systems with different dynamical properties---showcasing the practical advantages of the canonical approach. We also compare with point estimates provided by standard system identification techniques. These results are presented in \Cref{sec:numerics}.
      
    \end{itemize}
    
    \Cref{sec:bayesian_inference_lti_systems_standard_perspective} of the paper lays out the general problem framework. Subsequent sections contain our main results as outlined above. \Cref{sec:conclusion} summarizes our findings and outlines future research.

\section{Bayesian inference of LTI systems: standard perspective}
\label{sec:bayesian_inference_lti_systems_standard_perspective}
This section formulates a Bayesian framework for identifying the parameters of LTI systems using the standard state-space representation. We begin by defining notation and the probabilistic model underpinning the inference task.

\paragraph{Notation and conventions}
For a vector $v$ and a symmetric positive definite matrix $C$, we define the squared norm with respect to $C$ as $\|v\|_C^2 \coloneqq v^\top C^{-1} v$. Throughout this work, we denote a trajectory of a sequence $\{z_t\}_{t\in\mathbb{N}_0}$ from time $t=a$ to $t=b$ as $z_{[a:b]} \coloneqq \{z_a, z_{a+1}, \ldots, z_b\}$. When context implies the full observed trajectory up to time $T$, we may use the shorthand $z_{[T]} \coloneqq z_{[0:T]}$.
Parameter collections will be referred to as $\Theta \in \bm{\Theta}$; subscripts will be used when it is important to distinguish standard $\Theta_s$ and canonical $\Theta_c$ parameterizations.
A realization of the LTI system matrices due to a specific parametrization is denoted as $\mathcal{L}_\Theta$.
We will generally use $p$ to denote a probability density function; to avoid ambiguity, different densities will be identified by explicitly specifying their arguments, as in $p(z_{[T]}|\Theta)$, thereby making clear the underlying random variables and parameters.

\begin{assumption}[Gaussian noise and initial state, known deterministic input]
\label{ass:gaussian_noise}
We assume Gaussian distributions for the stochastic components of the LTI system in \eqref{eq:general_lti_state_space_model_noise}:
process noise \(w_t \sim \Ncal(0, \Sigma)\), measurement noise \(z_t \sim \Ncal(0, \Gamma)\), and initial state \(x_0 \sim \Ncal(0, P_0)\). The random variables \(x_0\), \(w_{[T]}\), and \(z_{[T]}\) are assumed to be mutually independent. The input sequence \(u_{[T]}\) is treated as known and deterministic.
\end{assumption}
In the ``\textit{standard perspective}'' addressed in this section, we let all entries of the LTI system matrices $A, B, C, D$ be treated as unknown parameters to be inferred, without imposing any specific structural constraints. We denote the complete set of standard parameters as $\Theta_{\text{s}}:= \{A, B, C, D, P_0, \Sigma, \Gamma\}$. This parameter set corresponds directly to the matrices and covariance structures appearing in \eqref{eq:general_lti_state_space_model_noise},
which as discussed in \Cref{sec:intro},
suffers from identifiability issues, as different sets $\Theta_{\text{s}}$ can produce identical LTI input-output behavior. 

We now specify key properties of LTI systems that are crucial for system identification.

\begin{definition}[Stability]
\label{def_stable_system}
An LTI system represented by parameters \(\Theta\) is \emph{stable} if all eigenvalues \(\lambda_i(A)\) of the state matrix \(A\) lie strictly within the unit disk in the complex plane, i.e., $\max_i |\lambda_i(A)| < 1.$ We say that the system is marginally stable if the eigenvalues can also lie on the boundary of the unit circle, i.e., $\max_i |\lambda_i(A)| \leq 1.$
\end{definition}

\begin{definition}[Controllability and observability]
\label{def:controllability_observability}
Consider the state-space system characterized by the matrices \(A \in \R^{d_x \times d_x}\), \(B \in \R^{d_x \times d_u}\), and \(C \in \R^{d_y \times d_x}\).
The pair \((A,B)\) is called \emph{controllable} if the \emph{controllability matrix}
\begin{equation}
\mathcal{C}
\;=\;
\bigl[
\,B \quad A\,B \quad A^2 B \quad \dots \quad A^{d_x-1}B
\bigr] \in \mathbb{R}^{d_x \times (d_x d_u)}
\end{equation}
has full row rank, i.e., \(\rank(\mathcal{C}) = d_x\).
Likewise, the pair \((A,C)\) is called \emph{observable} if the \emph{observability matrix}
\begin{equation}
\left.\mathcal{O}=\left[\begin{array}{llll}C^{\top} & (C A)^{\top} & \left(C A^2\right)^{\top} & \cdots\left(C A^{d_x-1}\right.\end{array}\right)^{\top}\right]^{\top}\in\mathbb{R}^{d_x\times(d_xd_y)}
\end{equation}
has full column rank, i.e., \(\rank(\mathcal{O}) = d_x\).
\end{definition}

\begin{definition}[Minimal realization]
\label{def:minimal_realization}
A state-space realization $\mathcal{L}_\Theta$ of an LTI system, such as the one described in \cref{eq:general_lti_state_space_model_noise}, is called \emph{minimal} if it is both controllable and observable according to \Cref{def:controllability_observability}.
\end{definition}
For a given input-output behavior, a minimal realization uses the smallest possible state dimension \(d_x\).

\subsection{LTI system identification as Bayesian inference} Given a sequence of inputs denoted as $u_{[T]}$ and corresponding observed outputs $y_{[T]}$, the goal of Bayesian system identification is to characterize the posterior probability distribution over the parameters $\Theta$. This inference is based on the likelihood of observing the data given the parameters, denoted as $p(y_{[T]} \mid \mathcal{L}_\Theta, u_{[T]})$. This likelihood function implicitly involves marginalizing out the unknown latent state trajectory $x_{[T]}$. Under the Gaussian assumptions, this marginalization can be performed analytically, via the Kalman filter. The joint probability density of the observations, representing the likelihood function, is given by the product of conditional densities:
\begin{equation}
\label{eq:joint_likelihood}
p\left(y_{[T]} \,\middle|\, \mathcal{L}_\Theta, u_{[T]}\right) = p\left(y_0 \,\middle|\, \mathcal{L}_\Theta, u_0\right) \prod_{t=1}^{T} p\left( y_t \,\middle|\, y_{[t-1]}, \mathcal{L}_\Theta, u_{[t]}\right) \,  ,
\end{equation}
where each term $p(y_t \mid \dots)$ corresponds to the predictive distribution of the observation $y_t$ given all past information, computable via Kalman update formulas. The explicit derivation and decomposition of this likelihood are detailed in \Cref{appendix_likelihood_decomposition} (see \Cref{thm:marg_lik}) both for the noiseless ($\Sigma = 0$) and noisy ($\Sigma \neq 0$) state cases. 

Bayesian inference combines the likelihood in \eqref{eq:joint_likelihood} with a prior density $p(\Theta)$, which encodes any prior belief or knowledge about the system parameters before observing the data. Applying Bayes' theorem yields the posterior probability density,
\begin{equation}
\label{eq:bayes_theorem}
p(\Theta \mid y_{[T]}, u_{[T]}) \propto p\left(y_{[T]} \,\middle|\, \mathcal{L}_\Theta, u_{[T]}\right) p(\Theta),
\end{equation}
which represents the updated state of knowledge about the system parameters after accounting for the observed input-output data.

As discussed in \Cref{sec:intro}, the posterior in the standard parameterization  \(p(\Theta_s \mid y_{[T]}, u_{[T]})\) typically has a complex geometry, with multimodality and strong correlations induced by the lack of identifiability. We will visualize this geometry later in \Cref{sec:post_geom_interp}, when comparing the posterior in standard and canonical forms, but for now we set aside such concerns and discuss uses of the posterior distribution in system identification. 

Beyond inferring the primary parameters \(\Theta \in \mathbf{\Theta}\subset\mathbb{R}^d\), where d is the number of parameters,
we are often interested in derived quantities of interest (QoIs) that characterize the system's behavior or structure, e.g., transfer functions or filtering and smoothing distributions. We describe below several derived quantities that are of particular relevance in linear system analysis and identification.
\begin{definition}[Eigenvalue map \(S^{\Lambda}\)]
\label{def:eigenvalue_pushforward}
The eigenvalues of the state matrix \(A\) determine the system's stability and dynamic modes. The eigenvalue map \(S^{\Lambda}: \mathbf{\Theta} \rightarrow \mathcal{P}(\C^{d_x})\) (where \(\mathcal{P}\) denotes the space of multisets, accounting for multiplicity) is defined implicitly by extracting the roots of the characteristic polynomial of \(A\):
\[
S^{\Lambda}(\Theta) = \{\lambda \in \C \mid \det(A - \lambda I_{d_x}) = 0\}\quad \text{where } A \in \mathcal{L}_\Theta.
\]
\end{definition}

\begin{definition}[Hankel matrix map \(S^{H}\)]
\label{def:hankel_pushforward}
The Hankel matrix encapsulates the system's impulse response and is fundamental to subspace identification methods. Given integers \(p, q \ge 1\), the Hankel matrix map \(S^{H}: \mathbf{\Theta} \rightarrow \R^{(p \cdot d_y) \times (q \cdot d_u)}\) constructs a block Hankel matrix from the system's Markov parameters \(M_t = CA^{t-1}B\) (for \(t \geq 1\)) and \(M_0=D\):
\[
S^{H}(\Theta) =
\begin{bmatrix}
M_1 & M_2 & \cdots & M_q \\
M_2 & M_3 & \cdots & M_{q+1} \\
\vdots & \vdots & \ddots & \vdots \\
M_p & M_{p+1} & \cdots & M_{p+q-1}
\end{bmatrix}
\quad \text{where } M_t = M_t(\Theta) \text{ and } A,B,C\in\mathcal{L}_\Theta .
\]
\end{definition}

\begin{definition}[Transfer function map \(S^{G}\)]
\label{def:transfer_function_pushforward}
The transfer function map \(S^{G}: \mathbf{\Theta} \rightarrow \mathbb{R}[z]^{d_y \times d_u}\) describes the system's input-output relationship, where \(\mathbb{R}[z]\) is the field of rational functions with real coefficients. For a parameter set \(\Theta\), the transfer function \(G(z) = S^{G}(\Theta)(z)\) is defined at values \(z \in \mathbb{C}\) for which the inverse exists as:
\[
G(z) = D + C(zI_{d_x} - A)^{-1}B \quad \text{where } A,B,C,D \in \mathcal{L}_\Theta.
\]
\end{definition}

\begin{definition}[Filtering and smoothing distributions]
\label{def:filtering_smoothing}
State estimation in the Bayesian setting amounts to characterizing filtering distributions, i.e.,  $p(x_t \mid y_{[t]}, u_{[t]}, \mathcal{L}_{\Theta})$ at any time $t$, or smoothing distributions, i.e., $p(x_{[T]} \mid y_{[T]}, u_{[T]}, \mathcal{L}_\Theta)$ or any marginal thereof. In the linear-Gaussian case, conditioned on parameters $\Theta$, these state distributions are Gaussian with means and covariances computable through Kalman filtering and smoothing recursions. Propagating posterior uncertainty in $\Theta$ through these algorithms yields marginal posterior state distributions, e.g., $p(x_{[T]} \mid y_{[T]}, u_{[T]}) = \int p(x_{[T]} \mid y_{[T]}, u_{[T]}, \mathcal{L}_{\Theta}) \, p(\Theta \mid y_{[T]}, u_{[T]}) \, d\Theta$, that account for parameter uncertainty.
\end{definition}

In general, we can apply the previously defined maps to establish a connection between the posterior over the parameters $\Theta$ and the posterior over the derived quantities. For deterministic maps (e.g., \Cref{def:eigenvalue_pushforward,def:hankel_pushforward,def:transfer_function_pushforward}), we simply consider the \textit{pushforward} of the posterior distribution by the measurable map $S^Q$, where $Q$ denotes the specific quantity of interest (e.g., $Q \in \{\Lambda, G, H \}$). We write this pushforward distribution as $P(Q \mid y_{[T]}, u_{[T]}) = S^Q_\# P(\Theta \mid y_{[T]}, u_{[T]})$.
The maps $S$ are in general not injective---consider that different parameters \(\Theta\) can yield the same transfer function or eigenvalues---and hence evaluating the \textit{density} of the pushforward measure at some point $Q=q$, i.e., $p(q \mid y_{[T]}, u_{[T]})$, requires integrating over the pre-image of $q$ under $S^Q$, $\{\Theta : S^Q(\Theta) = q\}$. This is in addition to the usual calculations involved in a change of variables, namely evaluating the Jacobian determinant of $S^Q$.

In practice, one may not be particularly interested in the analytical density  $p(q \mid y_{[T]}, u_{[T]})$. Instead, we can represent the posterior predictive of any quantity of interest empirically, by drawing samples \( \Theta^{(i)} \) from the parameter posterior \(p(\Theta \mid y_{[T]}, u_{[T]})\) (e.g., via MCMC) and evaluating the map $S$ on these samples, $Q^{(i)} = S^Q(\Theta^{(i)})$. The same is true for samples from the filtering and smoothing distributions, which are not deterministic transformations of $\Theta$. Here one can simply draw samples in two stages, e.g., $\Theta^{(i)} \sim p(\Theta \mid y_{[T]}, u_{[T]})$, then $x_{[T]}^{(i)} \sim p(x_{[T]} \mid y_{[T]}, u_{[T]}, \mathcal{L}_{\Theta^{(i)}}) $.

\section{Bayesian inference with canonical LTI parameterizations}
\label{sec:equivalences}

In this section, we introduce canonical forms of LTI systems and describe how these forms yield structural identifiability in the Bayesian framework.
We will perform inference directly on the parameters of the canonical form, denoted as $\Theta_{\text{c}}$, thus obtaining the posterior distribution \(p(\Theta_{\text{c}} \mid y_{[T]}, u_{[T]})\).

Canonical forms provide a unique, minimal (recall \Cref{def:minimal_realization}) set of parameters for each distinct input-output behavior \cite{chen1999}. 
Below we will show that for any LTI system specified by standard parameter values $\Theta_s$, there is a corresponding canonical form with parameters $\Theta_c$ (to be defined precisely below) that achieves
\begin{equation}
\label{eq:likelihood_preservation}
p\bigl(y_{[T]} \,\big\vert\, \mathcal{L}_{\Theta_s}, u_{[T]}\bigr)
= 
p\bigl(y_{[T]} \,\big\vert\, \mathcal{L}_{\Theta_{c}}, u_{[T]}\bigr),
\end{equation}
for all $y_{[T]}$ and $u_{[T]}$. This canonical form is in fact one element of an infinite equivalence class of minimal systems that preserve the likelihood function, i.e., the probabilistic input-output behavior of the LTI system.

We will also show that 
one can construct (non-unique) reverse mappings $\Psi: \bm{\Theta}_c \rightarrow \bm{\Theta}_s$ that embeds the reduced set of canonical parameters into the full standard parameter space, within the equivalence class defined by $\Theta_c \in \bm{\Theta}_c$.
However, the posterior distribution over the standard parameters $p(\Theta_s \mid y_{[T]},  u_{[T]} )$ cannot be fully recovered from the posterior over the canonical parameters $p(\Theta_c \mid y_{[T]},  u_{[T]} )$ via these mappings, due to the non-compactness of the equivalence class.
Nevertheless, for most practical purposes, recovering the posterior \(p(\Theta_s \mid y_{[T]}, u_{[T]})\) is unnecessary:
as we detail in \Cref{equivalence_of_pushforwards},
the posterior distributions over most quantities of interest (e.g., \Cref{def:eigenvalue_pushforward,def:hankel_pushforward,def:transfer_function_pushforward,def:filtering_smoothing}) can be computed directly from the canonical posterior $p(\Theta_{\text{c}} \mid y_{[T]}, u_{[T]} )$.

\subsection{Canonical LTI system parameterizations}
\label{sec:defs}



Recall our generic LTI system \eqref{eq:general_lti_state_space_model_noise} and the statistical setting in \Cref{ass:gaussian_noise}. For single-input (\(d_u=1\)) and single output (\(d_y=1\)) (SISO) systems, a well-known canonical form is the \textit{controller canonical form}.

\begin{definition}[SISO controller form]
\label{def:controller_form}
Consider a controllable SISO (\(d_u=d_y=1\)) system. The controller canonical form parameterizes the system dynamics as \(\{A_c, B_c, C_c, D_c\}\) with:
\begin{equation}
\label{eq:controller_form}
A_{c}=\begin{pmatrix}
0 & 1 & 0 & \cdots & 0 \\
0 & 0 & 1 & \cdots & 0 \\
\vdots & \vdots & \vdots & \ddots & \vdots \\
0 & 0 & 0 & \cdots & 1 \\
-a_0 & -a_1 & -a_2 & \cdots & -a_{d_x-1}
\end{pmatrix},\;
B_{c}=\begin{pmatrix}
0 \\ 0 \\ \vdots \\ 0 \\ 1
\end{pmatrix},\;
C_{c}=\begin{pmatrix} b_0 & \cdots & b_{d_x-1} \end{pmatrix},\;
D_{c}=d_0.
\end{equation}
In this setting, the canonical parameters are $ \Theta_{c} \coloneqq \{a_0, \dots, a_{d_x-1}, b_0, \dots, b_{d_x-1}, d_0 \}$. 
\end{definition}
Note that $(a_{d_x-1}, \ldots, a_0)$ are coefficients of the characteristic polynomial of $A_c$; we exploit this fact in \Cref{subsec:priors_A} when constructing priors over eigenvalues of a state-transition matrix, which can be transformed into priors over its characteristic polynomial coefficients by \Cref{prop:vieta_formulas}.
%
There are seven other SISO canonical forms, including the observer canonical form, see \Cref{appendix_a_canonical_forms} for details. For minimal (controllable and observable) systems, these forms are equivalent up to a state-space transformation: they all have the same minimal number of parameters, required to uniquely specify the input-output dynamics.

Generalizing canonical forms to multi-input, multi-output (MIMO) systems is more complex. One common structure is the MIMO block controller form \cite{Kailath1980LinearSystems}. For simplicity, we focus on SISO canonical forms in the main paper and defer a discussion of MIMO canonical structure to \Cref{appendix_mimo_case}; see also \Cref{remark:mimo_complexity}.

\paragraph{Parameterization complexities }
Here we discuss the size and scaling of different LTI parameterizations.
The standard approach to Bayesian LTI system identification involves inferring all entries of the system matrices \(A_s, B_s, C_s, D_s\). This parameterization involves a total of
$N_s = d_x^2 + d_x d_u + d_x d_y + d_u d_y$ free
parameters. The dominant term \(d_x^2\) makes the parameter space dimension grow quadratically with the state dimension, i.e., \(N_s = \mathcal{O}(d_x^2)\). As discussed previously, this standard form \(\Theta_s\) suffers from non-identifiability.
On the contrary, the canonical parameterization given in \Cref{def:controller_form}, like all other SISO canonical forms, involves \(N_c = 2d_x+1 \) free parameters.

Regardless of whether a standard \(\Theta_s\) or canonical \(\Theta_c\) parameterization is used for the system matrices \(\{A, B, C, D\}\), the covariance matrices \(\{P_0, \Sigma, \Gamma\}\) must also be parameterized. Using a Cholesky factorization (e.g., \(P_0 = L L^\top\), where \(L\) is lower triangular) is standard practice to ensure symmetric positive definiteness and remove redundancy \cite{Galioto_2020}. This replaces the \(d^2\) entries of a \(d \times d\) matrix with the \(d(d+1)/2\) entries of its Cholesky factor. Thus, the number of free parameters for the noise components \(\{P_0, \Sigma, \Gamma\}\) becomes $N_{P_0} \coloneqq  \frac{d_x(d_x+1)}{2}$, \(N_\Sigma \coloneqq \frac{d_x(d_x+1)}{2}\), \(N_\Gamma \coloneqq  \frac{d_y(d_y+1)}{2}\).

\begin{remark}[Complexity and minimality of MIMO canonical forms]
  \label{remark:mimo_complexity}
Canonical forms and the notion of a minimal realization are significantly more complex in the MIMO setting than in the SISO setting. The canonical structure depends on the observability and controllability indices, and many different forms exist (e.g., echelon form, Hermite form). The form presented in \Cref{def:mimo_canonical_form} is one example structure; its minimality and parameter count depend on the system's specific Kronecker indices. Unlike the SISO case, where \(2d_x+1\) is generically minimal, MIMO parameter counts vary. Determining the true minimal dimension and structure often requires analyzing the algebraic properties of the specific system (e.g., rank structure of its Hankel matrix). Often, an overparameterized but structurally identifiable form is used in practice.
  
\end{remark}

\subsection{Representative completeness of canonical forms}
\label{sec:representative_completeness} 

Now we prove that canonical forms provide a complete representation of the identifiable aspects of LTI systems in a Bayesian setting. Specifically, we show that parameterizations related by similarity transformations are {statistically indistinguishable}, in the sense of producing identical values of likelihood for any given input and output sequences, and that inference performed on a canonical form \(\Theta_c\) is sufficient to recover the posterior distributions of all system properties that are invariant under canonical transformations.

The cornerstone of this equivalence is the concept of \textit{statistical isomorphism} \cite{stat_iso_10.1214/aoms/1177699610}:
informally, two statistical models---each being a 
set of possible probability distributions over data---are considered isomorphic if they are equivalent for all inferential purposes. That is, for any statistical decision problem, the optimal performance one can achieve is identical for both systems.
In the context of LTI models, this means that two parameterizations are statistically isomorphic if they yield identical likelihoods for any observed output sequence given any input sequence.

The relationship between minimal LTI systems that are statistically isomorphic is formally characterized by the following theorem, whose proof is in \Cref{appendix_proof_thm_isomorphism_general}.

\begin{theorem}[Isomorphism of minimal LTI systems]
\label{thm:isomorphism_general} 
Let \(\mathcal{L}_{\Theta_s}\) and $\mathcal{L}_{\Theta_{s'}}$ represent two discrete-time LTI systems, described by parameter sets $\Theta_s = (A_s, B_s, C_s, D_s, P_{0s}, \Sigma_s, \Gamma_s)$ and $\Theta_{s'} = (A_{s'}, B_{s'}, C_{s'}, D_{s'}, P_{0s'}, \Sigma_{s'}, \Gamma_{s'})$, respectively, both with state dimension \(d_x\). 
We say \(\mathcal{L}(\Theta_s)\) and \(\mathcal{L}(\Theta_{s'})\) are \emph{statistically isomorphic} if, for any input sequence \(u_{[T]}\), the two systems induce \emph{identical output distributions}, namely,
\begin{equation}
\label{eq:likelihood_equality_iso}
p\bigl(y_{[T]} \,\mid\, \mathcal{L}_{\Theta_s},\;u_{[T]}\bigr)
\;=\;
p\bigl(y_{[T]} \,\mid\, \mathcal{L}_{\Theta_{s'}},\;u_{[T]}\bigr).
\end{equation}
If \(\mathcal{L}(\Theta_s)\) and \(\mathcal{L}(\Theta_{s'})\) are \emph{statistically isomorphic} and correspond to \emph{minimal} realizations, then there exists an invertible matrix \(T_c \in \operatorname{GL}(d_x)\) such that the parameters 
are related by: 
\begin{align}
\label{eq:similarity_transformation}
A_{s'} &= T_c^{-1}\,A_s\,T_c,
&\qquad B_{s'} &= T_c^{-1}\,B_s,
&\quad C_{s'} &= C_s\,T_c,
&\quad D_{s'} &= D_s,\\
\label{covariances}
P_{0s'} &= T_c^{-1}\,P_{0s} \, T_c^{-\top},
&\qquad \Sigma_{s'} &= T_c^{-1}\,\Sigma_s  \, T_c^{-\top}
&\quad \Gamma_{s'} &= \Gamma_s.
\end{align} 
Conversely, if two minimal systems are related by such a transformation \(T_c\), then they are \emph{statistically isomorphic}.
\end{theorem} 

\begin{remark}[Deterministic case as a specialization]
\label{remark_deterministic_case}
When the process and measurement noises vanish (\(w_t = 0, z_t = 0\)), the covariances \(\Sigma\) and \(\Gamma\) become zero matrices (as does \(P_0\) if \(x_0=0\)). Condition \eqref{eq:likelihood_equality_iso} then reduces to matching deterministic output trajectories \(y_{[T]}\) for any input \(u_{[T]}\). In this case, \eqref{covariances} becomes trivial, and the dynamic matrices must satisfy the classical deterministic similarity transformation relationships in \eqref{eq:similarity_transformation}.
\end{remark}

\Cref{thm:isomorphism_general} establishes that the equivalence class of all minimal systems producing the same likelihood is precisely the orbit under the group of similarity transformations \(T \in \operatorname{GL}(d_x)\).
It implies that all minimal systems sharing the same input-output behavior belong to an equivalence class defined by similarity transformations.

A canonical form of the system is an element of this class. In fact, canonical forms are designed precisely to provide a unique, identifiable representative for each such equivalence class.
We can explicitly write the transformation that produces a specific canonical form, i.e., the controller form defined previously in \Cref{def:controller_form}, from any other parameterization.
\begin{proposition}[Companion-form canonical realization]
\label{thm:companion_form_invertible_matrix}
Let \(\mathcal{L}_{\Theta_s}\) define a realization of a discrete-time LTI system with state dimension \(d_x\).
Assume the pair \((A_s, B_s)\) is controllable.
There exists an invertible matrix \(T_c \in \operatorname{GL}(d_x)\) such that the transformed system \(\mathcal{L}_{\Theta_c}\) (via \eqref{eq:similarity_transformation}) takes the \emph{SISO controller canonical form (\Cref{def:controller_form})}, and is given by:
\begin{equation}
T_c^{-1}=\left(\begin{array}{ccccc}
1 & 0 & 0 & \cdots & 0 \\
a_{d_x-1} & 1 & 0 & \cdots & 0 \\
a_{d_x-2} & a_{d_x-1} & 1 & \cdots & 0 \\
\vdots & \vdots & \vdots & \ddots & \vdots \\
a_1 & a_2 & a_3 & \cdots & 1
\end{array}\right),
\end{equation}
where $\{a_{d_x}\}$ are the coefficients of the characteristic polynomial of $A_s$.
Details regarding the construction of $T_c$ are provided in \Cref{appendix_thm:companion_form_invertible_matrix}. Other SISO canonical forms correspond to other transformation matrices. 
\end{proposition}
\Cref{thm:companion_form_invertible_matrix} guarantees that we can always map a controllable realization of an LTI system to a specific canonical one.  \Cref{thm:isomorphism_general} then guarantees that this transformation preserves the likelihood, and hence performing Bayesian inference directly on the canonical parameters \(\Theta_c\) captures all information contained in the input-output data.
\begin{remark}[Sufficiency of canonical forms for quantities of interest]
\label{remark_recovering_qoi}
While \Cref{thm:isomorphism_general} links equivalent representations $\Theta_s$ and $\Theta_c$ via state transformations $T$, attempting to formally recover the full posterior $p(\Theta_s \mid y_{[T]}, u_{[T]})$ from $p(\Theta_c \mid y_{[T]}, u_{[T]} )$ by defining a mixture of pushforward distributions
over all possible $T$ is difficult. The group $\operatorname{GL}(d_x)$ is non-compact and lacks a unique invariant probability measure (e.g., a Haar measure suitable for probability normalization), making the notion of a ``uniform distribution over transformations $T$'' ill-defined.

Fortunately, this formal recovery of the redundant posterior $p(\Theta_s \mid y_{[T]}, u_{[T]} )$  is often unnecessary. Many critical system properties, such as the system's eigenvalues (poles),  transfer function, Hankel matrix, and Hankel singular values, depend only on the equivalence class, not the specific representative \(\Theta_s\) chosen. These invariant QoIs can therefore be computed directly and exactly from the canonical parameters \(\Theta_c\).
In this case, performing inference on \(\Theta_c\) enables more interpretable priors (e.g., directly on canonical parameters related to stability or structure), avoids redundancies induced by the gauge symmetry of similarity transformations, reduces dimensionality of the parameter space, and still provides access to the dynamical features necessary for analysis, prediction, and control.
\end{remark}

This invariance of derived quantities of interest leads directly to a central result regarding the sufficiency of canonical forms for Bayesian inference on such quantities. 
Our goal is to have equivalent posterior push forwards (\Cref{equivalence_of_pushforwards}). To achieve this, we need equivalent likelihoods (\Cref{thm:isomorphism_general}) and compatible priors (\Cref{lem:induced_prior}).
\begin{lemma}[Induced prior on a canonical parameterization]
\label{lem:induced_prior}
Let $\bm{\Theta}_s$ be the standard parameter space for an LTI system of a given state dimension, and let $\bm{\Theta}_s^{\text{min}} \subseteq \bm{\Theta}_s$ be the subset of minimal systems. Let $\bm{\Theta}_c$ be a corresponding canonical parameter space, which is the quotient space of $\bm{\Theta}_s^{\text{min}}$ under the equivalence relation of state-space similarity. 
Let $\tau: \bm{\Theta}_s^{\text{min}} \to \bm{\Theta}_c$ denote the canonical projection map, assigning each minimal parameterization $\Theta_s \in \bm{\Theta}_s^{\text{min}}$ to its unique equivalence class $\tau(\Theta_s) = \Theta_c \in \bm{\Theta}_c$.

Let $p(\Theta_s)$ be a prior probability density on $\bm{\Theta}_s$ such that the set of non-minimal systems has measure zero, i.e.,
$$
\int_{\bm{\Theta}_s \setminus \bm{\Theta}_s^{\text{min}}} p(\Theta_s) \, d\Theta_s = 0
$$
Then this prior \emph{induces} a corresponding probability density $p(\Theta_c)$ on the canonical parameter space $\bm{\Theta}_c$. We say that the priors on these spaces are \emph{consistent}.
\end{lemma}
%
The proof is provided in \Cref{appendix_lem:induced_prior}. For example, if we take as our canonical parameterization the SISO controller canonical form, the canonical projection map $\tau$ acts on the LTI system matrices in $\Theta_s$ as in \Cref{thm:isomorphism_general}, with $T = T_c$ given in \Cref{thm:companion_form_invertible_matrix}.

\begin{theorem}[Equivalence of pushforward posteriors]
\label{equivalence_of_pushforwards}
Let \(\mathbf{\Theta}_s\) and \(\mathbf{\Theta}_c\) denote the standard and canonical parameter spaces for LTI systems, respectively. Define the bounded parameter spaces:
\begin{align}
\mathbf{\Theta}_l^M = \{(A, B, C, D) \in \mathcal{L}_{\mathbf{\Theta}_l} : \|A\|_F \leq M_A, \|B\|_F \leq M_B, \|C\|_F \leq M_C, \|D\|_F \leq M_D\} 
\end{align}
for \(l \in \{c,s\}\), where \(\|\cdot\|_F\) is the Frobenius norm and \(M_A, M_B, M_C, M_D > 0\) are constants.
Let \(Q \in \mathcal{Q}\) be a quantity of interest, and let \(S_l: \mathbf{\Theta}_q^M \to \mathcal{Q}\) be  a measurable map that is invariant under similarity transformations, meaning \(S_l(\Theta_l) = S_l(\Theta'_l)\) whenever \(\Theta_l, \Theta'_l \in \mathbf{\Theta}_l^M\) are related by a similarity transformation. Let $p(\Theta_s)$ be the prior distribution on the standard parameter space and $p(\Theta_c)$ be the corresponding induced prior from \Cref{lem:induced_prior}. 
%
Then, for any measurable set \(\mathcal{B} \subseteq \mathcal{Q}\), the pushforward posterior probabilities are identical:
\begin{equation}
\int_{\{\Theta_s \in \mathbf{\Theta}_s^M \mid S_s(\Theta_s) \in \mathcal{B}\}} p(\Theta_s \mid u_{[T]}, y_{[T]}) \, d\Theta_s = 
\int_{\{\Theta_c \in \mathbf{\Theta}_c^M \mid S_c(\Theta_c) \in \mathcal{B}\}} p(\Theta_c \mid u_{[T]}, y_{[T]}) \, d\Theta_c .
\end{equation}
\end{theorem}

%

This proposition establishes that the posterior distribution of any transformation-invariant QoI does not depend on whether inference is performed in the standard space \(\mathbf{\Theta}_s\) or the canonical space \(\mathbf{\Theta}_c\). A proof is provided in \Cref{appendix_equivalence_of_pushforwards}. This equivalence justifies using the identifiable, lower-dimensional canonical parameterization for Bayesian inference without loss of information about key system properties. 
Key examples of such invariant QoIs include the eigenvalue spectrum, Markov parameters, and transfer function, as shown in \Cref{invariance_and_computational_advantage_canonical_forms}.

\section{Prior specification}
\label{sec:priors}

A crucial component of the Bayesian framework outlined in \Cref{sec:bayesian_inference_lti_systems_standard_perspective} is the specification of the prior distribution \(p(\Theta)\) in \cref{eq:bayes_theorem}. This prior encodes knowledge or beliefs about the system parameters available before observing data.
Ideally, the chosen prior should be scientifically meaningful, reflecting domain knowledge, and impose desirable system properties, such as stability. If possible, the prior should also be computationally convenient, allowing efficient sampling and/or density evaluation. A central theme of this work is the challenge of specifying such priors effectively.

While priors can, in principle, be placed directly on matrix entries of the standard parameterization \(\Theta_s\), this approach lacks transparency regarding the induced properties of the system dynamics.
In contrast, as we elaborate below, priors defined directly on interpretable system quantities---notably the eigenvalues of the state matrix \(A\), which govern essential dynamics---tend to be more relevant and intepretable. 
The  entries of $A$ can of course be mapped to its eigenvalues, but using this transformation to specify a prior can be intractable.
In contrast, the canonical parameterization \(\Theta_c\) connects very naturally to priors on eigenvalues, offering a feasible computational path and providing clearer insights into the system's dynamics; this is the basis of our recommended approach.

\subsection{Priors over state dynamics}
\label{subsec:priors_A}
The matrix \(A\) in \eqref{eq:general_lti_state_space_model_noise} or \eqref{eq:controller_form} dictates the system's internal dynamics (e.g., stability, timescales, oscillation, and dissipation). Specifying a prior over \(A\) or its properties thus warrants careful consideration.
As discussed above, we will define priors directly on the eigenvalues \(\Lambda = \{\lambda_1, \dots, \lambda_{d_x}\}\) of \(A\). The connection between $\Lambda$ and the parameters $(a_0, \ldots, a_{d_x-1})$ of the canonical form is established through the characteristic polynomial of \(A\), which is sufficient to instantiate its canonical form via \Cref{def:controller_form}. Vieta's formulas provide the explicit algebraic relationship between the roots (eigenvalues) and the polynomial coefficients \cite{barbeau2003polynomials}. 

\begin{proposition}[Vieta's formulas]
\label{prop:vieta_formulas}
Let \(p(\lambda) = \lambda^{d_x} + a_{{d_x}-1}\lambda^{{d_x}-1} + \dots + a_1 \lambda + a_0\) be a monic polynomial of degree \({d_x}\) with coefficients \(a_k \in \C\). Suppose \(p(\lambda)\) has \({d_x}\) roots \(\lambda_1, \dots, \lambda_{d_x} \in \C\) (counted with multiplicity). Then, for each integer \(k\) with \(1 \leq k \leq {d_x}\), the elementary symmetric polynomials in the roots relate to the coefficients as:
\begin{equation}
a_{{d_x}-k} = (-1)^k \sum_{1 \leq i_1 < i_2 < \dots < i_k \leq {d_x}} \left(\prod_{j=1}^k \lambda_{i_j}\right).
\end{equation}
\end{proposition}

These formulas define a mapping \(\Psi: \Lambda \mapsto \{a_0, \dots, a_{{d_x}-1}\}\) from the space of eigenvalues to the space of coefficients appearing in canonical forms.
Therefore, if we specify a prior density \(p_{\Lambda}(\lambda_1, \dots, \lambda_{d_x})\) on the eigenvalues, we can derive the induced \textit{prior density on the canonical coefficients} \(p_{\mathcal{P}}(a_0, \dots, a_{{d_x}-1})\) using the change of variables formula, provided the map \(\Psi\) is suitably invertible (perhaps locally or on specific domains).

\begin{proposition}[Change of variables via polynomial roots]
\label{thm:pushforward_polynomial_roots}
Let \(\mathbf\Lambda \subseteq \mathbb{C}^{d_x}\) (or \(\mathbb{R}^{d_x}\))
be the space of \({d_x}\) roots (eigenvalues) and let \(\mathcal{P}_{d_x}\) be the space of corresponding monic polynomial coefficients \(\{a_0, \dots, a_{{d_x}-1}\}\) identified with \(\mathbb{C}^{d_x}\) (or \(\mathbb{R}^{d_x}\)). Let \(\Psi: \mathbf\Lambda \to \mathcal{P}_{d_x}\) be the mapping defined by Vieta's formulas (\Cref{prop:vieta_formulas}).
Assume \(\Psi\) is invertible with inverse \(\Psi^{-1}\). If \(p_{\Lambda}(\lambda_1,\dots,\lambda_{d_x})\) is a probability density on \(\Lambda\), the induced (pushforward) probability density on the coefficients \(\{a_k\}\) is given by:
\begin{equation}
p_{\mathcal{P}}(a_0, \dots, a_{{d_x}-1})
\;=\;
p_{\Lambda}\left (\Psi^{-1}\left (a_0, \dots, a_{{d_x}-1}\right )\right)
\bigl|\det D\bigl(\Psi^{-1}\bigr)\!\bigl(a_0, \dots, a_{{d_x}-1} \bigr) \bigr|,
\end{equation}
where \(D(\Psi^{-1})\) is the Jacobian matrix of the inverse map. If the eigenvalues \(\lambda_1,\dots,\lambda_{d_x}\) are distinct, the  Jacobian determinant of the forward map \(\Psi\) is given by the Vandermonde product:
\begin{equation}
\label{eq:det_vandermonde}
\bigl|\det D\Psi(\lambda_1,\dots,\lambda_{d_x}) \bigr|
\;=\;
\prod_{1 \,\le\, i<j \,\le\, {d_x}} \bigl|\lambda_i - \lambda_j\bigr|.
\end{equation}
Finally, \(|\det D(\Psi^{-1}) | = 1 / |\det D\Psi |\).
\end{proposition}
A proof and further details are provided in \Cref{appendix_thm:pushforward_polynomial_roots}. This result allows us to translate priors specified on eigenvalues into priors on the canonical parameters needed for computation.

%

\subsubsection{Specifying priors on eigenvalues}
A primary advantage of eigenvalue priors is the ability to directly enforce \textit{stability} (\Cref{def_stable_system}). A prior distribution over the eigenvalues \(\{\lambda_1, \dots, \lambda_{d_x}\}\) ensuring stability (\(|\lambda_i| < 1\) for all \(i\)) can generally be expressed as:
\begin{equation}
p\bigl(\lambda_1,\ldots,\lambda_{d_x}\bigr)
\;\propto\; \mathbbm{1}_{\{\max_i |\lambda_i| < 1\}} \,
\widetilde{p}\bigl(\lambda_1,\ldots,\lambda_{d_x}\bigr),
\label{eq:lti_stable_prior}
\end{equation}
where \(\widetilde{p}\) represents a base prior density 
defined on \(\mathbb{C}^{d_x}\) or some subset thereof. Choices for \(\widetilde{p}\) reflect different assumptions or prior information about system behavior. \Cref{fig:eigenvalue_priors} illustrates several options for defining priors on eigenvalues within the unit disk ($|\lambda|<1$), which we discuss below.

\paragraph{Restricted region prior}\label{prior_restricted_magnitude} Priors can enforce specific stability margins by restricting eigenvalues to, e.g., $0 < \lambda_i \le 0.9$ for real eigenvalues, as shown in \Cref{fig:eigenvalue_priors}(a) (cf.~\cite{oymak2019}).

\paragraph{Uniform real eigenvalue prior}\label{prior_real_uniform} A simple reference prior is to restrict eigenvalues to be real, i.e., $\lambda_i \in \mathbb{R}$. With a uniform base distribution $\widetilde{p}$, the stability constraint $\mathbbm{1}_{\{\max_i |\lambda_i| < 1\}}$ confines support to the real hypercube $(-1, 1)^{d_x}$. \Cref{fig:eigenvalue_priors}(b) shows samples on the real segment $(-1, 1)$.

\paragraph{Polar coordinate prior}\label{prior_polar_coordinates} For complex eigenvalues ($\lambda = \rho e^{i\theta}$, $\rho < 1$), priors can use polar coordinates. Non-informative choices might assume uniformity in area (e.g., $\rho^2 \sim \text{Uniform}(0, 1)$) and angle ($\theta \sim \text{Uniform}(0, \pi)$ for the upper half-plane), as visualized in \Cref{fig:eigenvalue_priors}(c).

\paragraph{Implied prior from uniform coefficients}\label{prior_stable_uniform} Alternatively, a uniform prior can be placed on the characteristic polynomial coefficients $\{a_k\}$ within their stability region (e.g., the stability triangle for ${d_x}=2$). This induces a specific, non-uniform \textit{mixture distribution} on the eigenvalues, as depicted in \Cref{fig:eigenvalue_priors}(d) and detailed for ${d_x}=2$ in the following \Cref{lemma:uniform_prior_2d}.
\begin{figure}[t]
    \centering
    \begin{subfigure}[t]{0.48\textwidth}
        \centering
        \includegraphics[width=\textwidth]{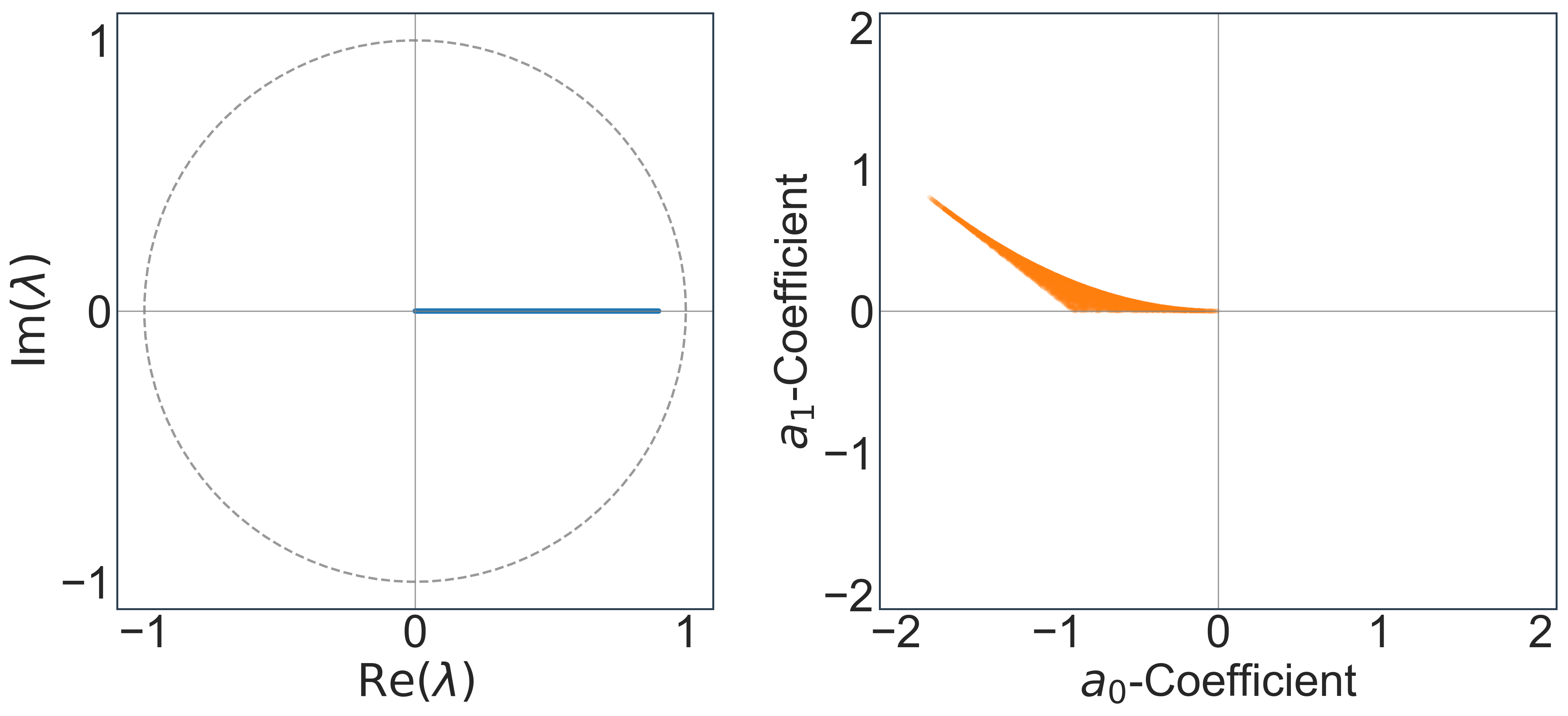}
        \caption{Prior restricting $\lambda$ to $[0,0.9]$}
    \end{subfigure}\hfill
    \begin{subfigure}[t]{0.48\textwidth}
        \centering
        \includegraphics[width=\textwidth]{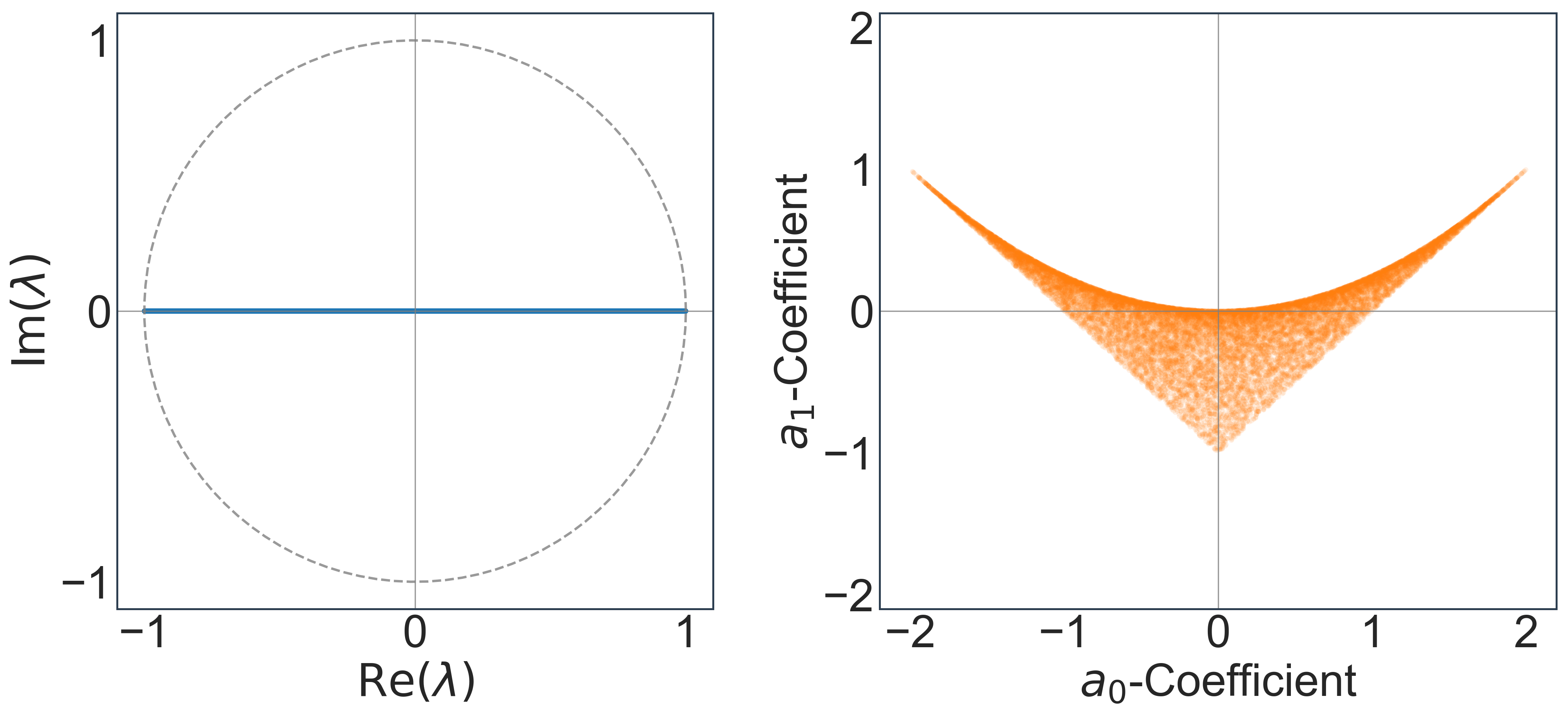}
        \caption{Uniform real-eigenvalue prior}
    \end{subfigure}

    \vspace{1em}

    \begin{subfigure}[t]{0.48\textwidth}
        \centering
        \includegraphics[width=\textwidth]{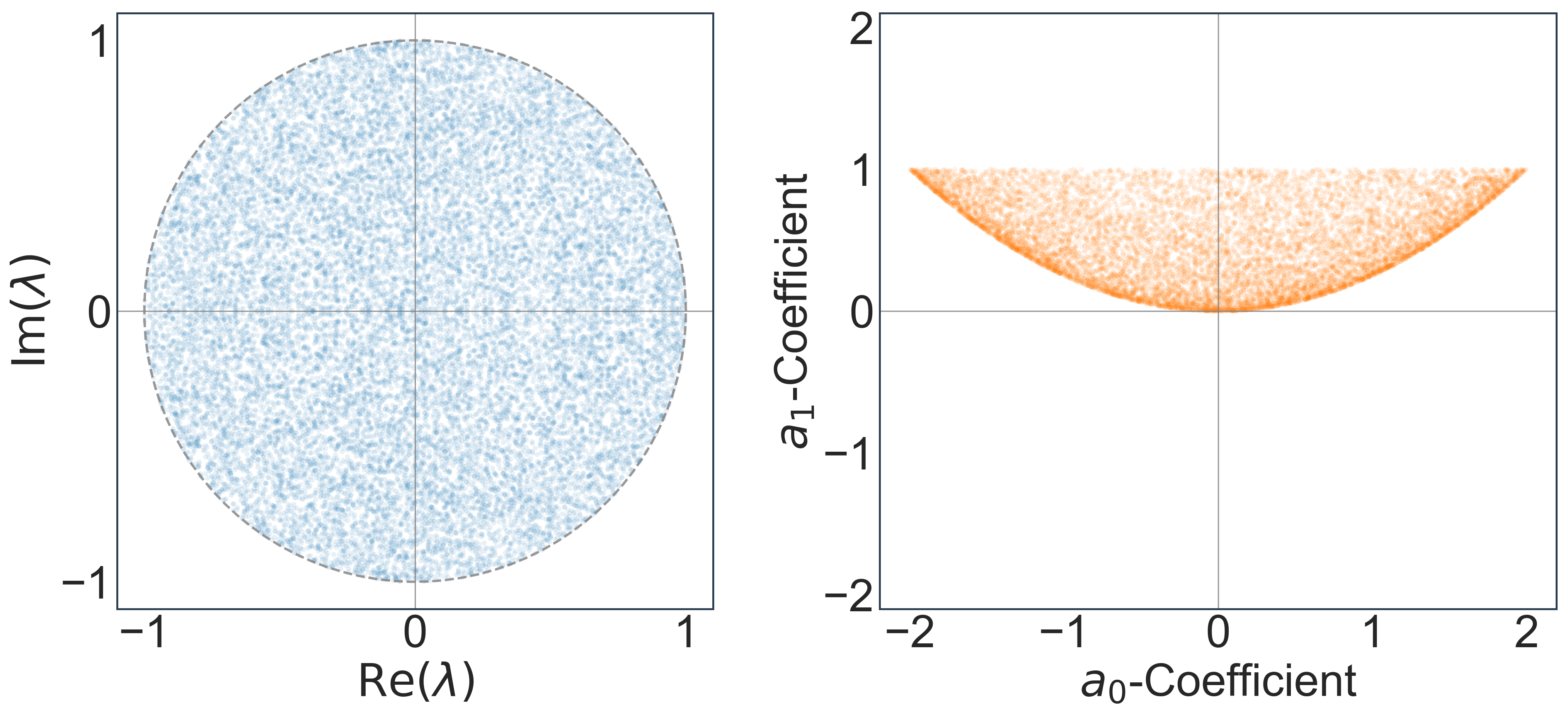}
        \caption{Polar-coordinate prior}
    \end{subfigure}\hfill
    \begin{subfigure}[t]{0.48\textwidth}
        \centering
        \includegraphics[width=\textwidth]{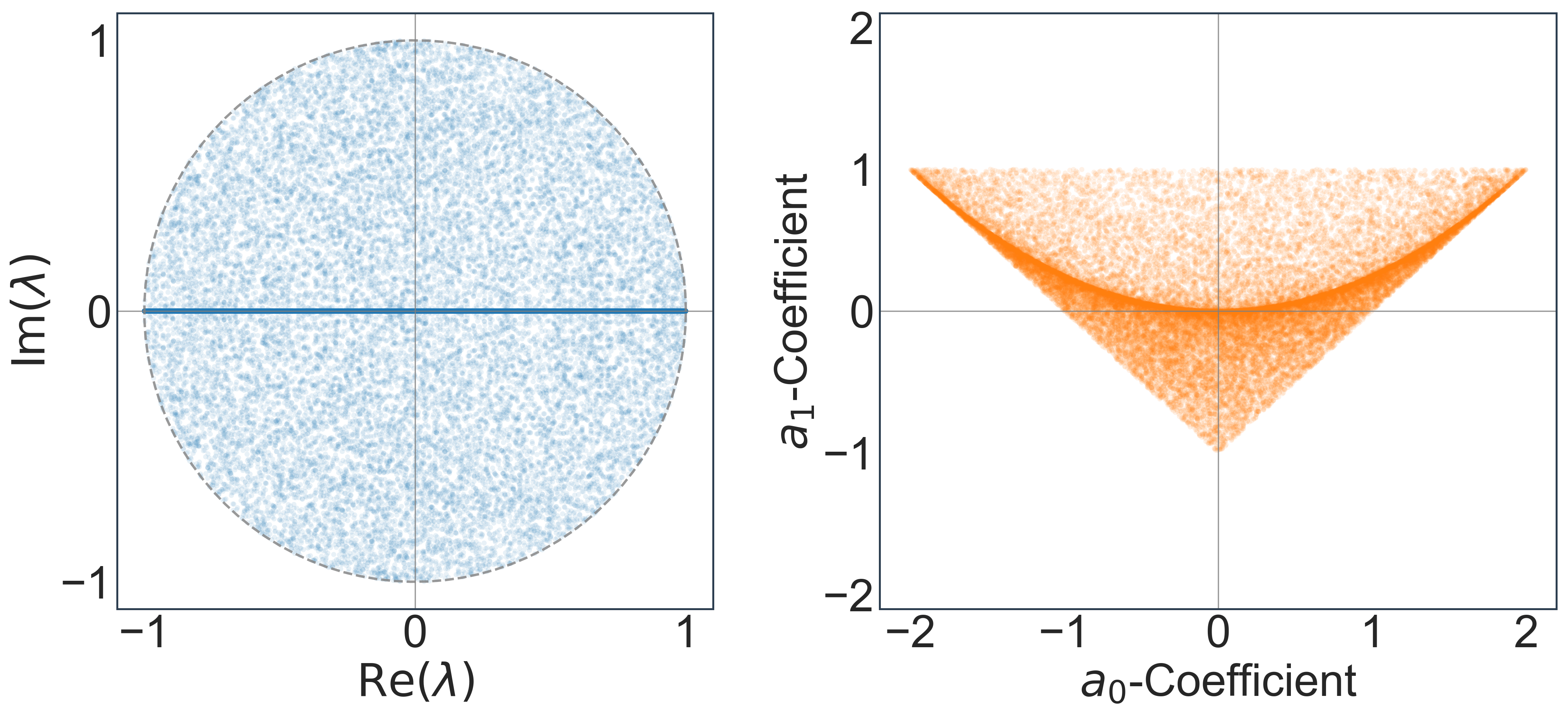}
        \caption{Uniform prior on stable coefficients}
    \end{subfigure}

    \caption{Visualization of samples from different prior distributions over eigenvalues
      $\lambda$ for stable linear systems ($|\lambda|<1$), alongside samples from the resulting prior on polynomial coefficients $(a_0, a_1)$. (a) Eigenvalue
      magnitudes restricted to $|\lambda|\le 0.9$ (cf.~\cite{oymak2019}).
      (b) Uniform distribution over stable real eigenvalues on $(-1,1)$.
      (c) Prior parameterized via polar coordinates $(\rho,\theta)$ over the
      complex unit disk. (d) Implied eigenvalue distribution from a uniform
      prior on stable 2-D system coefficients, showing a characteristic
      mixture shape (\Cref{lemma:uniform_prior_2d}).}
    \label{fig:eigenvalue_priors}
\end{figure}

\begin{lemma}[Implied eigenvalue density from uniform stable 2D systems]
\label{lemma:uniform_prior_2d}
Assume a uniform prior distribution over the coefficients \((a_0, a_1)\) corresponding to stable ${d_x}=2$ real LTI systems (i.e., uniform within the stability triangle defined by \(|a_0| < 1\) and \(|a_1| < 1+a_0\)). The implied marginal probability density for a single eigenvalue \(\lambda\) is the mixture:
\begin{equation}
p(\lambda) = \frac{2}{3} \left(\frac{1}{2}\mathbbm{1}_{\{\lambda \in \mathbb{R}, |\lambda|<1\}}\right) + \frac{1}{3}  \left(\frac{1}{\pi}\mathbbm{1}_{\{\lambda \in \mathbb{C}\setminus\mathbb{R}, |\lambda|<1\}}\right).
\end{equation}
\end{lemma}
The proof is provided in \Cref{appendix_lemma:uniform_prior_2d}. This indicates a 2:1 preference for real eigenvalues over complex conjugate pairs when sampling uniformly from stable coefficients. The resulting density is uniform over \((-1, 1)\) for real eigenvalues and uniform over the upper (or lower) half of the open unit disk for complex eigenvalues. \Cref{fig:eigenvalue_priors}(d) visualizes samples from this implied distribution.

The tendency for priors over stable real systems (whether defined on eigenvalues or implied by coefficient priors) to involve mixtures of real and complex eigenvalues, as explicitly derived in \Cref{lemma:uniform_prior_2d}, is a general characteristic, formalized below.
\begin{corollary}[Mixture distribution for eigenvalues of a stable system]
\label{cor:mixture_stable_eigs}
For any dimension \(d \geq 2\), the eigenvalue spectrum of a stable real matrix \(A\) comprises \(r\) real eigenvalues in \((-1, 1)\) and \((d-r)/2\) complex conjugate pairs within the open unit disk, where the number of real eigenvalues, \(r\), satisfies \(0 \le r \le d\) with \(d-r\) being an even number. Consequently, priors over the eigenvalues of stable real systems can be expressed as mixtures over the possible values of \(r\). The marginal density of a single eigenvalue \(\lambda\) can be written as:
\begin{equation}
p(\lambda) = \sum_{\substack{0 \le r \le d \\ d-r \text{ is even}}} p_r(r) \left[ \, \frac{r}{d} \, p_{\mathrm{real}}(\lambda; r) + \frac{d-r}{d} \, p_{\mathrm{complex}}(\lambda; r)\right],
\end{equation}
where \(p_r(r)\) is the prior probability mass on having \(r\) real eigenvalues; \(p_{\mathrm{real}}(\lambda; r)\) is the conditional density for a real eigenvalue given \(r\), supported on \((-1,1)\); and \(p_{\mathrm{complex}}(\lambda; r)\) is the conditional density for a complex eigenvalue given \(r\), supported on the open unit disk excluding the real line.
\end{corollary}
The proof is provided in \Cref{appendix_cor:mixture_stable_eigs}. This decomposition allows considerable flexibility in prior modeling through the specification of the base densities \(p_{\text{real}}\), \(p_{\text{complex}}\), and the mixture weights \(p_r(r)\). Even though for \({d_x}=2\) we have an explicit analytical result for the mixture weights, in higher dimensions their direct calculation becomes intractable. Such a calculation would rely on stability criteria, such as the Schur--Cohn test for determining whether all roots of the system's characteristic polynomial lie within the unit disk, which are computationally infeasible for large \({d_x}\). Consequently, estimating the weights via Monte Carlo simulation becomes the main approach.

\paragraph{Comparison with priors on matrix entries}
\label{remark_prior_comparison_with_matrix_entries}
%
Previous work 
in Bayesian system identification placed independent priors (e.g., \(p(A_{ij}) \propto \mathcal{N}(0,1)\)) on the entries of a full matrix \(A_s \in \mathbb{R}^{d_x \times d_x} \) \cite{Galioto_2020}. Such entry-wise priors lack a clear connection to system-level behavior like stability or damping, making it difficult to encode meaningful prior knowledge. While one could try to use any of the eigenvalue priors discussed above while performing inference in the standard parameterization (i.e., over full matrices $A_s$), doing so would require evaluating the complex many-to-one mapping from \(A_s\) to \(\Lambda\), which has an intractable Jacobian, and resolving its general non-invertibility. See further discussion in \Cref{remark_challenges_setting_priors_over_standard_form}. 

Placing priors directly on the coefficients \(\{a_k\}\) of a canonical form is mathematically simpler, but still lacks the direct interpretability of eigenvalue priors. Therefore, defining a prior on eigenvalues \(p(\Lambda)\) and
(via \Cref{thm:pushforward_polynomial_roots}) transforming this distribution to a prior on the canonical coefficients \(p_{\mathcal{P}}(a_k)\) represents an effective strategy. Note also that the the number of coefficients $\{a_k\}$ is exactly the same as the number of eigenvalues, such that $\Psi^{-1}$ and its Jacobian determinant are almost surely well defined (for the eigenvalue priors discussed above). This approach thus combines the intuitive control over system dynamics offered by eigenvalues with the computational benefits of working with a unique, identifiable canonical parameterization.
\begin{remark}[Non-distinct eigenvalues]
  A critical consideration arises when eigenvalues $\lambda_i$ are not distinct. The map $\Psi$ from eigenvalues to coefficients becomes singular at such points. While this degeneracy occurs only on a lower-dimensional subspace (typically assigned zero probability by continuous priors), it still has numerical implications. Eigenvalues that are very close cause the inverse Jacobian term to become extremely large. This numerical instability presents a practical hurdle for MCMC samplers exploring the parameter space near such degeneracies. Robust implementations may require numerical regularization, potentially facing sampling inefficiencies when the Jacobian $D\Psi(\Lambda)$ becomes ill-conditioned,  or necessitate alternative prior specifications (e.g., defining a prior directly on the canonical coefficients and enforcing stability through separate checks).
\end{remark}

\subsection{Priors over parameters (\(B_c, C_c, D_c\) and noise covariances)}
\label{subsec:priors_other}
While the prior on the state dynamics matrix $A$ (often specified via its eigenvalues $\Lambda$) demands careful consideration due to its strong influence on system behavior, priors on the remaining parameters are typically chosen based on simpler, standard considerations, often aiming for relative non-informativeness or some degree of regularization. For the SISO controller canonical form, parameters in $B_c, C_d$, and $D_c$ consist of the numerator coefficients $\{b_0, \dots, b_{{d_xs-1}}\}$ in the $C_c$ vector and the direct feed-through term $d_0$. Common choices involve independent zero-mean Gaussian priors on these coefficients, e.g., $b_i \sim \mathcal{N}(0, \sigma_b^2)$ and $d_0 \sim \mathcal{N}(0, \sigma_d^2)$. The prior variances ($\sigma_b^2, \sigma_d^2$) can be fixed or assigned hyperpriors.

Priors for covariance matrices like $\Sigma$ and $\Gamma$ can be set directly, often using an inverse-Wishart distribution. For greater flexibility, the matrix can be parameterized via its Cholesky factor, $\Sigma = L_\Sigma L_\Sigma^T$, allowing for separate priors on its elements. Typically, zero-mean Gaussian priors are used for the off-diagonal elements, while distributions with positive support, such as half-Cauchy or Gamma, are used for the strictly positive diagonal elements. A key special case is when the covariance matrix is assumed to be diagonal. In this scenario, its Cholesky factor $L_\Sigma$ is also diagonal, and its entries directly correspond to the standard deviations of the noise terms. Therefore, placing a half-Cauchy prior on the diagonal elements of $L_\Sigma$ is equivalent to placing a prior directly on the standard deviations \cite{alvarez2014bayesian}.

\section{Posterior asymptotics: Bernstein--von Mises theorem}
\label{sec:fim_bvm}
Understanding the asymptotic behavior of the posterior distribution is fundamental in Bayesian inference, providing insights into the efficiency of Bayes estimates and uncertainty calibration \cite{gelman2013bayesian}.  Here, we show that employing an identifiable canonical parameterization is essential for the validity of Bernstein--von Mises (BvM) results and the associated Gaussian posterior approximations, in the context of LTI system identification.

The BvM theorem is a cornerstone result \cite{vandervaart1998asymptotic} that establishes conditions under which the Bayesian posterior distribution converges to a Gaussian distribution centered at an efficient point estimate as the amount of data grows.
This asymptotic normality links Bayesian inference to frequentist concepts and justifies Gaussian approximations for uncertainty quantification in large-sample regimes. Central to this theorem is the Fisher information matrix (FIM), which quantifies the information about the parameters contained in the data via the likelihood function, defining the precision of the limiting Gaussian posterior \cite{vandervaart1998asymptotic}.

\begin{definition}[LTI Fisher information matrix]
\label{def:likelihood_fim} 
Consider the likelihood function \(p(y_{[T]} \mid \mathcal{L}_\Theta, u_{[T]})\) for parameters \(\Theta\). Assume standard regularity conditions hold, ensuring differentiability of the log-likelihood with respect to \(\Theta\)
and allowing interchange of differentiation and integration with respect to the data \(y_{[T]}\). The \emph{Fisher information matrix (FIM)} at \(\Theta\) based on data up to time \(T\) is:
\[
  I_T(\Theta, u_{[T]})
  \;=\;
  \E_{y_{[T]} \sim Y_{[T]} \mid \Theta, u_{[T]}} \! \left[
    \left( \frac{\partial}{\partial\Theta} \log p\left(y_{[T]} \,\middle|\, \mathcal{L}_\Theta, u_{[T]}\right) \right) \!
    \left( \frac{\partial}{\partial\Theta} \log p\left(y_{[T]} \,\middle|\, \mathcal{L}_\Theta, u_{[T]}\right) \right)^{\!\top}
  \right].
\]
\end{definition} 
Note that in our setting, the FIM $I_T$ is also a function of the known input or excitation $u_{[T]}$. \(I_T(\Theta, u_{[T]})\) quantifies the expected curvature of the log-likelihood function at \(\Theta\), representing the information provided about the parameter \(\Theta\) by an experiment with inputs $u_{[T]}$ yielding data \(y_{[T]}\). The definition above holds for both $\Theta = \Theta_s$, i.e., estimation of full-dimensional LTI system matrices $(A, B, C, D)$ and for $\Theta = \Theta_c$, i.e., estimation of the reduced parameters of a canonical form.
\paragraph{Fisher information matrix in canonical forms}
The FIM is calculated by accumulating output sensitivities with respect to the canonical parameters; it can computed efficiently using recursive formulas. Detailed derivations, including explicit expressions for the FIM of the controller and observer forms (\Cref{prop:fisher_LTI}) and extensions for systems with process noise via Kalman filter sensitivities (\Cref{prop:fisher_kalman_canonical}), are provided in \Cref{appendix_fisher_information_and_BvM}. For finite-data scenarios, we also introduce an alternative metric for posterior geometry, the expected posterior curvature, in \Cref{expected_posterior_curvature_section}.
We now discuss the conditions under which a BvM theorem holds for Bayesian LTI system identification. For simplicity, our analysis focuses on the inference of the LTI dynamics matrices $(A, B, C, D)$ rather than the noise covariances. Crucially, our result requires that inference be performed using the canonical parameterization, $\Theta_c$
(\Cref{def:controller_form}), and that the dimension of the model class used for inference is consistent with the minimal dimension of the data-generating system.
\begin{assumption}[Consistent state dimension]
\label{ass:minimal_state_dim}
The minimal state dimension of the data-generating system, denoted $d_x^{0}$, equals the state dimension $d_x$ of the models described by $\mathcal{L}_\Theta$.\footnote{Estimating the minimal state dimension \(d_x^0\) itself from the data $(y_{[T]}, u_{[T]})$ alone is considered an outer-loop problem, beyond the scope of the core inference task addressed here.}
\end{assumption}

Next, we require that the inputs $u_{[T]}$ be ``sufficiently rich'' to excite relevant aspects of the dynamics. 
\begin{definition}[Persistence of excitation]
\label{def:persistence_excitation_1}
A sequence of scalar inputs $\{u_t\}_{t \in \mathbb{N}_0}$ is \emph{persistently exciting of order $d$} 
if the $d \times d$ limiting sample autocovariance matrix $R_u^{(d)}$ exists and is positive definite:
\[
R_u^{(d)} := \lim_{T \to \infty} \frac{1}{T} \sum_{t=d}^T 
\begin{pmatrix}
u_t \\
u_{t-1} \\
\vdots \\
u_{t- d + 1}
\end{pmatrix}
\begin{pmatrix}
u_t & u_{t-1} & \ldots & u_{t- d +1}
\end{pmatrix}
\succ 0.
\]
\end{definition}

Under these conditions, and when performing inference with the canonical parameterization $\Theta_c$, we can prove a BvM. We state this theorem somewhat informally below; a more precise statement is given in \Cref{appendix_thm:BvM_LTI_full}.
\begin{theorem}[Bernstein--von Mises for LTI systems]
\label{thm:BvM_LTI_applicability}
Let $\{y_t\}_{t \in \mathbb{N}_0}$ be the sequence of outputs generated by a SISO LTI system with true canonical parameters $\Theta^0_c$, given an input sequence $\{u_t\}_{t \in \mathbb{N}_0}$ that is persistently exciting of order $d_x^0$. Assume the noise covariances are known, so that the $N_c$-dimensional parameter vector $\Theta_c$ contains only the canonical representation of the system dynamics matrices and, further, satisfies \Cref{ass:minimal_state_dim}. Let $p(\Theta_c \mid y_{[T]}, u_{[T]})$ denote the posterior distribution over this canonical parameterization. Under suitable regularity conditions on the likelihood \eqref{eq:joint_likelihood} and the prior $p(\Theta_c)$, the posterior distribution converges in total variation to a Gaussian distribution as $T \to \infty$, as follows:
\begin{equation}
\label{eq:bvm_convergence_generic}
\sup_{E \subset \mathbb{R}^{N_c}} \left| P \left(\sqrt{T} (\Theta_c - \widehat{\Theta}_{c}^T ) \in E \mid y_{[T]}, u_{[T]}\right) - \Phi\left(E; 0, \mathcal{I}(\Theta_c^0)^{-1}\right) \right| \xrightarrow{\mathbb{P}_{\Theta_c^0}} 0,
\end{equation}
where $\widehat{\Theta}_c^T$ is a $\sqrt{T}$-consistent estimator (e.g., the MLE or MAP), $P(\cdot \vert y_{[T]}, u_{[T]} )$ is the posterior probability measure, and $\Phi(E; \mu, \Sigma)$ denotes the probability under a Gaussian distribution $\mathcal{N}(\mu, \Sigma)$ of a set $E$. $\mathcal{I}(\Theta_c^0)$ is the asymptotic Fisher information matrix per observation at the true parameter value, defined as
$\mathcal{I}(\Theta_c^0) = \lim_{T \to \infty} \frac{1}{T} I_T(\Theta_c^0, u_{[T]})$, with $I_T$ being the FIM for $T$ observations. 
Convergence is in probability under the data-generating distribution $\mathbb{P}_{\Theta_c^0}$, which has density $p(y_{[T]} \mid u_{[T]}, \Theta_c^0)$.
\end{theorem}

In contrast with the preceding result, suppose we perform inference in the standard parameterization $\Theta_s$ of the LTI system. Then, even if \Cref{ass:minimal_state_dim} is satisfied and the inputs are persistently exciting of order $d_x$, the BvM fails to hold.
\begin{proposition}[BvM failure in standard parameterizations]
\label{thm:BvM_LTI_failure}
Consider an LTI state-space model \cref{eq:general_lti_state_space_model_noise} satisfying the previous assumptions (\Cref{ass:gaussian_noise}, \Cref{ass:minimal_state_dim}) and with persistently exciting inputs of full order. If the system is parameterized using the standard form \(\Theta_s\), the BvM theorem does not hold.
\end{proposition}
BvM failure for the standard parameterization $\Theta_s$ stems from a singular Fisher information matrix. This singularity occurs because $\Theta_s$ is not identifiable; the posterior consequently fails to concentrate on a single point. Rather, it concentrates on the manifold of parameters related by similarity transformations. Employing an identifiable canonical parameterization is therefore essential to satisfy the conditions of \Cref{thm:BvM_LTI_applicability}. The canonical form provides local identifiability, while Gaussian noise and a persistently exciting input ensure likelihood regularity and a positive definite asymptotic FIM. For a detailed proof and discussion, see \Cref{appendix_thm:BvM_LTI_failure}.

Violating \Cref{ass:minimal_state_dim} has distinct implications depending on the direction of the mismatch. Overestimating the minimal state dimension---i.e., inferring a canonical model with $d_x > d_x^{0}$---results in a singular Fisher information matrix, as redundant directions in parameter space are unidentifiable. Conversely, underestimating the minimal state dimension, i.e., setting $d_x < d_x^{0}$, leads to model misspecification, rendering the parameter estimates inconsistent and violating key conditions required for a Bernstein--von Mises theorem to hold.

\section{Numerical experiments}
\label{sec:numerics}

This section presents numerical experiments designed to validate the preceding theory and assess the practical performance of Bayesian inference for LTI systems using canonical parameterizations \(\Theta_c\) compared to using the standard, non-identifiable parameterization \(\Theta_s\).
\Cref{sec:exp_setup_revised} details the common experimental setup, including system generation, noise specifications, algorithmic implementations, and parameter estimation methods. \Cref{sec:post_geom_interp} analyzes the impact of different system parameterizations on posterior geometry and interpretability. \Cref{sec:eff_perf} evaluates the efficiency and computational cost of MCMC sampling for these posterior distributions. \Cref{sec:acc_prior} investigates the accuracy of parameter and QoI estimates, examining the influence of prior specification, particularly in low-data regimes. Finally, \Cref{sec:asymp_scale} investigates posterior convergence with increasing data and relates these results to BvM predictions. In \Cref{appendix_f:scalability_with_dimension}, we assess the computational scalability of inference in canonical forms with increasing system dimension.

\subsection{Summary of key findings}

The experiments detailed below demonstrate the following advantages of using canonical forms for Bayesian LTI system identification: 
\begin{enumerate}[label=(\roman*)]
    \item Canonical forms \(\Theta_c\) yield well-behaved, typically unimodal posterior distributions suitable for reliable inference and interpretation, whereas standard forms \(\Theta_s\) result in complex, multimodal posteriors (\Cref{sec:post_geom_interp}, \Cref{fig:canonical_analysis}, \Cref{fig:parameter_distribution_comparison}).

    \item Posterior sampling via MCMC exhibits greater computational efficiency in the canonical parameterization, e.g., achieving higher effective sample size per second (ESS/s) (\Cref{sec:eff_perf}, \Cref{fig:canonical_vs_standard_efficiency}).

    \item Bayesian inference with canonical forms accurately recovers QoIs, and the associate parameter/QoI estimates are more accurate than those produced by baseline system identification methods such as the Ho--Kalman algorithm \cite{oymak2019} (also called the eigensystem realization algorithm \cite{Minster2021Efficient}), particularly in low-data or noisy scenarios (\Cref{sec:acc_prior}, \Cref{fig:easy-vs-hard} right,  \Cref{fig:trajectory_comparison}, and \Cref{fig:convergence_grid} left).
      
    \item While all tested priors lead to posteriors that converge on true parameter values with sufficient data, informative priors enhance accuracy significantly in data-limited regimes, and stability-enforcing priors provide robustness (\Cref{sec:acc_prior}, \Cref{fig:convergence_grid}).
      
    \item Empirical results confirm asymptotic consistency and convergence of the canonical posterior towards the Gaussian distribution given by BvM theorem, validating use of the FIM for uncertainty analysis (\Cref{sec:asymp_scale}, \Cref{fig:fisher_information_convergence}).
      
    \item Bayesian system identification in the canonical form demonstrates significantly better computational scalability with increasing system dimension compared to inference in the standard parameterization (\Cref{appendix_f:scalability_with_dimension}, \Cref{fig:high-dim-grid}).
\end{enumerate}

\subsection{Experimental setup and implementation}
\label{sec:exp_setup_revised}

We use simulated data generated from known ground-truth LTI systems. We set the feedthrough matrix \(D=0\) in all experiments, as its inclusion does not significantly alter the relative performance comparison between parameterizations. We focus on \emph{stable} LTI systems (\(|\lambda_i(A)| < 1\)), representing the most common practical scenario.

\paragraph{System generation methodology} Unless otherwise stated (e.g., for scalability tests in \Cref{sec:asymp_scale}), a diverse random set of stable systems is generated as follows:
\begin{enumerate}
    \item \textit{Stable system matrix \(A\):} We construct the stable matrix $A \in \mathbb{R}^{d_x \times d_x}$ by separately sampling its eigenvalues and eigenvectors. First, we draw a set of stable eigenvalues from a prior distribution, such as the uniform polar prior discussed in \Cref{sec:priors}. Next, we sample an orthonormal basis of eigenvectors by drawing a matrix $V$ uniformly from the orthogonal group. The matrix is then assembled as $A=V \Lambda V^T$, where $\Lambda$ is the diagonal matrix of sampled eigenvalues.
    \item \textit{Input and output matrices \(B, C\):} We independently draw each entry of \(B \in \R^{d_x \times d_u}\) and \(C \in \R^{d_y \times d_x}\) from a standard normal distribution, \(\mathcal{N}(0,1)\).
\end{enumerate}

\paragraph{Ensuring controllability and observability} While random generation usually produces controllable and observable systems, numerical near-degeneracies can occur. To mitigate this, we compute the controllability Gramian \(W_c = \sum_{k=0}^{\infty} A^k B B^\top (A^k)^\top\) and observability Gramian \(W_o = \sum_{k=0}^{\infty} (A^k)^\top C^\top C A^k\) (convergent for stable \(A\)). Following \cite{Antoulas2005}, we assess conditioning by examining the eigenvalue spectra \(\{\sigma_i(W)\}\). We reject systems where the dominant eigenvalue accounts for more than 99\% of the total energy (trace), i.e., if \(\sigma_{\max}(W) / \operatorname{trace}(W) > 0.99\) for either \(W_c\) or \(W_o\). This filtering ensures that our test systems are robustly controllable and observable in practice; all $d_x$ system states have meaningful dynamics.
In real-world settings, this choice relates to the outer-loop problem of choosing a sufficient state dimension $d_x$.

\paragraph{Input and noise specifications} Persistently exciting inputs are generated as i.i.d.\ standard Gaussian random variables, \(u_t \sim \mathcal{N}(0, I_{d_u})\). Process noise \(w_t \sim \mathcal{N}(0, \Sigma)\) and measurement noise \(z_t \sim \mathcal{N}(0, \Gamma)\) are independent zero-mean Gaussian sequences, with diagonal covariances \(\Sigma = \sigma_\Sigma^2 I_{d_x}\) and \(\Gamma = \sigma_\Gamma^2 I_{d_y}\). The noise variances $\sigma_{\Sigma}^2$ and $\sigma_{\Gamma}^2$ are either fixed and reported per experiment or, if treated as unknown, assigned Truncated Normal priors. These priors are based on a Gaussian distribution centered at 0.5 but are restricted to positive values, ensuring a physically meaningful variance.

\paragraph{Prior distributions}
Default prior distributions for model parameters are specified as follows; any deviations for particular experiments will be detailed in their respective subsections. For parameters in the standard state-space form $\Theta_s$, all coefficients are assigned independent standard normal priors. For parameters in canonical forms $\Theta_c$, the state matrix $A$ is assigned a uniform stable prior (as discussed in \Cref{lemma:uniform_prior_2d}). Coefficients in the input matrix $B$ and output matrix $C$ are assigned independent standard normal priors.

\paragraph{Algorithmic implementation} Posterior distributions are sampled using state-of-the-art algorithms and tools for Markov chain Monte Carlo (MCMC), specifically NUTS \cite{hoffman2011nouturnsampleradaptivelysetting} via \texttt{BlackJAX} \cite{cabezas2024blackjaxcomposablebayesianinference},
and \texttt{JAX} \cite{lin2024automaticfunctionaldifferentiationjax} for automatic differentiation of the log-posterior.
The NUTS step size \(\varepsilon\) and mass matrix \(M\) are tuned during warm-up via window adaptation as introduced in Stan's manual \cite{stan-manual}. 
The log-likelihood \(p(y_{[T]} \mid \mathcal{L}_{\Theta}, u_{[T]})\) is computed using a Kalman filter; a deterministic implementation is used for $\sigma_\Sigma=0$. A small nugget (\(10^{-12}\)) ensures covariance stability in the Kalman filter. For each posterior, we generate four MCMC chains, each with a sample trajectory of $20\,000$ steps and a warm-up period (with window adaptation) of $5\, 000$ steps. Convergence is monitored using effective sample size (ESS) and the Gelman--Rubin statistic \(\hat{R}\). The experiments were run on a server equipped with an Intel Xeon Platinum 8260 CPU.

\paragraph{Parameter estimates}
MCMC samples $\{\Theta^{(i)}\}_{i=1}^N$ from the posterior allow for the estimation of any quantity of interest $Q=f(\Theta; u_{[T]})$ that is a function of the model parameters and possibly the control sequence. We consider two main point estimators of such QoIs. First is the posterior mean estimate (PME), which is approximated by empirical expectation of the quantity over the posterior distribution:
\begin{equation}
    \hat{Q}^{\text{PME}} = \frac{1}{N} \sum_{i=1}^N f(\Theta^{(i)}; u_{[T]}) . 
\end{equation}
Note that setting $f$ to be the identity function recovers the PME of the parameters themselves. Other choices of $f$ correspond to specific QoIs. Second is a plug-in estimate, e.g., based on the maximum a posteriori (MAP) estimate of the parameters:
\begin{equation} \label{eq:plugin}
    \hat{Q}^{\text{MAP}} = f(\Theta^{\text{MAP}}; u_{[T]}), \quad \text{where} \quad \Theta^{\text{MAP}} = \operatorname{argmax}_{\Theta^{(i)}}\, p(\Theta^{(i)} \mid y_{[T]}, u_{[T]}) . 
\end{equation}
Recall that for a general nonlinear function $f$, the posterior mean is not equivalent to the plug-in estimate that uses the posterior mean of the parameters, i.e., $\frac{1}{N}\sum_i f(\Theta^{(i)}; u_{[T]}) \neq f(\frac{1}{N}\sum_i \Theta^{(i)}; u_{[T]})$.

\subsection{Posterior geometry and interpretability}
\label{sec:post_geom_interp}
The choice of LTI parameterization fundamentally shapes the posterior landscape. Canonical forms \(\Theta_c\) induce well-behaved posteriors, as illustrated in \Cref{fig:canonical_analysis} for inference from a trajectory of length $T=400$ with $\sigma_\Sigma = 0.3$ and $\sigma_\Gamma = 0.0$ and $d_x = 2$.
\Cref{fig:canonical_analysis}(a) shows pairwise posterior marginals of these parameters, revealing a unimodal and approximately Gaussian distribution easily amenable to MCMC exploration and interpretation. The posterior distribution of the leading eigenvalue, in \Cref{fig:canonical_analysis}(b), is straightforward to extract and similarly Gaussian-like, offering a clear characterization of the inferred system dynamics.

\begin{figure*}[htbp]
  \centering
  \begin{minipage}{0.48\textwidth}
    \centering
    \includegraphics[width=\textwidth]{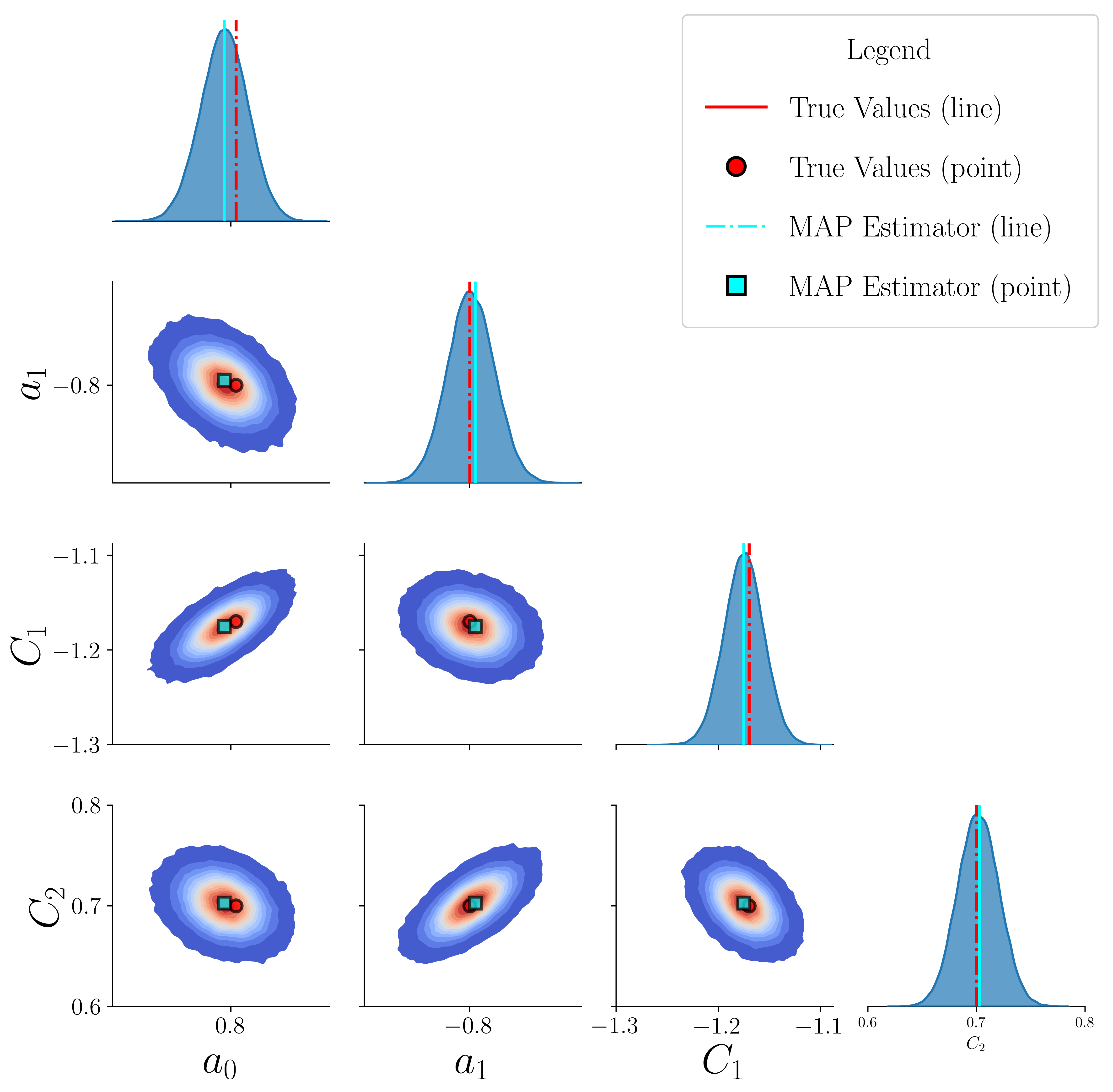}
    \vspace{0.5ex}\centering\textbf{(a)} Posteriors in canonical parameter space (\(\Theta_c\))\par\vspace{1ex}
  \end{minipage}
  \hfill
  \begin{minipage}{0.48\textwidth}
    \centering
    \includegraphics[width=\textwidth]{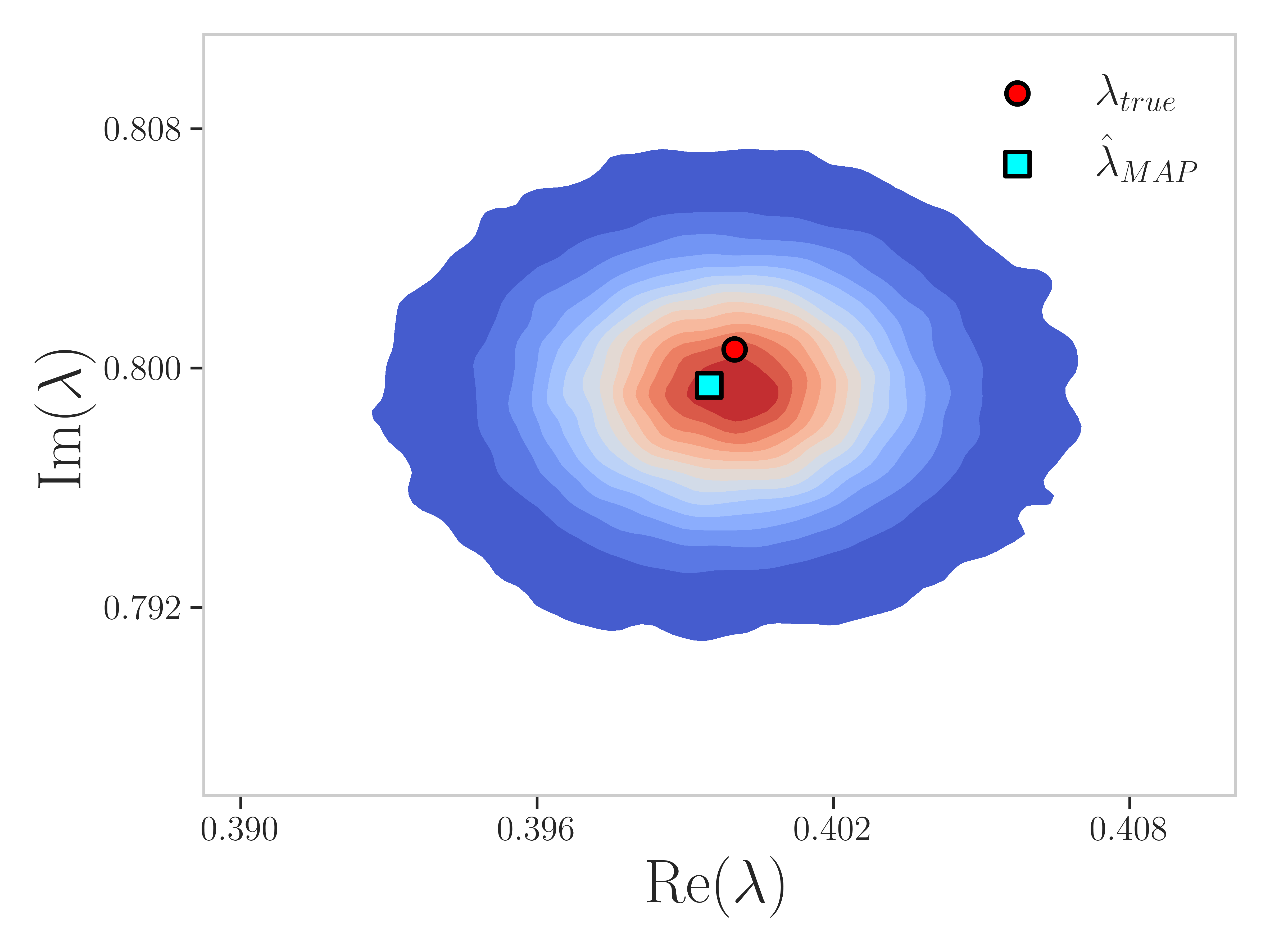}
    \vspace{0.5ex}\centering\textbf{(b)} Posterior of dominant eigenvalue\par\vspace{1ex}
  \end{minipage}
  \caption{Inference of an LTI system using the canonical form ($T=400$, $d_x = 2$). Panel (a) shows pairwise marginals of the canonical parameters \(\Theta_c\), revealing a well-behaved, unimodal posterior. True values and MAP estimate marked. Panel (b) shows the posterior distribution of the dominant complex eigenvalue pair, plotted in the complex plane, derived directly from samples of \(\Theta_c\).}
  \label{fig:canonical_analysis}
\end{figure*}
In contrast, the standard parameterization \(\Theta_s\) produces complex posterior geometries due to non-identifiability. \Cref{fig:parameter_distribution_comparison}(a) visualizes the \(\Theta_s\) posterior for the same data realization as \Cref{fig:canonical_analysis}; strong correlations and distinct modes are clearly visible, arising directly from the state-basis equivalence symmetry (\Cref{thm:isomorphism_general}).
This multimodal posterior structure,
corresponding implicitly to different choices of state transformation \(T_c\), 
hinders efficient MCMC sampling and complicates the interpretation of parameter estimates and uncertainties, compared to the unimodal structure observed for the posterior on \(\Theta_c\).

As discussed in \Cref{sec:representative_completeness}, the distribution over standard parameters \(p(\Theta_s \mid u_{[T]}, y_{[T]} )\) is intrinsically linked to the canonical posterior \(p(\Theta_c \mid u_{[T]}, y_{[T]})\) via similarity transformations \(T \in \mathrm{GL}(d_x)\). Formally constructing \(p(\Theta_s \mid u_{[T]}, y_{[T]})\) by averaging over the pushforward distributions induced by all such transformations is ill-defined? due to the non-compactness of \(\mathrm{GL}(d_x)\) and the lack of a suitable uniform measure.
To visualize the equivalences induced by what is nonetheless a large set of similarly transformations, \Cref{fig:parameter_distribution_comparison}(b) employs transformations $T$ restricted to the compact orthogonal subgroup \(\mathrm{O}(d_x) \subset \mathrm{GL}(d_x)\), which possesses a unique uniform Haar measure.
%
%
%
This visualization is generated by taking samples \(\Theta_c^{(i)}\) from the canonical posterior, drawing random orthogonal matrices \(Q^{(i)}\) uniformly from \(\mathrm{O}(d_x)\) (achieved by applying the QR decomposition to matrices with i.i.d.\ standard Gaussian entries), applying the corresponding similarity transformation \(\Psi_{Q^{(i)}}(\Theta_c^{(i)})\) to obtain samples in the standard parameter space, and then plotting the resulting empirical distribution.
The distribution shown in \Cref{fig:parameter_distribution_comparison}(b) reveals the intricate, non-elliptical geometry and strong parameter correlations introduced by this set of orthogonal transformations. Though \Cref{fig:parameter_distribution_comparison}(a) and \Cref{fig:parameter_distribution_comparison}(b) are not entirely the same, since the latter considers a more restricted set of transformations, they show a remarkable qualitative similarity for many marginals.
These structures contrast with the typically unimodal geometry of the canonical posterior, in \Cref{fig:canonical_analysis}(a).
\begin{figure*}[htbp]
  \centering
  \begin{minipage}{0.48\textwidth}
    \centering
    \includegraphics[width=\textwidth]{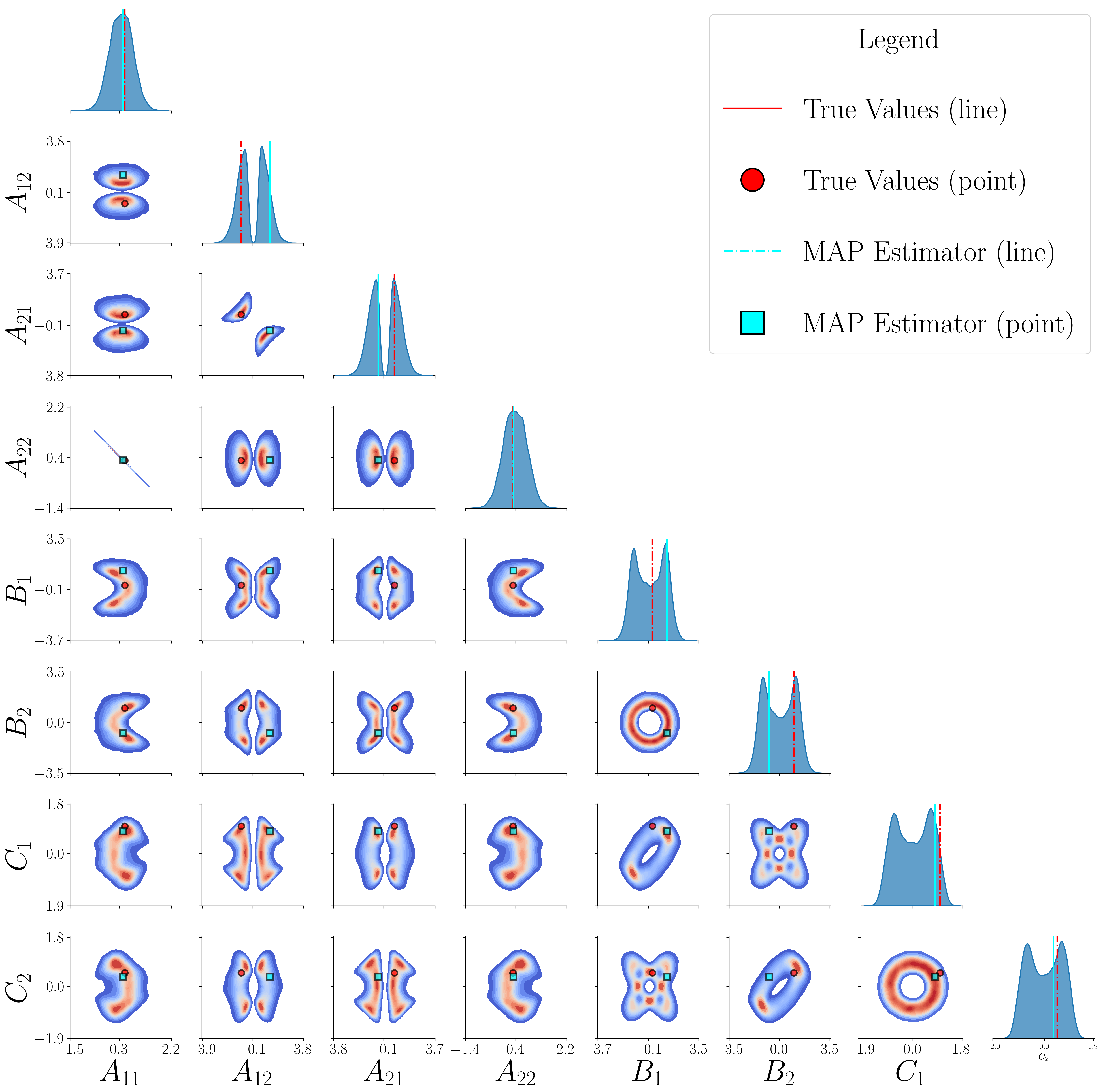}\vspace{1ex}
    \centering\textbf{(a)} Standard Parameter Space (\(\Theta_s\))\par\vspace{1ex}
  \end{minipage}
  \hfill
  \begin{minipage}{0.48\textwidth}
    \centering
    \includegraphics[width=\textwidth]{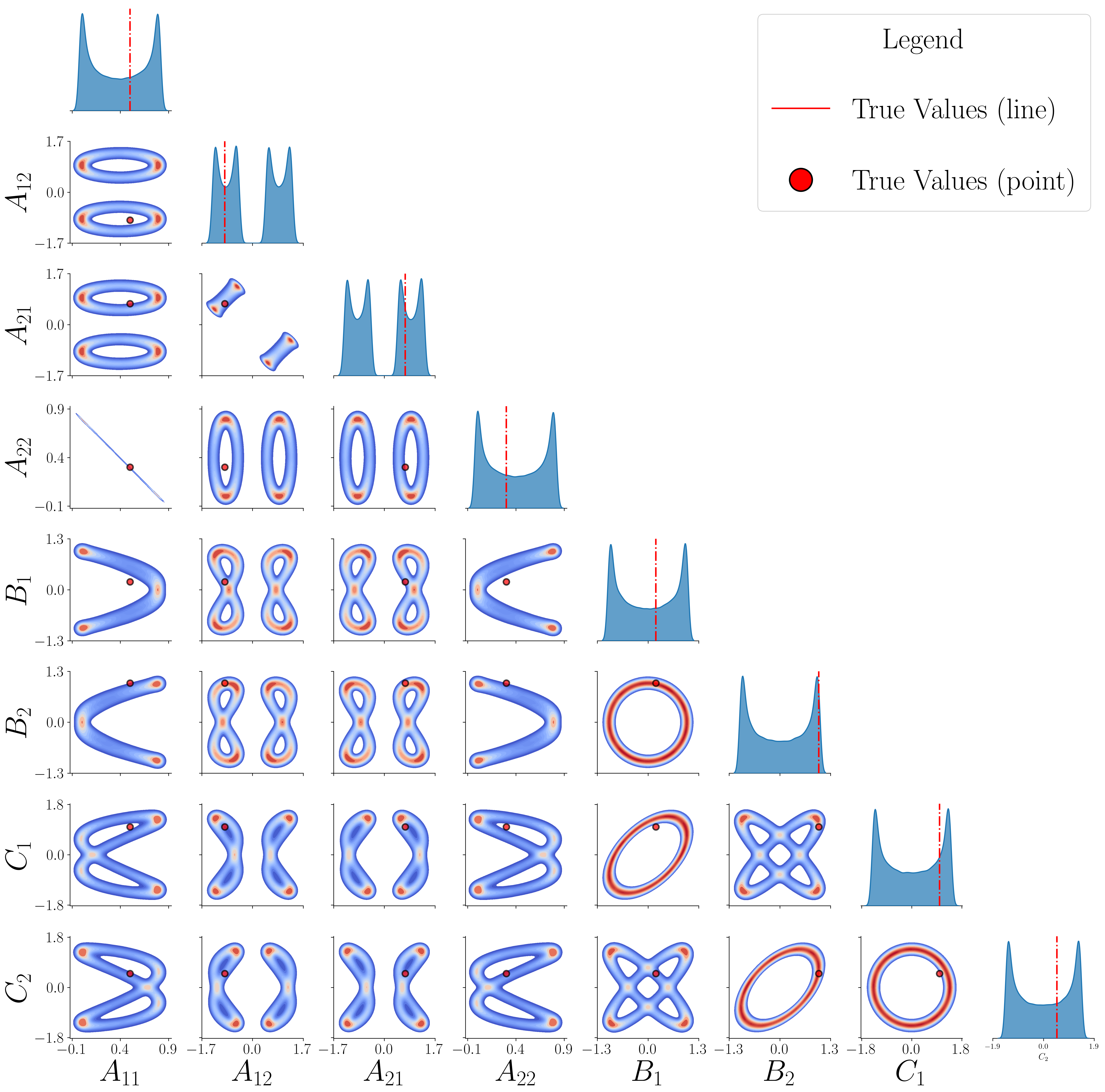}\vspace{1ex}
    \centering\textbf{(b)} Orthogonally Transformed Space\par\vspace{1ex}
  \end{minipage}
  \caption{Visualization of posterior geometry ($T=400$ experiment). Panel (a) shows the pair plot for standard parameters \(\Theta_s\), illustrating correlations and potential multimodality. True values and MAP estimate marked. Panel (b) visualizes the posterior pushed forward from the canonical space via random orthogonal transformations, highlighting the complex structure of the equivalence class.}
  \label{fig:parameter_distribution_comparison}
\end{figure*}

\subsection{Computational efficiency and MCMC performance}
\label{sec:eff_perf}

The posterior geometry induced by canonical forms translates directly into improvements in MCMC sampling efficiency and overall performance, compared to standard parameterizations.

We quantify this efficiency gain across a range of conditions by performing inference using both canonical and standard forms on 20 randomly generated, well-conditioned LTI systems, generated via the procedure in \Cref{sec:exp_setup_revised} with $d_x=2$, process noise \(\sigma_\Sigma = 0.3\), and measurement noise \(\sigma_\Gamma = 0.5\). Our experiments vary the length of the data trajectory produced by each system from $T=50$ to $T=1250$.
\Cref{fig:canonical_vs_standard_efficiency} displays the effective sample size per second (ESS/s) of MCMC sampling, averaged across all parameters and systems, for both parameterizations. The results demonstrate the efficiency of the canonical form. Notably, this efficiency gap tends to widen as \(T\) increases, to nearly three orders of magnitude; this highlights the computational advantage conferred by the identifiable and lower-dimensional canonical parameterization \(\Theta_c\).
\begin{figure}[htbp]
    \centering
    \includegraphics[width=0.8\linewidth]{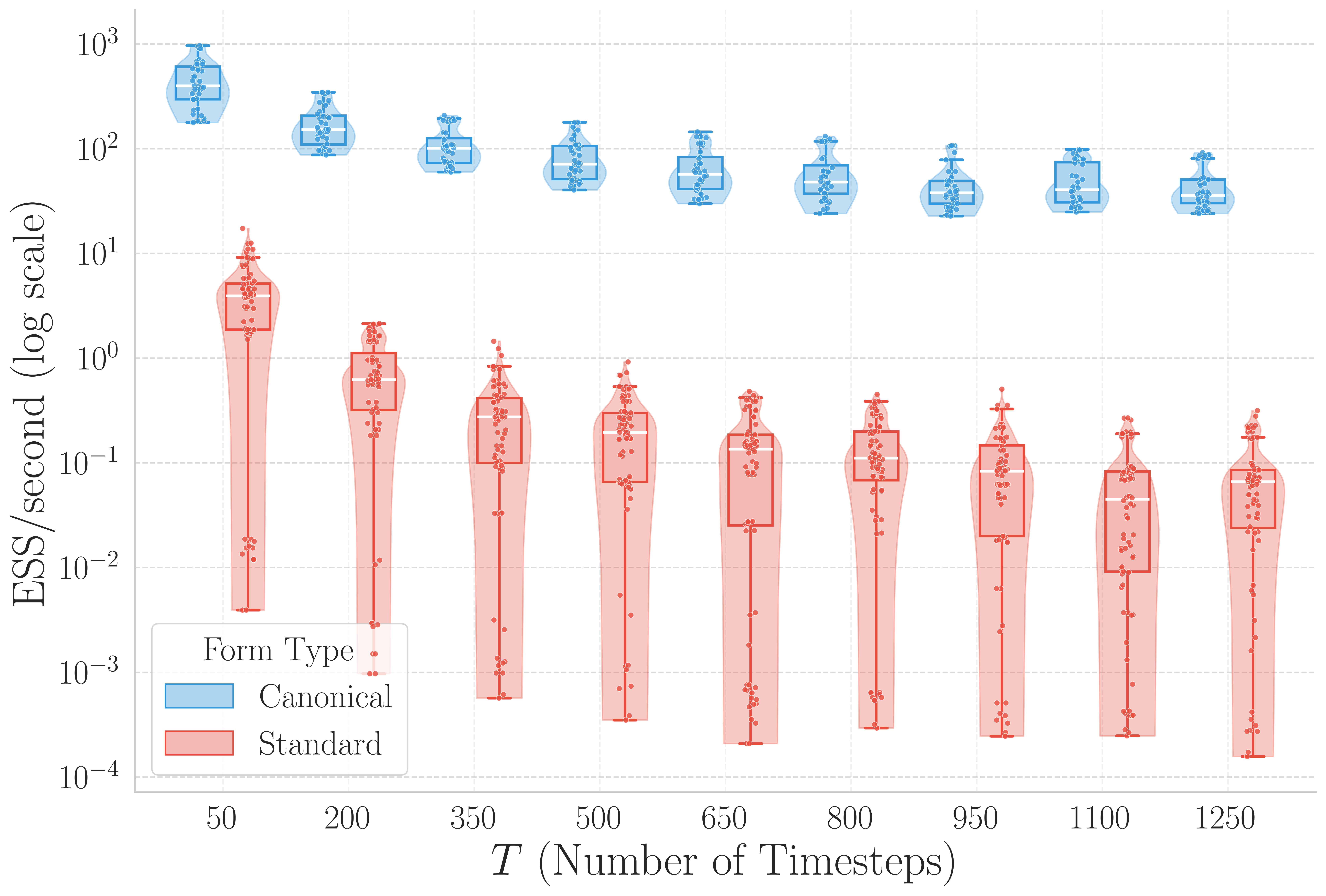} 
    \caption{Effective sample size per second, in log-scale, for MCMC sampling; comparison between canonical and standard inference approaches. Figures show a distribution of results for 20 randomly sampled systems, with \(\sigma_\Sigma=0.3\) and  \( \sigma_\Gamma=0.5\), as a function of trajectory length \(T\). Higher values indicate greater efficiency.}
    \label{fig:canonical_vs_standard_efficiency}
\end{figure}
The ESS per second is a useful aggregate measure of MCMC sampling efficiency, but other diagnostics reveal improved posterior exploration as well. Trace plots (see \Cref{fig:high-dim-grid}(d) for a $d_x=8$ system) reveal better mixing and faster convergence for canonical parameters \(\Theta_c\), while chains for standard parameters \(\Theta_s\) exhibit slower mixing and jumps between distinct modes corresponding to equivalent system representations.

Now focusing only on the canonical parameterization, we evaluate how the computational efficiency and accuracy of inference depend on the structure of the underlying true (data-generating) system. Specifically, we consider an ``easy'' LTI system, constructed to have well-balanced controllability and observability Gramians (procedure described in \Cref{appendix_e}),  and a ``hard'' system, constructed to have ill-conditioned Gramians. For the latter, we create a system such that the dominant eigenvalue ratio is higher than $0.95$. 
\Cref{fig:easy-vs-hard} presents results for 15 different noise realizations of the ``easy'' and ``hard'' system. \Cref{fig:easy-vs-hard}(a) shows that inference of the ``easy'' system achieves higher ESS/s than inference of the ``hard'' system. \Cref{fig:easy-vs-hard}(b) assesses the accuracy of the resulting posterior mean estimates of the canonical parameters \(\Theta_c\); we plot the the MSE of these estimators alongside the MSE of a baseline Ho--Kalman estimate (HKE), transformed to canonical coordinates for comparison. Parameter estimates for the poorly conditioned system have somewhat larger errors than estimates for the well-conditioned system, but in both cases the posterior mean estimate is significantly more accurate than the HKE. These results demonstrate the robustness of the canonical Bayesian approach across varying levels of system conditioning.
\begin{figure}[htbp]
    \centering
    \begin{minipage}[b]{0.49\textwidth}
        \centering
        \includegraphics[width=\linewidth]{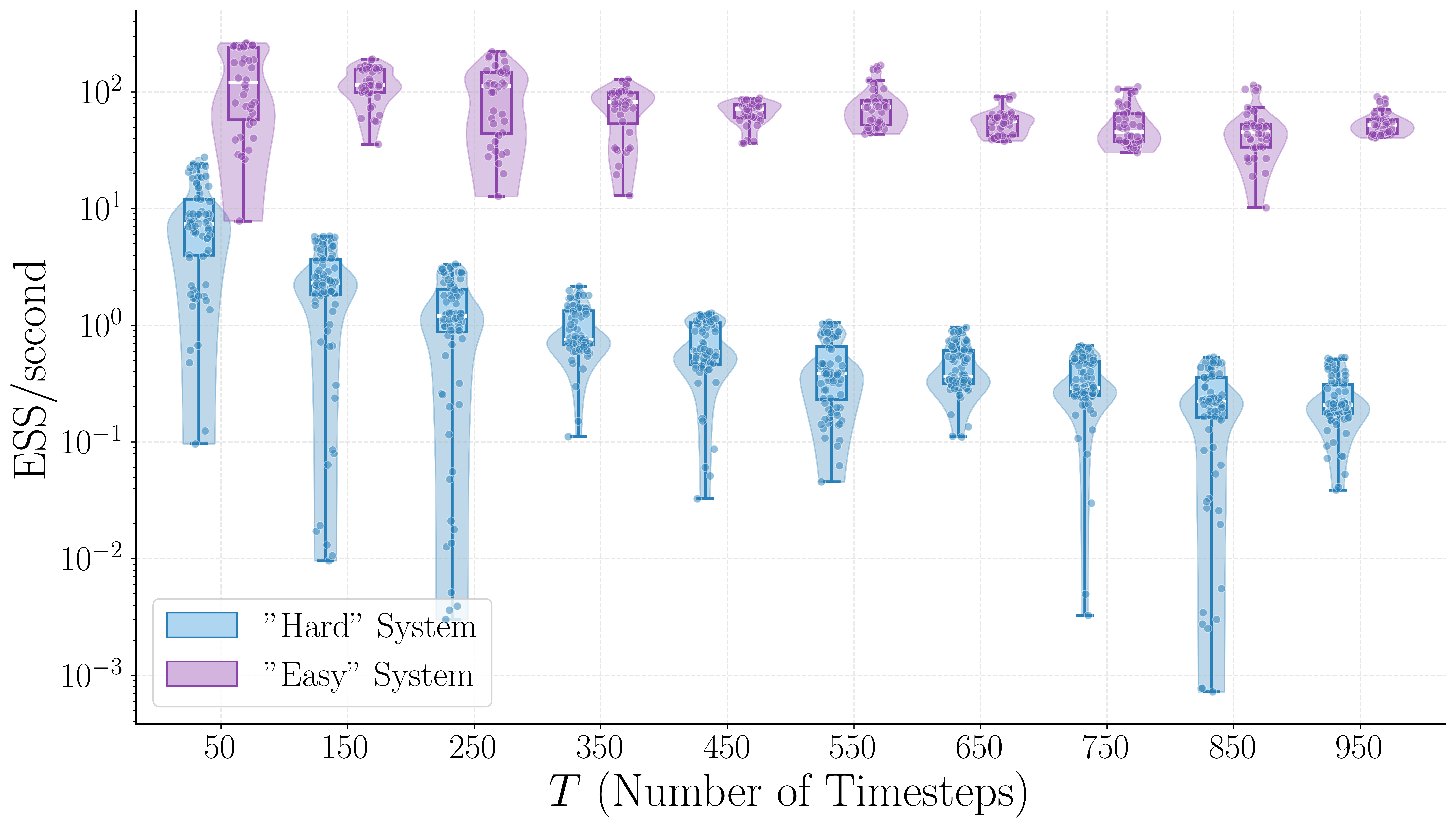}
        \vspace{0.5ex}\centering\textbf{(Left)} ESS/second comparison\par\vspace{1ex}
    \end{minipage}
    \hfill
    \begin{minipage}[b]{0.49\textwidth}
        \centering
        \includegraphics[width=\linewidth]{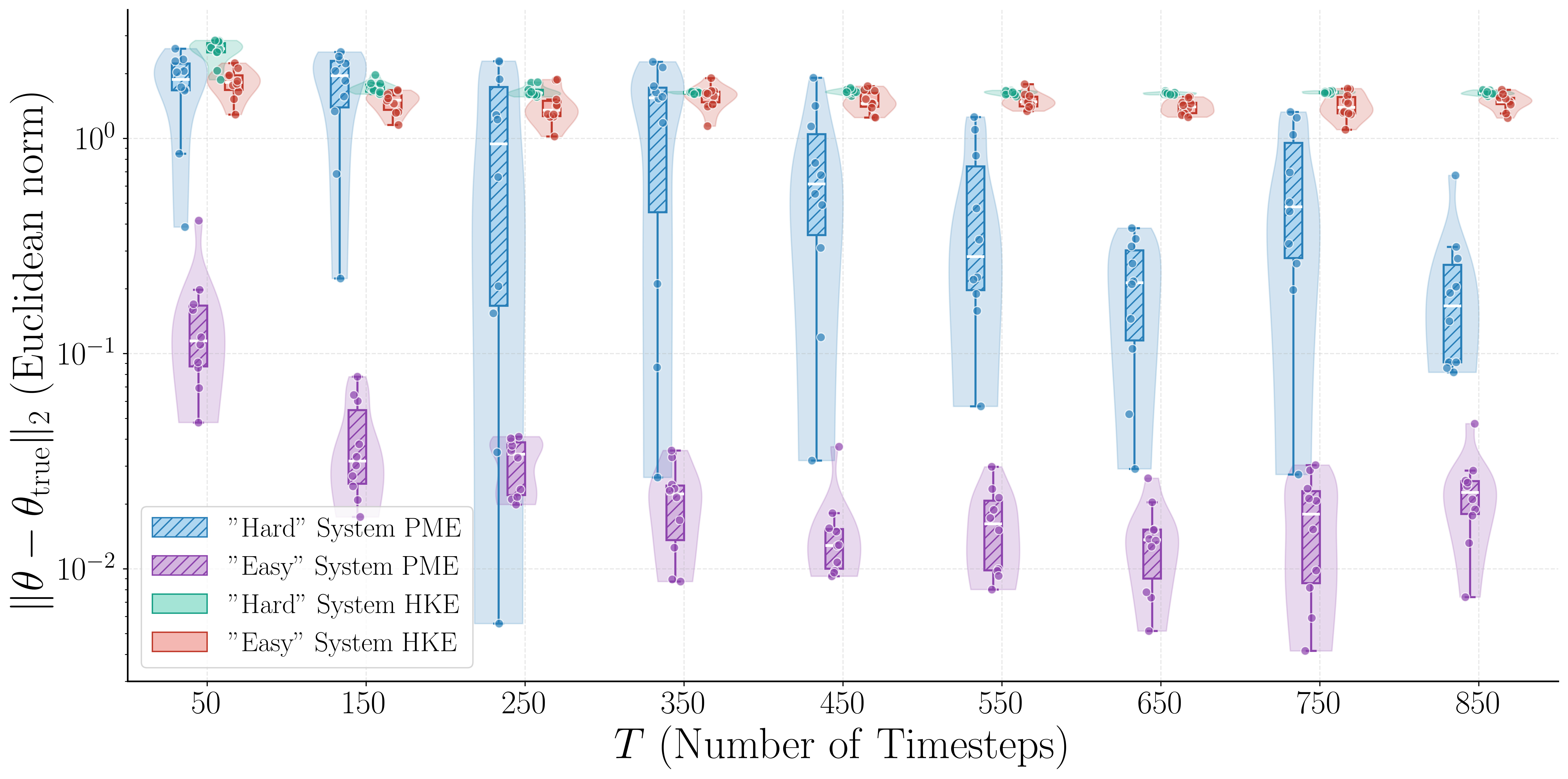}
        \vspace{0.5ex}\centering\textbf{(Right)} Estimation accuracy (MSE)\par\vspace{1ex}
    \end{minipage}
    \caption{Comparing the computational efficiency and accuracy of inference for representative ``easy'' (well-balanced Gramians) and ``hard'' (ill-conditioned Gramians) LTI systems; scatter shows results for 15 independent noise realizations. Left: ESS/second. Right: MSE of canonical parameter estimates, comparing the posterior mean (PME) to the Ho--Kalman estimate (HKE).}
    \label{fig:easy-vs-hard}
\end{figure}

We have also conducted experiments to assess how the execution time of inference, in both standard and canonical forms, scales with the system dimension $d_x$; see \Cref{appendix_f:scalability_with_dimension}. Here we observe that Bayesian inference in the canonical parameterization exhibits significantly better scaling than inference in the standard parameterization.


\subsection{Recovering quantities of interest}
\label{sec:acc_prior}
Now we evaluate the accuracy of QoI predictions produced by our methods and the role of prior information, particularly in data-limited scenarios.

\paragraph{Accuracy of point estimates and predictions}
In \Cref{fig:trajectory_comparison}, we assess the ability of our inference methods to recover system behavior by extracting posterior predictive estimates of the trajectory $y_{[T]}$, for a fixed ground-truth LTI system ($d_x =2$) simulated with process noise \(\sigma_\Sigma = 0.3\) and measurement noise \(\sigma_\Gamma = 0.5\).
Even with relatively limited data ($T=50$), plug-in estimates of the trajectory based on the MAP estimate of both the canonical (\(\Theta_c\), \Cref{fig:trajectory_comparison}(a)) and standard (\(\Theta_s\), \Cref{fig:trajectory_comparison}(d)) parameters (recall \eqref{eq:plugin}) capture the true trajectory reasonably well, despite noise-induced fluctuations. The shaded regions are centered 95\% credible regions of the marginal posterior predictive distribution at each time $t$. Moving to $T=400$ (\Cref{fig:trajectory_comparison}(b) and (e)), these regions contract and both MAP estimates of the trajectory improve in quality. These results suggest that 
optimization can find a representative point within the correct equivalence class despite the parameterization's identifiability issues. 

Differences emerge when using the posterior mean parameter estimate, however. \Cref{fig:trajectory_comparison}(c) and (f) show (i) the mean of the posterior predictive over trajectories, (ii) the plug-in estimate produced by the MAP, as before; and (iii) the plug-in trajectory estimate produed by the PME of the parameters, whether $\Theta_c$ or $\Theta_s$. While the plug-in estimate based on the PME of $\Theta_c$ appears accurate, the plug-in estimate based on the PME of $\Theta_s$  performs poorly, yielding an unrepresentative, nearly constant trajectory.  This failure is due to the complex multi-modal structure of the posterior of $\Theta_s$; its PME lies between modes, in a region of low posterior probability that does not correspond to any single valid system realization. 
This result underscores the advantage of the canonical form: identifiability leads to a unimodal posterior for which the PME serves as a robust and interpretable summary statistic.
\begin{figure*}[htbp]
  \centering
  \begin{minipage}{0.48\textwidth}
    \centering
    \includegraphics[width=\textwidth]{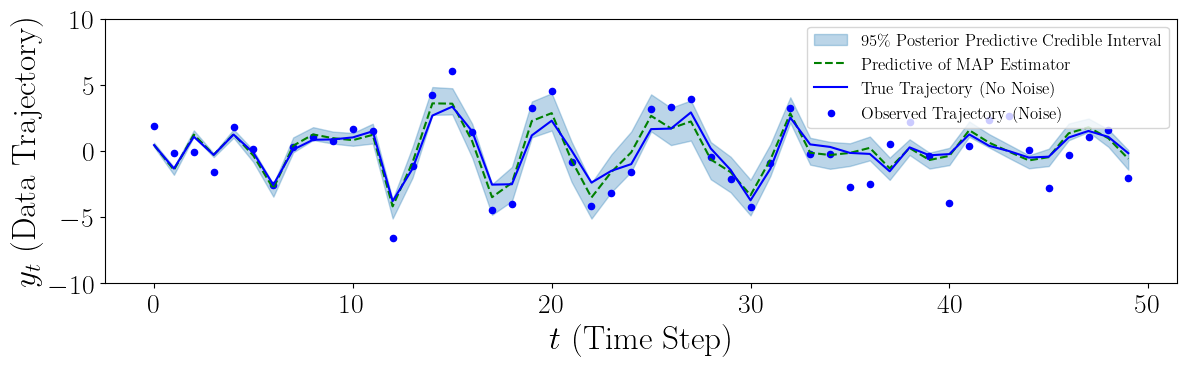}\\
    \vspace{0.5ex}\centering\textbf{(a)} Canonical, $T=50$, MAP\par\vspace{1ex}
    \includegraphics[width=\textwidth]{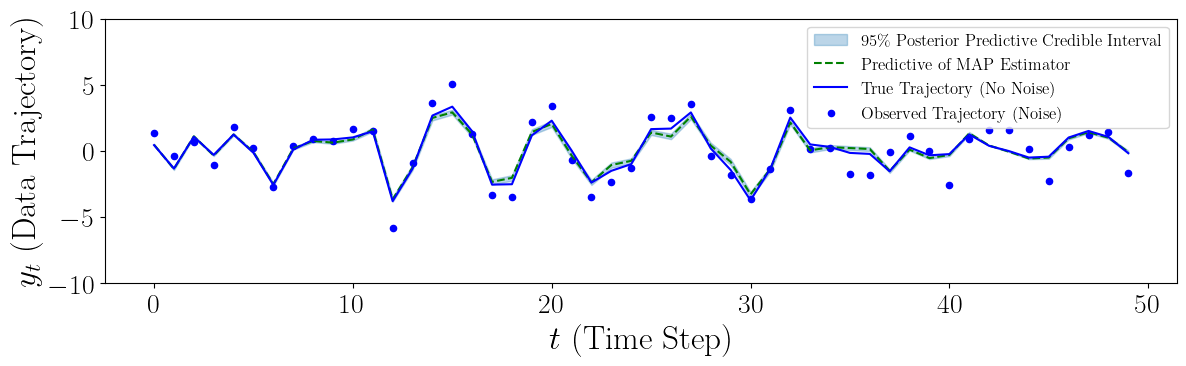}\\
    \vspace{0.5ex}\centering\textbf{(b)} Canonical, $T=400$, MAP\par\vspace{1ex}
    \includegraphics[width=\textwidth]{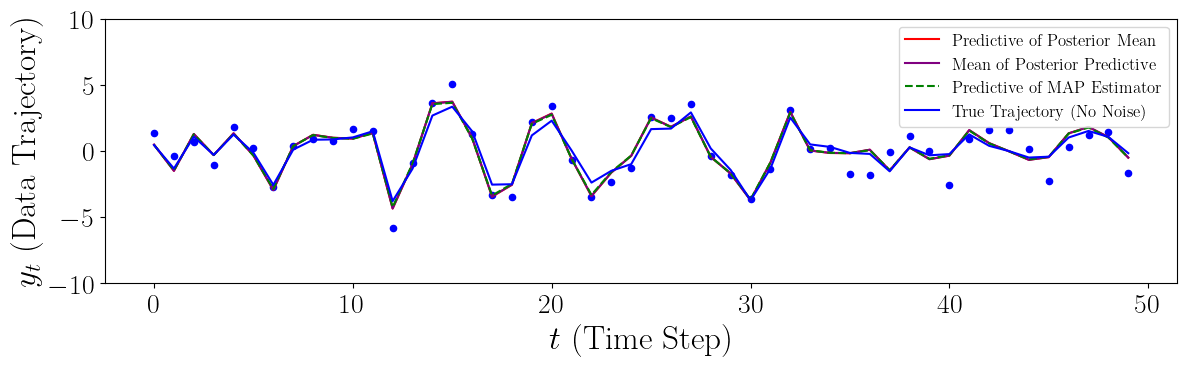}\\
    \vspace{0.5ex}\centering\textbf{(c)} Canonical, $T=50$, PME\par\vspace{1ex}
  \end{minipage}
  \hfill
  \begin{minipage}{0.48\textwidth}
    \centering
    \includegraphics[width=\textwidth]{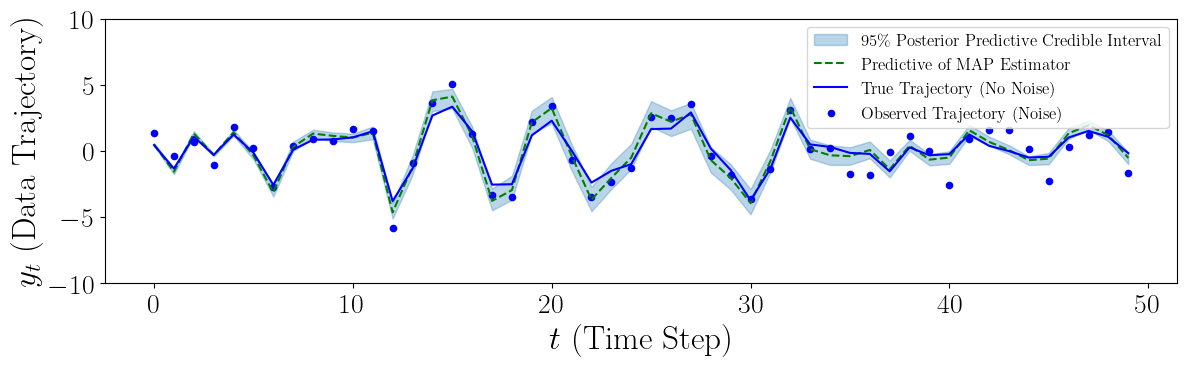}\\
    \vspace{0.5ex}\centering\textbf{(d)} Standard, $T=50$, MAP\par\vspace{1ex}
    \includegraphics[width=\textwidth]{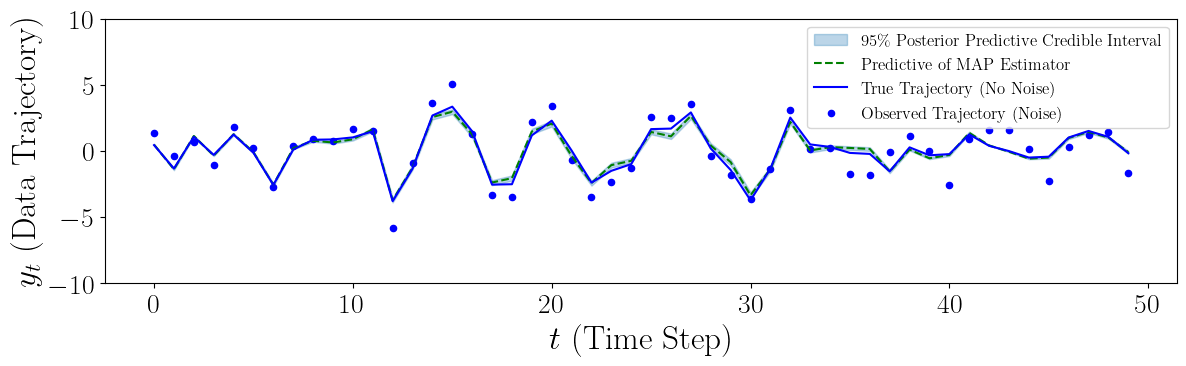}\\
    \vspace{0.5ex}\centering\textbf{(e)} Standard, $T=400$, MAP\par\vspace{1ex}
    \includegraphics[width=\textwidth]{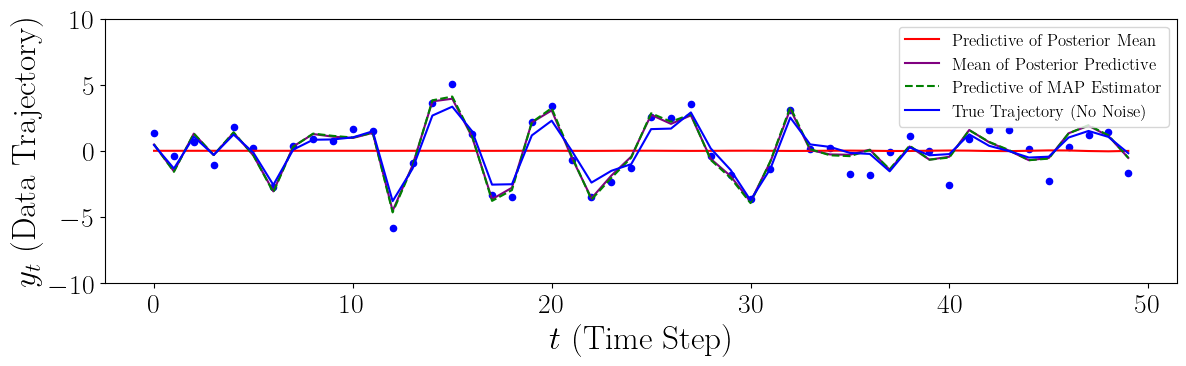}\\
    \vspace{0.5ex}\centering\textbf{(f)} Standard, $T=50$, PME\par\vspace{1ex}
  \end{minipage}
  \caption{Comparison of inferred trajectories using canonical (\(\Theta_c\), left panels) and standard (\(\Theta_s\), right panels) parameterizations for a system with \(\sigma_\Sigma=0.3, \sigma_\Gamma=0.5\). Rows compare MAP estimates at $T=50$ (top), MAP estimates at $T=400$ (middle), and several estimators at $T=50$ (bottom). All plots show the first 50 timesteps.
  }
  \label{fig:trajectory_comparison}
\end{figure*}

\paragraph{Prior sensitivity analysis}
To investigate the influence of the prior, we perform a Bayes risk-type analysis as follows. 
%
%
We generate 30 ground-truth systems \(\Theta_{c,0}\) by sampling eigenvalues from the restricted region prior as defined in \Cref{prior_restricted_magnitude} and applying random orthogonal transformations. Data trajectories up to $T=1250$ steps are simulated with fixed measurement noise (\(\sigma_\Gamma=0.5\)) and two levels of process noise: none (\(\sigma_\Sigma=0.0\)) and moderate (\(\sigma_\Sigma=0.3\)).
Bayesian inference is performed using the canonical form \(\Theta_c\) under three different prior models:
the restricted region prior, which we will call the ``informative prior'';
and two ``weakly informative priors,'' the polar coordinate prior and the implied prior from uniform coefficients as defined in \Cref{prior_stable_uniform}.
For an additional baseline comparison, we use the HKE \cite{oymak2019}, which provides a non-Bayesian point estimate and involves no prior.

We first assess the accuracy of estimates $\widehat{H} = H(\widehat{\Theta}_c)$ of the Hankel matrix, computed by plugging in a point estimate $\widehat{\Theta}_c$ of $\Theta_c$---either the PME or the HKE. For each randomly generated system, we compute the Frobenius norm of the error in the Hankel matrix derived from the PME, \( \| H (\widehat{\Theta}_c) - H(\Theta_{c,0})\|_F\), and compute the median and range of these errors. We do the same for $L^2$ errors in estimates of the canonical parameters, \(\| \widehat{\Theta}_c - \Theta_{c,0}\|_2^2\), all over a range of \(T\).
\begin{figure}[htbp]
    \centering
    \begin{minipage}[b]{0.48\textwidth}
        \centering
        \includegraphics[width=\linewidth]{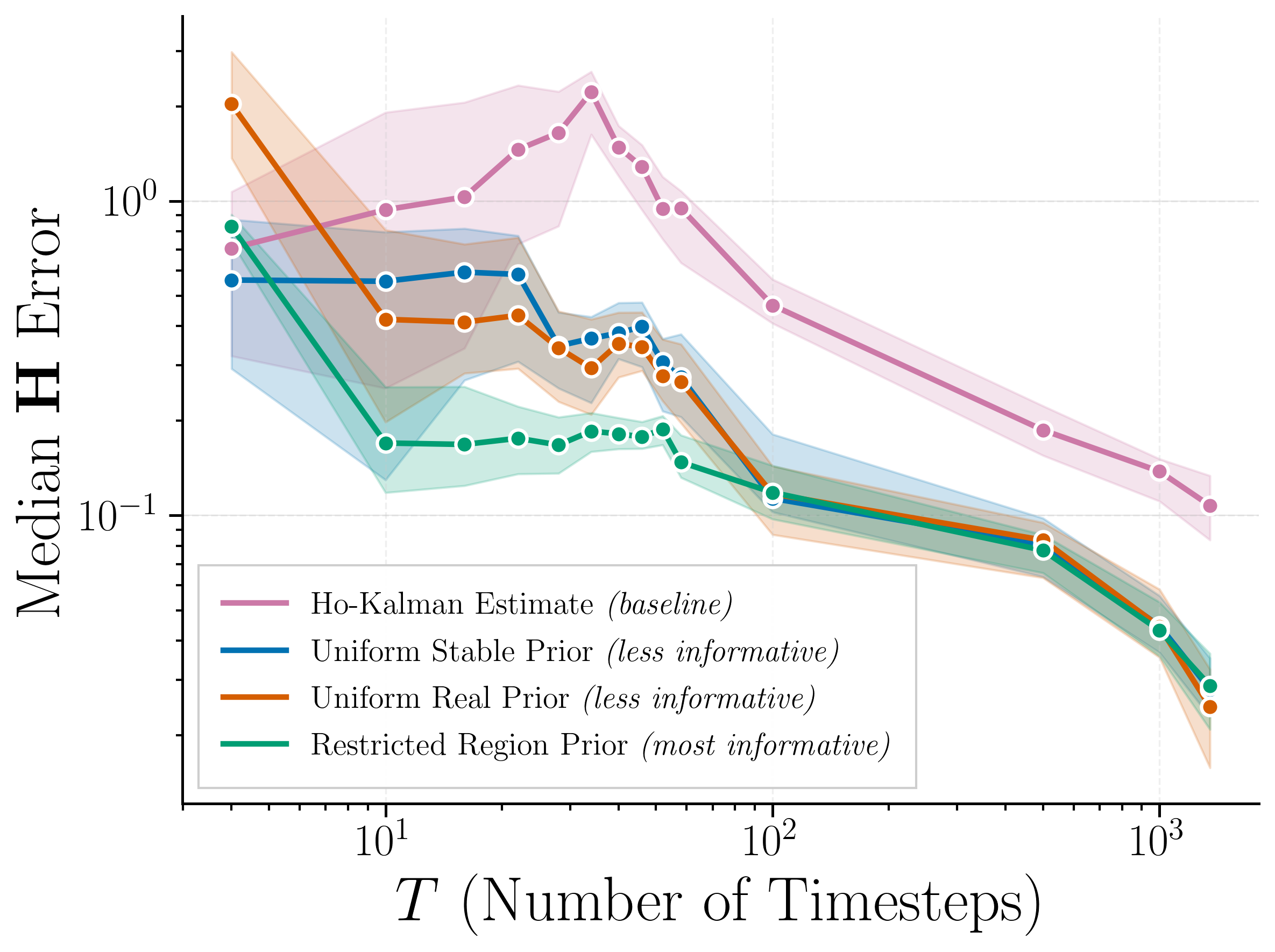}
        \vspace{0.5ex}\centering\textbf{(a)} Hankel error, \(\sigma_\Sigma=0.0\)\par
    \end{minipage}
    \hfill
    \begin{minipage}[b]{0.48\textwidth}
        \centering
        \includegraphics[width=\linewidth]{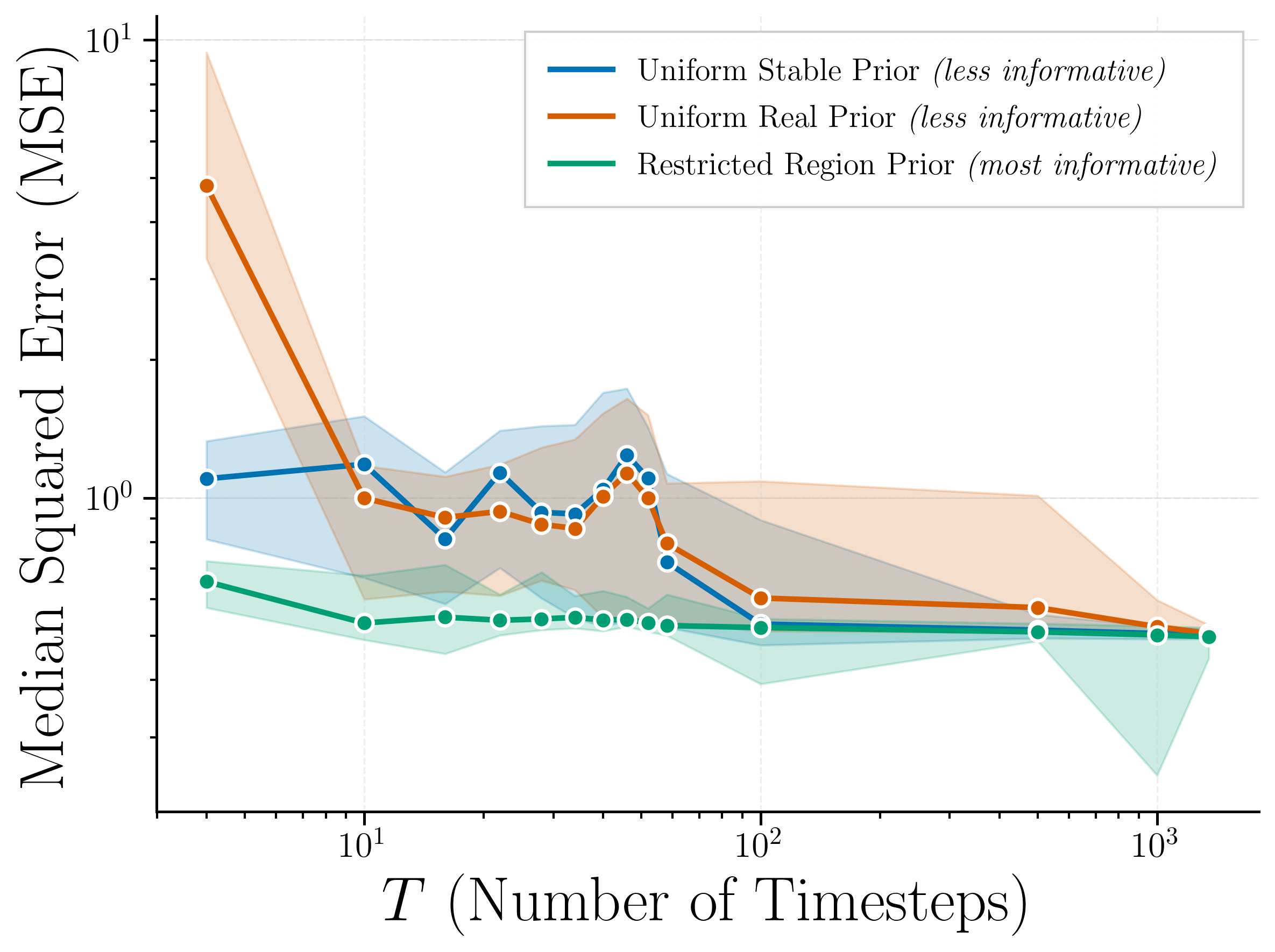}
        \vspace{0.5ex}\centering\textbf{(b)} MSE, \(\sigma_\Sigma=0.0\)\par
    \end{minipage}

    \vspace{0.5em} 

    \begin{minipage}[b]{0.48\textwidth}
        \centering
        \includegraphics[width=\linewidth]{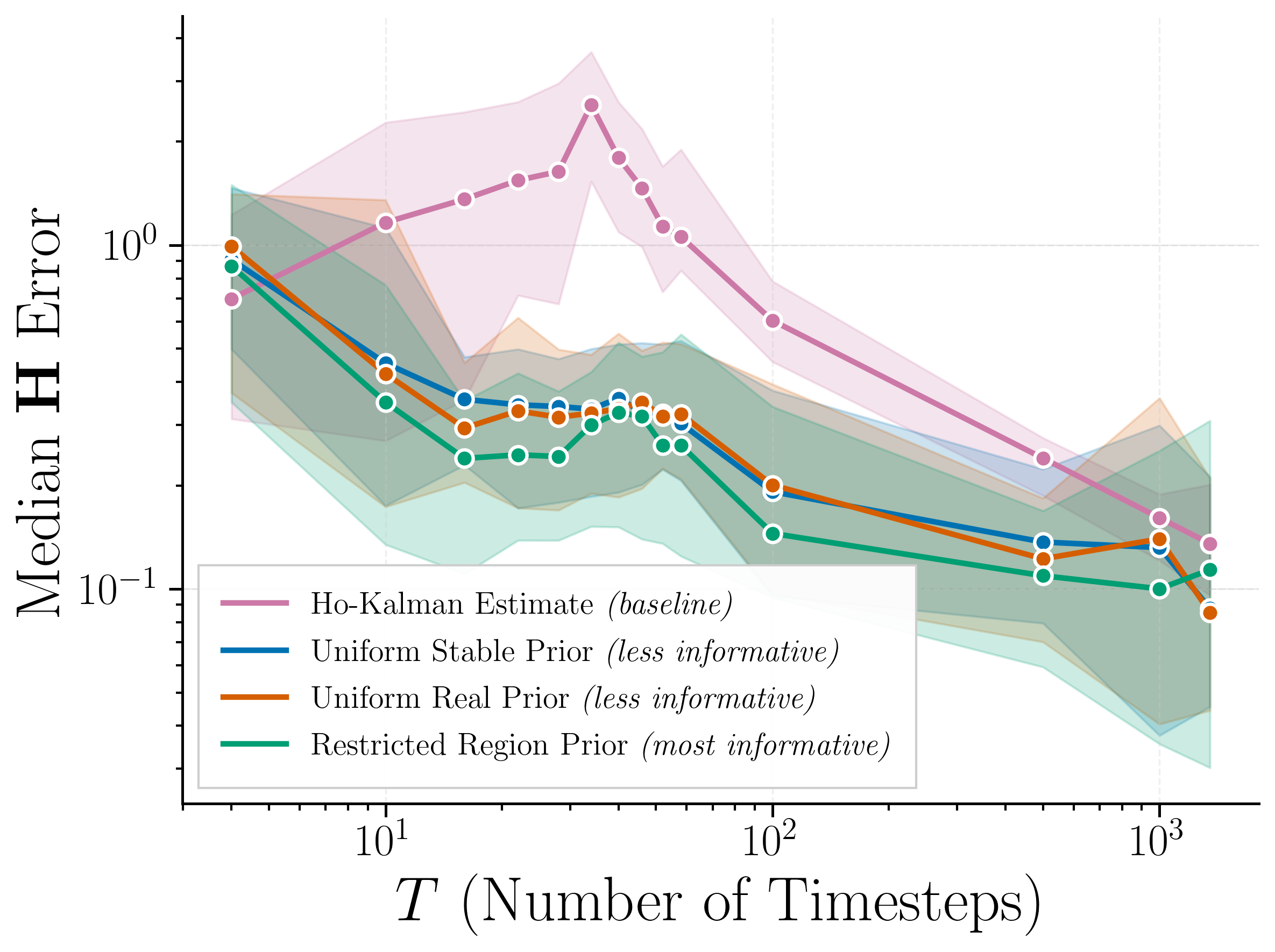}
        \vspace{0.5ex}\centering\textbf{(c)} Hankel error, \(\sigma_\Sigma=0.3\)\par
    \end{minipage}
    \hfill
    \begin{minipage}[b]{0.48\textwidth}
        \centering
        \includegraphics[width=\linewidth]{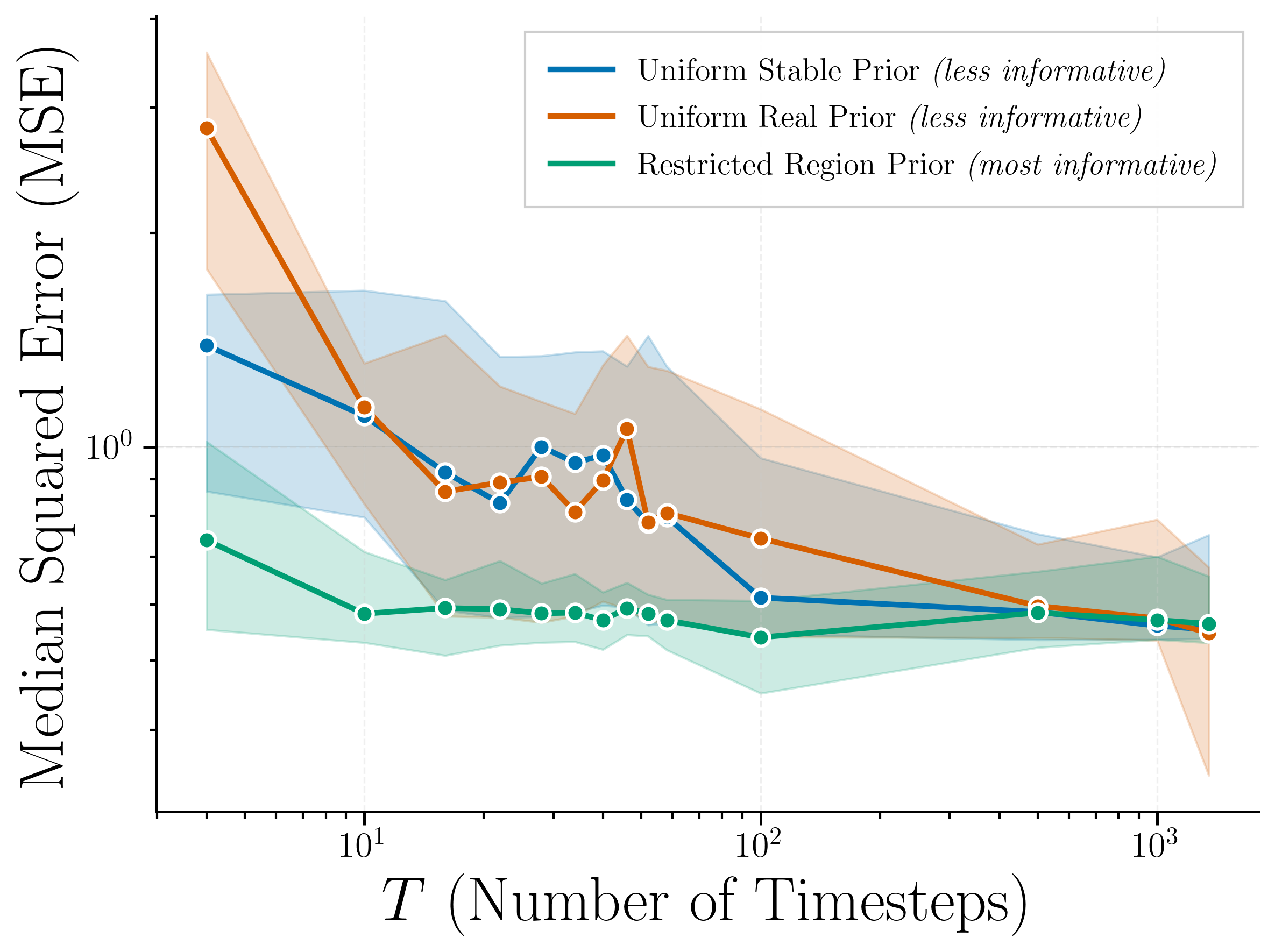}
        \vspace{0.5ex}\centering\textbf{(d)} MSE, \(\sigma_\Sigma=0.3\)\par
    \end{minipage}

    \caption{Prior sensitivity analysis: Median errors (over 30 systems)---of the Hankel matrix, in Frobenius norm (left column) and of the parameters, in $L^2$ norm (right column)---versus data length \(T\) (log-log scale). Comparisons are made under informative and uninformative priors against the HKE baseline, with fixed process noise (\(\sigma_\Gamma=0.5\)). Shaded regions depict the 5th--95th percentile range of these errors across the 30 systems.
    }
    \label{fig:convergence_grid}
\end{figure}
Results are shown in \Cref{fig:convergence_grid}.

As expected, for small \(T\),  the informative prior yields the lowest error, significantly outperforming less informative priors and the baseline,  highlighting the value of accurate prior information when data are scarce. The parameter MSE in \Cref{fig:convergence_grid}(b, d) consistently shows the benefit of the informative prior across both noise levels. However, the Hankel error in \Cref{fig:convergence_grid}(a, c), which relates directly to the input-output impulse response, shows a more nuanced picture. In the presence of process noise (\(\sigma_\Sigma=0.3\)), the advantage of the informative prior in terms of Hankel error diminishes at very low and very large \(T\); these estimates become comparable in error to the HKE or the weakly informative PME, while in both cases, the HKE-induced estimate is bad for the medium amounts of data. This suggests that while the informative prior leads to better parameter estimates, process noise can obscure the initial impulse response, making direct estimators such as Ho--Kalman competitive for predicting short-term input-output behavior with very limited, noisy data. As \(T\) increases, however, likelihood information dominates, all Bayesian methods converge, and they consistently outperform the baseline---demonstrating the robust behavior of the Bayesian canonical inference framework over a range of data conditions.

\subsection{Asymptotic convergence and scalability}
\label{sec:asymp_scale}
We empirically validate the consistency of parameter estimates and the convergence of the posterior distribution towards the Gaussian predicted by the BvM \Cref{thm:BvM_LTI_applicability} for identifiable canonical forms. This experiment uses 15 randomly generated ground truth systems. For each system, data are simulated with zero process noise ($\sigma_\Sigma=0$) and fixed measurement noise $\sigma_\Gamma=0.5$ over increasing trajectory lengths $T$. From these simulated data, we determine the MAP estimate of the canonical parameters, denoted as $\Theta_c^{\text{MAP}}$.

A key aspect of BvM is the role of the FIM. We compare two FIM calculations, both evaluated at the MAP estimate $\Theta_c^{\text{MAP}}$ for each ground truth system: (1) $I_T^{\text{analytic}}(\Theta_c^{\text{MAP}})$ and (2) $I_T^{\text{numeric}}(\Theta_c^{\text{MAP}})$. $I_T^{\text{analytic}}$ is computed using the recursive formulas derived in \Cref{prop:fisher_LTI}. These formulas are exact for the $\sigma_\Sigma=0$ case considered here. $I_T^{\text{numeric}}$ is a numerical approximation of the expected FIM. It is obtained by first generating $M=100$ independent noisy trajectory realizations using the ground truth parameters. A single MAP estimate, $\Theta_c^{\text{MAP}}$, is computed from one of these $M$ realizations (which we designate as the reference realization). Then, the observed FIM is computed for each of the $M$ realizations by applying automatic differentiation (AD) to the log-likelihood (using \texttt{JAX}), with each observed FIM being evaluated at this single $\Theta_c^{\text{MAP}}$ derived from the reference realization. Finally, these $M$ observed FIMs are averaged. This AD-based approach provides an efficient way to estimate the expected FIM and readily extends to scenarios with non-zero process noise where analytical formulas become more complex (see \Cref{appendix_fisher_information_and_BvM}).

\Cref{fig:fisher_information_convergence} provides empirical evidence for the predictions of the BvM theorem.
The insets shown at $T=\{105, 305, 495\}$ offer visual validation, showing that the posterior distribution (approximated via MCMC samples) converges to the predicted Gaussian as the trajectory length $T$ increases. These posterior contours not only become elliptical but also align with the confidence ellipsoid derived from the FIM (e.g., $I_T^{\text{analytic}}$). This confirms both the asymptotic normality and the fact that the inverse FIM correctly describes the asymptotic posterior covariance.

The main plot underlines a critical prerequisite for BvM: the consistency of the MAP estimator. It shows the log-ratio of the confidence ellipsoid volumes from the two FIMs converging to zero with increasing observation sequence. This occurs precisely because the MAP estimate $\Theta_c^{\text{MAP}}$ converges to the ground truth, causing the two distinct FIM calculations to agree in the large-$T$ limit. Further analyses, though not plotted, confirm this consistency, showing that the MSE of the posterior mean decreases appropriately with $T$ and that the empirical coverage of FIM-based confidence intervals approaches nominal levels.

For short trajectories (small $T$ ), the MAP estimate $\Theta_c^{\mathrm{MAP}}$ is derived from a single, noisy realization and is therefore a poor estimate of the ground truth. Because the two FIM calculations are based on expectations over different data-generating distributions, they are fundamentally calculating different objects and are not expected to agree. As the trajectory length $T$ increases, however, the MAP estimator becomes consistent. As the evaluation point $\Theta_c^{\text {MAP }}$ approaches the data-generating ground truth, the two FIM calculations converge to the same quantity: the FIM at the ground truth. This convergence is precisely what the plot demonstrates and is a key piece of evidence for the consistency of the estimator.

\begin{figure*}[htbp]
  \centering
  \includegraphics[width=0.9\textwidth]{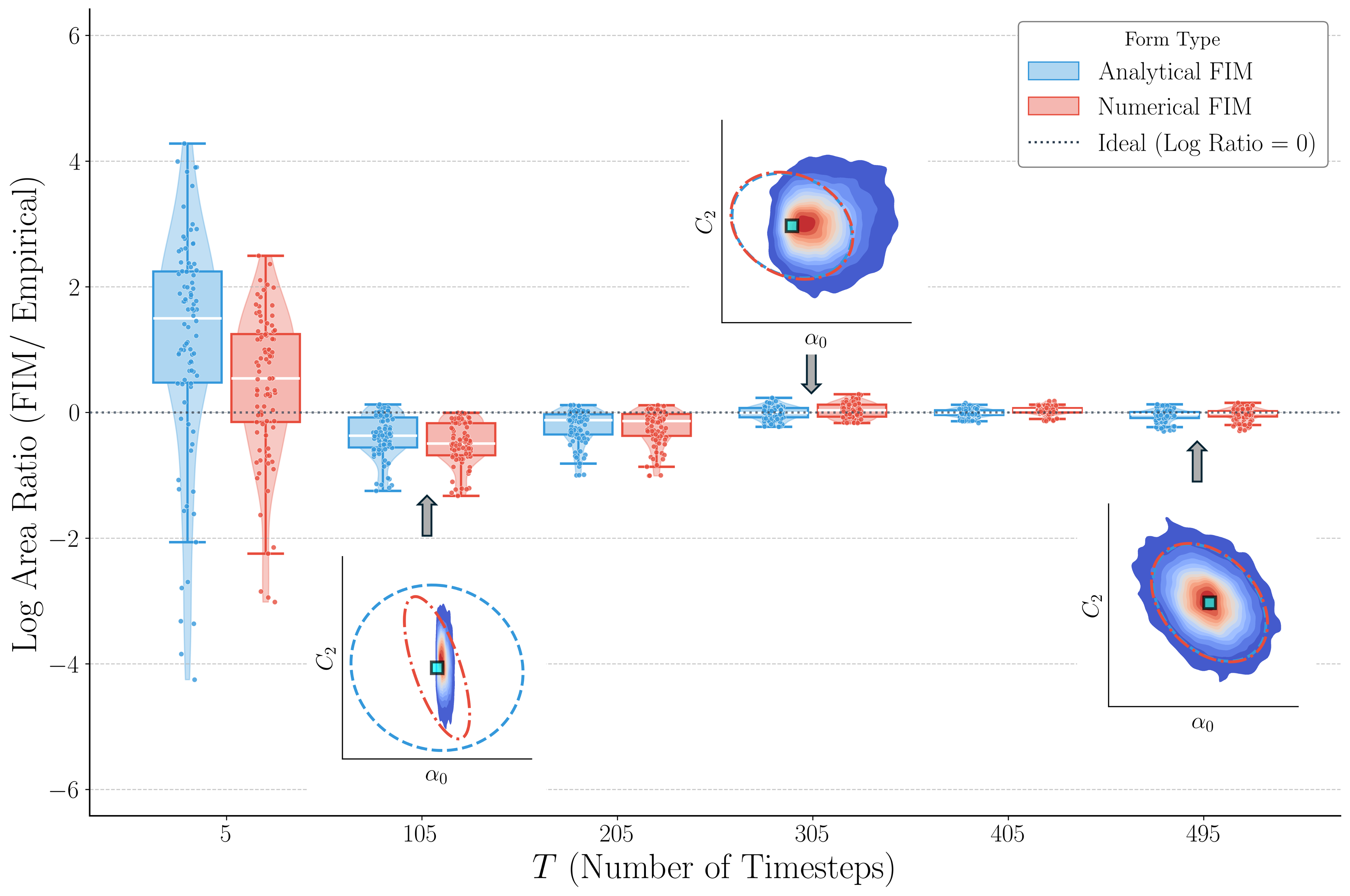} 
  \caption{Consistency and BvM convergence analysis (\(\sigma_\Sigma=0\), \(\sigma_\Gamma=0.5\)). Main plot: Boxplots of the log-ratio of confidence ellipsoid volumes (derived from numerical AD-based FIM versus analytical recursive FIM) over time \(T\) across 15 systems. Convergence to zero indicates agreement. Insets: Visualization of posterior density contours versus analytical and derived FIM ellipsoid for parameters \(\alpha_0, C_2\) at $T=\{105, 305, 495\}$ for the median system, illustrating convergence towards the appropriate Gaussian.}
  \label{fig:fisher_information_convergence}
\end{figure*}

\section{Conclusions and future work}
\label{sec:conclusion}

We have presented an efficient Bayesian framework for identifying LTI dynamical systems from observed trajectory data, leveraging identifiable state-space canonical forms.
We have demonstrated both theoretically and empirically the significant advantages of the canonical approach over inference based on standard parameterizations.

\paragraph{Summary of contributions}
Our core contribution is the integration of canonical state-space representations with principled Bayesian inference to resolve the intrinsic non-identifiability of linear dynamical systems. We prove that inference on these identifiable canonical forms yields the same posterior distributions for invariant quantities of interest as inference in standard but non-identifiable parameterizations (\Cref{equivalence_of_pushforwards}). We also show that this framework facilitates meaningful prior specification (e.g., directly on eigenvalues controlling dynamical properties). It also yields posterior concentration around a consistent parameter estimate and justifies Gaussian posterior approximations, as described by a Bernstein--von Mises theorem (\Cref{thm:BvM_LTI_applicability}); this is not the case for standard parameterizations. Numerical experiments show that our canonical approach consistently provides superior computational efficiency compared to standard parameterizations (\Cref{sec:numerics}).

\paragraph{Limitations}
Though it is much more computationally efficient than inference on \(\Thetastan\) (\Cref{sec:eff_perf}), our MCMC approach inherently carries a higher computational burden than non-Bayesian point estimation methods for system identification (e.g., subspace ID or Ho--Kalman estimates \cite{vanOverschee1996subspace, oymak2019}); this gap reflects the cost of full Bayesian uncertainty quantification
and the improved performance of posterior mean estimates (as observed in our numerical experiments).
We also note that computational performance may be sensitive to the specific canonical form chosen, notably for MIMO systems. 
Another key consideration is our assumption of a known state dimension \(d_x\). Determining \(d_x\) from data, in the present Bayesian setting, is a problem of Bayesian model selection. It is solved by comparing marginal likelihoods across dimensions---a principled but computationally intensive approach, which may be especially relevant if the ``practical'' system order is lower than the minimal one, given finite and limited data.

\paragraph{Extensions and future work}
%
%
Fully instantiating the proposed Bayesian inference in canonical representations for MIMO systems presents several further challenges. The intricate structure of MIMO systems complicates structural identifiability analysis, demanding more sophisticated prior specification methods that make sense of internal parameter structure. Inference schemes may then need to be tailored to specific structural indices (\Cref{remark:mimo_complexity}). Despite the non-uniqueness of canonical forms in the MIMO setting, the foundational results of this paper can still hold by making a choice and committing to one specific form. But it is natural to then consider ``hybrid'' inference methods that exploit multiple, equivalent canonical forms; these might improve sampling efficiency near controllability or observability boundaries.

Another important extension, as mentioned above, is to move beyond fixed model orders and infer the state dimension \(d_x\) itself, e.g., using Bayesian model selection methods or some approximate information criterion.
To tackle very high-dimensional systems, addressing scalability is essential; we suggest exploring variational inference methods or stochastic gradient MCMC tailored to state-space structures.
%
%
From a theoretical standpoint, a key challenge is to quantify finite-sample performance by deriving rigorous non-asymptotic bounds on posterior convergence with increasing data, moving beyond the typical Bernstein--von Mises guarantees. Finally, extending structure-informed Bayesian techniques to nonlinear systems, with a similar decomposition of the model class into identifiable and non-identifiable features, represents a considerably more challenging but highly promising direction.

\bibliographystyle{siamplain}
\bibliography{references}

\appendix
\section{Scalability with dimension}
\label{appendix_f:scalability_with_dimension}
Here we report on an additional experiment investigating how computational performance scales with system dimension $d_x$. We generate test systems with \({d_x}\) ranging from 2 to 10 using a specific procedure designed to yield balanced controllability and observability properties, crucial for fair comparison across dimensions (see \Cref{appendix_e}).

We compare the computation time for inference using standard \(\Theta_s\) and canonical \(\Theta_c\) forms on these balanced systems for fixed trajectory lengths \(T=200\) and \(T=500\). \Cref{fig:high-dim-grid} presents the timing results (normalized per evaluation where appropriate). \Cref{fig:high-dim-grid}(a) and \Cref{fig:high-dim-grid}(b) show the time required for posterior and gradient evaluations, respectively. The cost for the standard form increases markedly faster with dimension \({d_x}\) than for the canonical form. While the gradient computation for the canonical form might have slightly higher overhead at the lowest dimension (\(d_x=2\)) due to its specific structure, its superior scaling becomes evident quickly. This advantage arises because standard state-space operations (like the Kalman filter) inherently involve dense matrix calculations scaling polynomially with \({d_x}\) (e.g., \(\propto {d_x}^3\)), whereas the often sparse structure of canonical forms allows for more efficient implementations. \Cref{fig:high-dim-grid}(c) confirms this trend for the total inference time, demonstrating substantial computational savings with the canonical approach, especially for \(d_x>4\).

Beyond computation time, MCMC sampling efficiency also benefits from the canonical form in higher dimensions. \Cref{fig:high-dim-grid}(d) shows diagnostics for a \(d_x=8\) system (T=500). The trace plot for a canonical parameter (\(\alpha_0\)) exhibits good mixing and exploration of a concentrated posterior, reflected in the unimodal histogram centered near the true value (albeit with a slight offset, potentially due to noise or minor model misspecification). In contrast, the trace plot for a corresponding standard parameter (\(A_{11}\)) shows poorer mixing, and its histogram is significantly more dispersed. This difference in sampling behavior is typical across parameters and underscores the advantage of the well-behaved canonical posterior landscape for efficient MCMC exploration in higher dimensions. 

We also note a subtle point: our inference assumes diagonal noise covariances \(\Sigma, \Gamma\), whereas the balancing transformation \(T\) applied to the original system could induce off-diagonal terms in the true noise covariances of the balanced system. This slight model misspecification might contribute to small biases observed in estimates like that in \Cref{fig:high-dim-grid}(d). Currently, our inference assumes a simplified noise model, estimating a single variance parameter proportional to a standard normal distribution for both process and measurement noise. While a more accurate approach would involve learning the full Cholesky factors of the covariance matrices $\Sigma$ and $\Gamma$, we have observed that the associated computational cost is substantial, and a diagonal covariance assumption provides a valid simplification even under model misspecification, suggesting inference of diagonal matrices for $\Sigma$ and $\Gamma$ as a practical recommendation.
\begin{figure}[htbp]
    \centering
    \begin{minipage}[b]{0.48\textwidth}
        \centering
        \includegraphics[width=\linewidth]{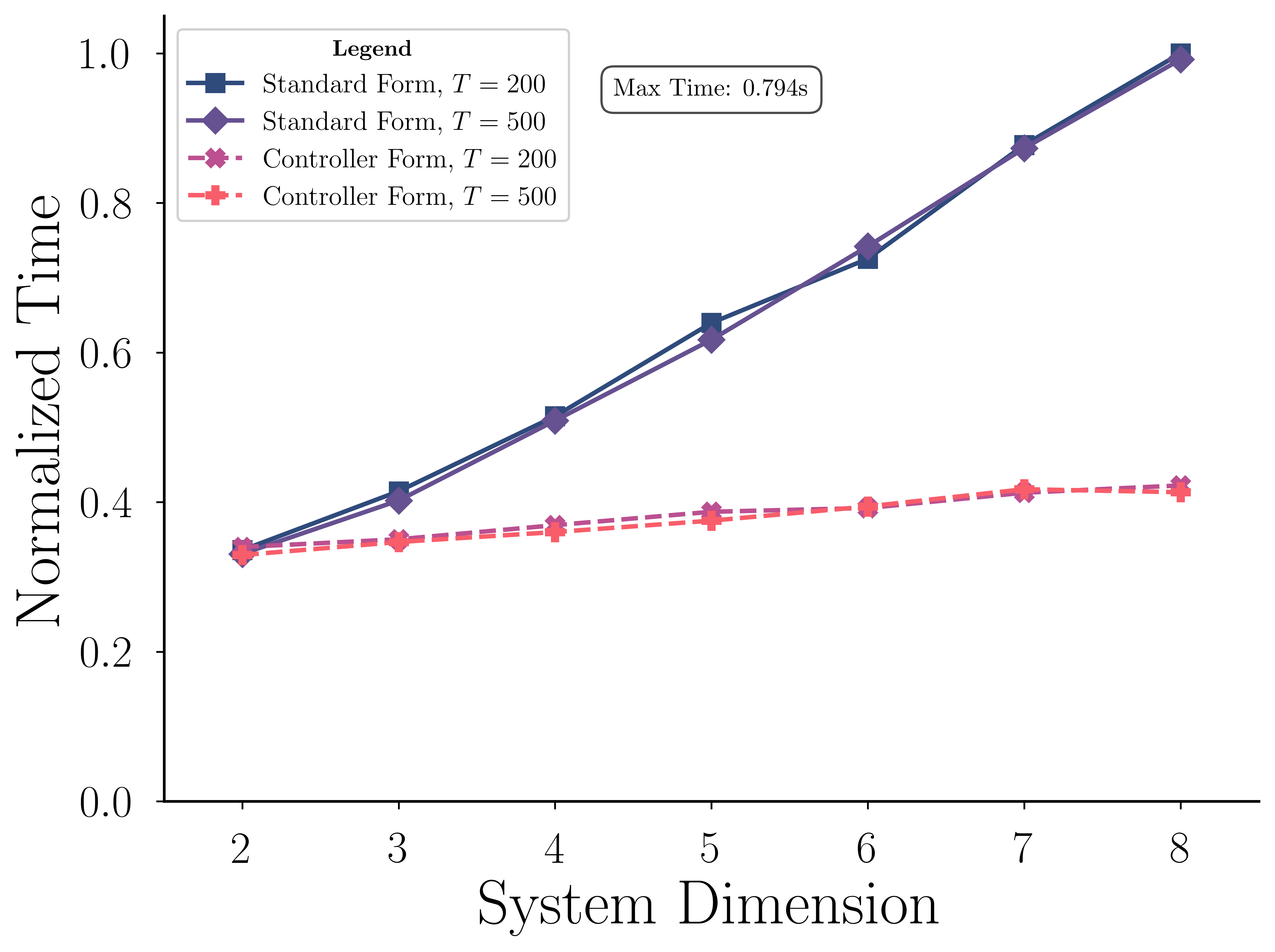}
        \vspace{0.5ex}\centering\textbf{(a)} Posterior Eval Time\par
    \end{minipage}
    \hfill
    \begin{minipage}[b]{0.48\textwidth}
        \centering
        \includegraphics[width=\linewidth]{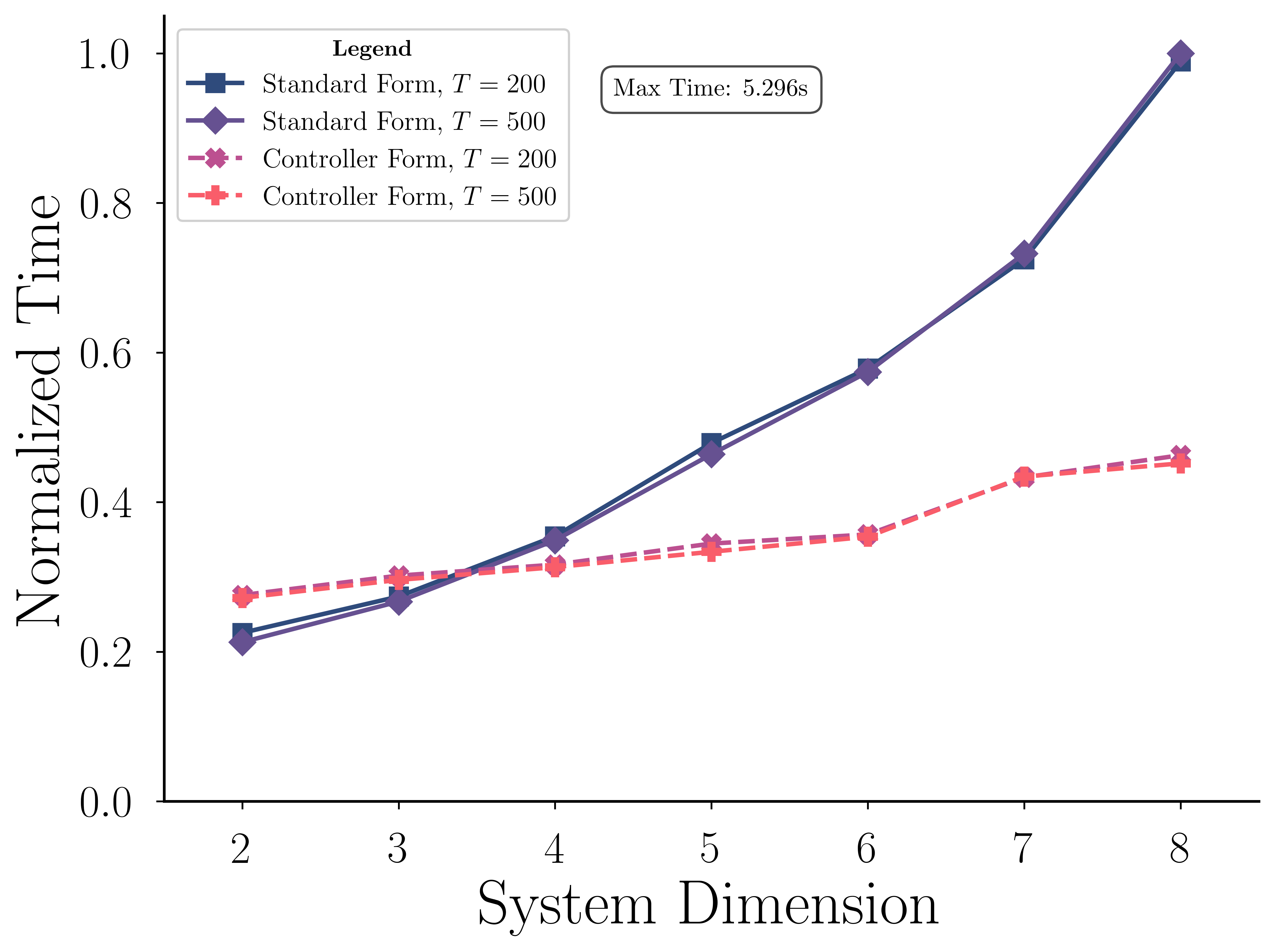}
        \vspace{0.5ex}\centering\textbf{(b)} Gradient Eval Time\par
    \end{minipage}

    \vspace{0.5em} 
    \begin{minipage}[b]{0.48\textwidth}
        \centering
        \includegraphics[width=\linewidth]{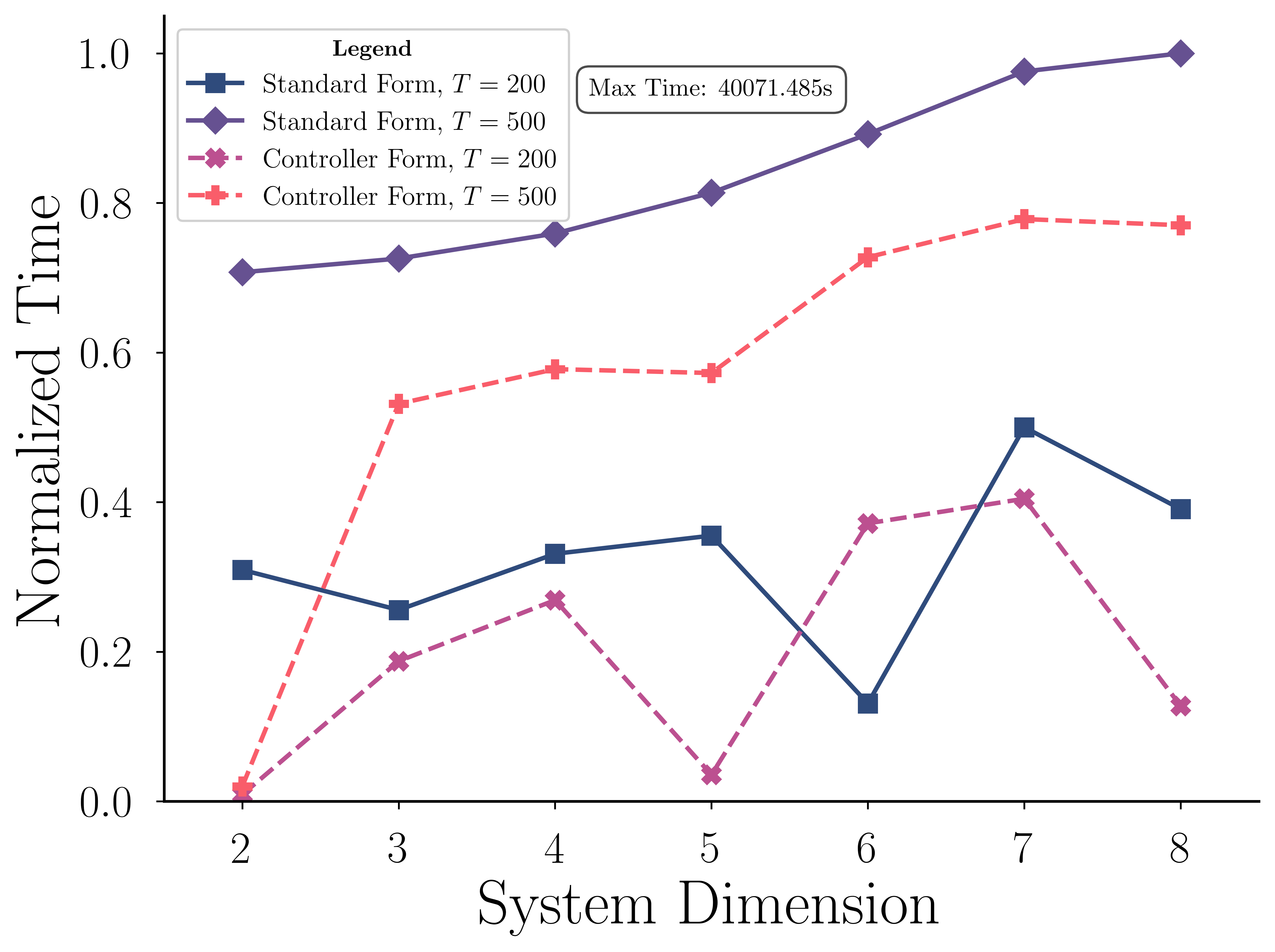}
        \vspace{0.5ex}\centering\textbf{(c)} Total Inference Time\par
    \end{minipage}
    \hfill
    \begin{minipage}[b]{0.48\textwidth}
        \centering
        \includegraphics[width=\linewidth]{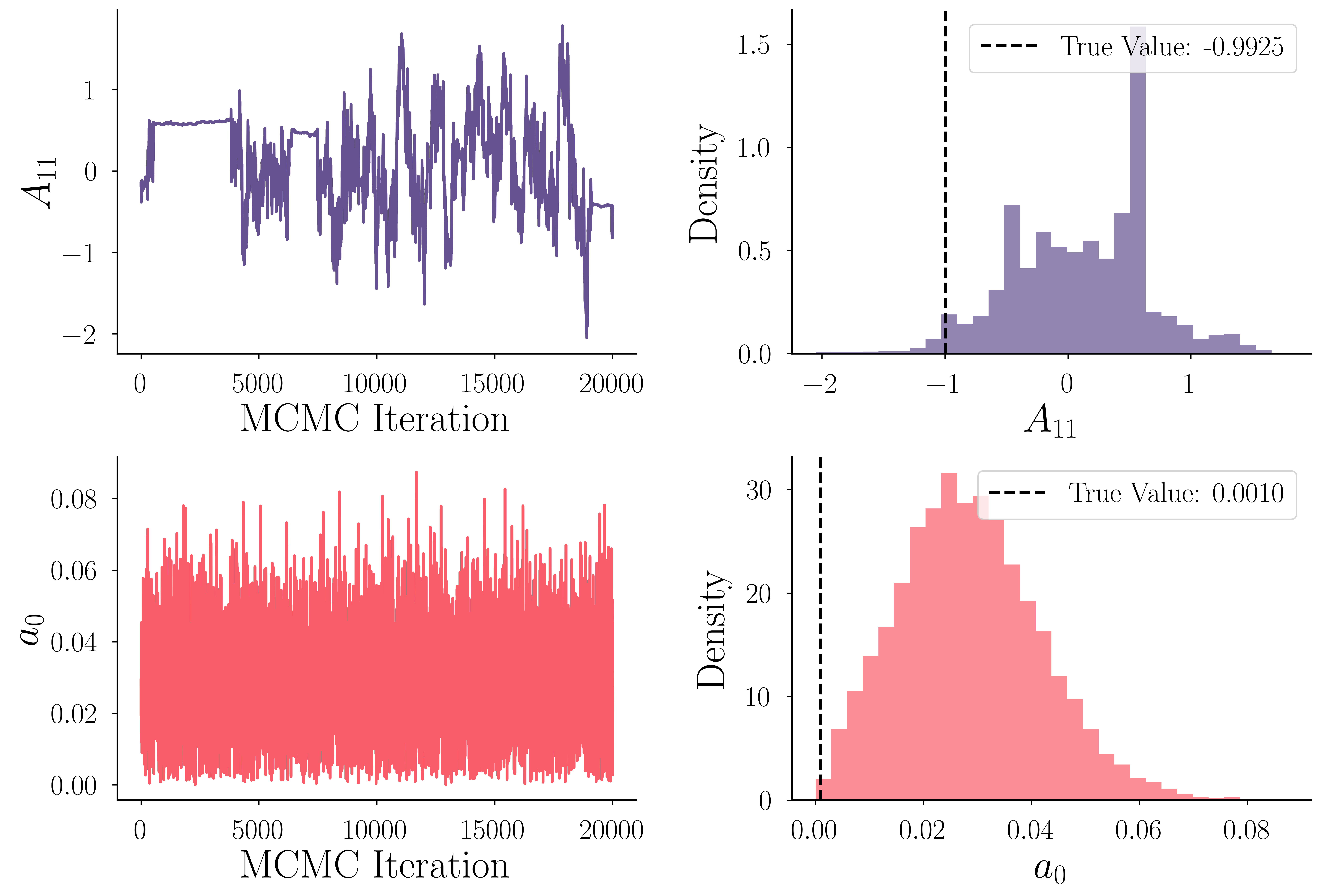}
        \vspace{0.5ex}\centering\textbf{(d)} MCMC diagnostics ($d_x=8$, $T=500$)\par
    \end{minipage}

    \caption{Scalability analysis comparing standard vs.\ canonical forms for increasing system dimension \(d_x\) ($T\in \{200, 500\}$ fixed) using balanced systems. Panels (a-c) show computation times (normalized where appropriate) for posterior evaluation, gradient evaluation, and total inference. Panel (d) shows MCMC trace plots and histograms for selected parameters (\(A_{11}\) standard vs.\ \(\alpha_0\) canonical) for $d_x=8$, $T=500$, illustrating better mixing and concentration for the canonical form.}
    \label{fig:high-dim-grid}
\end{figure}

\section{Likelihood decomposition}
\label{appendix_likelihood_decomposition}
Recall the likelihood in \Cref{eq:joint_likelihood} as 
\begin{equation}
\label{eq:markov_likelihood}
p(y_{[T]} \mid \mathcal{L}_{\Theta}, u_{[T]}) = p(y_0 \mid \mathcal{L}_{\Theta}, u_{0}) \prod_{t=1}^T p(y_t \mid \mathcal{L}_{\Theta}, y_{[t-1]}, u_{[t]})
\end{equation}


Under the assumption that both the process and observation noises are normal, we can obtain the likelihood for the state-update and the observation step, respectively:
\begin{align}
\label{eq:state_emission_update}
    p(X_{t+1}|X_t=x_t, u_t, \mathcal{L}_{\Theta}) &= \frac{1}{\sqrt{(2\pi)^{d_x}|\Sigma|}}\operatorname{exp}\left(-\frac{1}{2}\|X_{t+1}-Ax_t-Bu_t\|^2_\Sigma\right) \\
    p(Y_{t+1}|X_{t+1}=x_{t+1}, u_{t+1}, \mathcal{L}_{\Theta}) &= \frac{1}{\sqrt{(2\pi)^{d_y}|\Gamma|}}\operatorname{exp}\left(-\frac{1}{2}\|Y_{t+1}-CX_{t+1}-Du_{t+1}\|^2_{\Gamma}\right)
\end{align}
Following \cite{sarkka2013bayesian}, we describe the likelihood recursion in the following theorem:
    \begin{theorem}[Marginal likelihood recursion]
\label{thm:marg_lik}
Let \(y_{[T]}\) and \(u_{[T]}\) denote the sequences of observations and inputs. Assume that the state-space model is given by \eqref{eq:general_lti_state_space_model_noise} with an initial prior \(p(X_0\mid \mathcal{L}_{\Theta})\) on \(X_0\), where the parameter set is \(\Theta\). Then, under the assumption that both the process and observation noises are Gaussian, the marginal likelihood 
\[
p(y_{[T]} \mid \mathcal{L}_{\Theta}, u_{[T]})
\]
can be computed recursively via the following three steps for \(t=1,\dots,T-1\):

\begin{enumerate}
    \item Prediction: Compute the predictive density
    \begin{equation}
    p(X_{t+1} \mid \mathcal{L}_{\Theta}, y_{[t]}, u_{[t]})
    \;=\;
    \int \mathcal{N}\Bigl(X_{t+1};\, A\,X_t + B\,u_t,\,\Sigma\Bigr) \, p(X_t \mid \mathcal{L}_{\Theta}, y_{[t]}, u_{[t]}) \, dX_t.
    \end{equation}
    
    \item Update: Evaluate the one-step-ahead observation likelihood via
    \begin{equation}
    p(y_{t+1} \mid \mathcal{L}_{\Theta}, y_{[t]}, u_{[t+1]})
    \;=\;
    \int \mathcal{N}\Bigl(y_{t+1};\, C\,X_{t+1} + D\,u_{t+1},\,\Gamma\Bigr) \, p(X_{t+1} \mid \mathcal{L}_{\Theta}, y_{[t]}, u_{[t]}) \, dX_{t+1}.
    \end{equation}
    
    \item Marginalization: Update the filtering distribution for the state by
    \begin{equation}
    p(X_{t+1} \mid \mathcal{L}_{\Theta}, y_{[t+1]}, u_{[t+1]})
    \;=\;
    \frac{
      \mathcal{N}\Bigl(y_{t+1};\, C\,X_{t+1} + D\,u_{t+1},\,\Gamma\Bigr)
      \, p(X_{t+1} \mid \mathcal{L}_{\Theta}, y_{[t]}, u_{[t]})
    }{
      p(y_{t+1} \mid \mathcal{L}_{\Theta}, y_{[t]}, u_{[t+1]})
    }.
    \end{equation}
\end{enumerate}

By iterating these steps from \(t=1\) (with \(p(X_0 \mid \mathcal{L}_{\Theta})\)) to \(T-1\), one obtains the marginal likelihood in \Cref{eq:markov_likelihood},
as well as the sequence of filtering distributions \(p(X_t \mid \mathcal{L}_{\Theta}, y_{[t]}, u_{[t]})\) for each time step.
\end{theorem}
\section{Fisher information and BvM}
\label{appendix_fisher_information_and_BvM}
\subsection{Proof of \Cref{thm:BvM_LTI_applicability} (also \Cref{thm:BvM_LTI_full})}
\label{appendix_thm:BvM_LTI_full}
We follow \cite{vandervaart1998asymptotic} and restate the BvM in full generality here:
\begin{theorem}[Bernstein--von Mises theorem for LTI systems]
\label{thm:BvM_LTI_full}
Consider a discrete-time LTI system with the state-space representation \eqref{eq:general_lti_state_space_model_noise}. Let $y_{[T]}$ and $u_{[T]}$ denote the observed outputs and inputs up to time $T$. Let $\Theta \in \mathbf{\Theta}$ be the parameter vector, where $\mathbf{\Theta}$ is an open subset of $\mathbb{R}^d$. Suppose that:

\begin{enumerate}
    \item Regularity and identifiability: The true parameter $\Theta_0$ is an interior point of $\mathbf{\Theta}$, and the model is structurally identifiable in the canonical form. Moreover, the likelihood $p(y_{[T]} \mid \mathcal{L}_{\Theta}, u_{[T]})$ is smooth in $\Theta$, and the system is both controllable and observable at $\Theta_0$.
    
    \item Local asymptotic normality (LAN): The sequence of dynamical system experiments satisfies the LAN property at $\Theta_0$; that is, for a suitable sequence of estimators $\hat{\Theta}_T$ (e.g., the maximum likelihood estimator) with $\sqrt{T}(\hat{\Theta}_T-\Theta_0)$ converging in distribution, the log-likelihood admits the expansion
    \begin{equation}
    \log \frac{p(y_{[T]} \mid u_{[T]}, \Theta_0 + h/\sqrt{T})}{p(y_{[T]} \mid \mathcal{L}_{\Theta_0}, u_{[T]})}
    = h^\top \Delta_T - \frac{1}{2} h^\top I_T(\Theta_0, u_{[T]}) h + o_{P_{\Theta_0}}(1),
    \end{equation}
    where $\Delta_T$ converges in distribution to $\mathcal{N}(0, \mathcal{I}(\Theta_0))$, and $I_T(\Theta_0, u_{[T]})$ is the Fisher information matrix for the LTI system.
    \item Persistence of excitation: The input sequence persistently excites, as defined in \Cref{def:persistence_excitation_1}
    \item Prior regularity: The prior density $\pi(\Theta)$ is positive and continuous in a neighborhood of $\theta_0$.
\end{enumerate}

Then, if $\Pi(\Theta \mid y_{[T]}, u_{[T]})$ denotes the posterior distribution of $\Theta$ given the observations and inputs, we have that
\begin{equation}
\sup_{A \subset \mathbb{R}^d} \left| \Pi(\Theta \in A \mid y_{[T]}, u_{[T]}) - \Phi(A; \hat{\Theta}_T, I_T(\Theta_0, u_{[T]})^{-1}) \right| \xrightarrow{P_{\Theta_0}} 0,
\end{equation}
as $T \to \infty$. Equivalently, the posterior distribution of the rescaled parameter $\sqrt{T}(\Theta - \hat{\Theta}_T)$ converges in total variation to the multivariate normal distribution $\mathcal{N}\left(0, \mathcal{I}(\Theta_0)^{-1}\right)$, where $\mathcal{I}(\Theta_0) = \lim_{T \to \infty} \frac{1}{T}I_T(\Theta_0, u_{[T]})$.
\end{theorem}
\begin{proof}[Proof of \Cref{thm:BvM_LTI_full}\slash \Cref{thm:BvM_LTI_applicability}]
\label{proof_bvm_applicability}
We begin by proving the applicability of BvM for the canonical form. Let $\Theta_c \in \mathbf{\Theta_c} \subset \mathbb{R}^{n_c}$ denote the parameter vector in canonical form, where $n_c$ is the minimal number of parameters required to characterize an $d_x$-dimensional LTI system with $d_y$ inputs and $d_u$ outputs. The minimal canonical form uses exactly $n_c$ parameters, which is the minimal number required to uniquely characterize the input-output behavior of an $d_x$-dimensional LTI system. Each parameter in $\Theta$ appears exactly once in a specific position in the system matrices. For any two distinct minimal canonical parameter vectors $\Theta_c' \neq \Theta_c$, the corresponding transfer functions differ: $G(z; \Theta_c') \neq G(z; \Theta_c)$ see \cite{chen1999}). The log-likelihood function for this Gaussian system has the quadratic form. The Fisher information matrix (FIM) is therefore:
\begin{equation}
I_T(\Theta_c, u_{[T]}) = \sum_{t=1}^{T}\mathbb{E}_{\Theta_c}\left[\left(\frac{\partial \hat{y}_t(\Theta_c, u_{[t]})}{\partial\Theta_c}\right)^T \Gamma^{-1}\left(\frac{\partial \hat{y}_t(\Theta_c, u_{[t]})}{\partial\Theta_c}\right)\right]
\end{equation}
Following the argument above we conclude that the FIM is non-singular as $T \rightarrow \infty$ for canonical forms. Persistent excitation implies that $\frac{1}{N}\sum_{t=k+1}^{k+N}u_t u_t^T$ converges to a positive definite matrix $\Sigma_u$ as $N \rightarrow \infty$, ensuring that all state variables are influenced by the inputs. These two properties together ensure that the limit of the normalized FIM:
\begin{equation}
\mathcal{I}(\Theta_c) = \lim_{T \rightarrow \infty}\frac{1}{T}I_T(\Theta_c, u_{[T]})
\end{equation}
is positive definite for canonical parameterizations, which contain exactly the minimal number of parameters needed to characterize the input-output behavior.

Furthermore, the log-likelihood ratio for canonical parameters satisfies the Local Asymptotic Normality condition, given that the innovation sequence $\{y_t - \hat{y}_t(\Theta_c, u_{[t]})\}$ consists of independent Gaussian random variables with zero mean and covariance $\Gamma$, the Central Limit Theorem ensures that $\Delta_T \xrightarrow{d} \mathcal{N}(0, \mathcal{I}(\Theta_c))$, where $$\Delta_T=\frac{1}{\sqrt{T}} \sum_{t=1}^T\left(\frac{\partial \hat{y}_c\left(\Theta_c, u_{[t]}\right)}{\partial \Theta_c}\right)^T \Gamma^{-1}\left(y_t-\hat{y}_t\left(\Theta_c, u_{[t]}\right)\right)$$. Thus, all conditions of the Bernstein-von Mises theorem are satisfied for canonical parameterizations, and the posterior distribution converges in total variation to a normal distribution.
\end{proof}
\section{Proofs for Sections 1--4}
\label{appendix_proofs}
\subsection{Proof of \Cref{thm:isomorphism_general}}
\label{appendix_proof_thm_isomorphism_general}
We prove the two implications of the theorem separately.

($\Leftarrow$) First, assume the two minimal systems, $\mathcal{L}_{\Theta_s}$ and $\mathcal{L}_{\Theta_{s'}}$, are related by an invertible matrix $T_c \in \operatorname{GL}(d_x)$ according to \eqref{eq:similarity_transformation} and \eqref{covariances}. We aim to show that this structural equivalence forces the systems to be statistically isomorphic by demonstrating that they produce identical output distributions for any given input sequence. Let $x_t$ be a state trajectory of $\mathcal{L}_{\Theta_s}$ and define a new state variable $x'_{t} = T_c^{-1}x_t$. The initial state $x'_0$ is a zero-mean Gaussian variable with covariance
$$
\operatorname{Cov}(x'_0) = \operatorname{Cov}(T_c^{-1}x_0) = T_c^{-1} \operatorname{Cov}(x_0) (T_c^{-1})^\top = T_c^{-1} P_{0s} T_c^{-\top}.
$$
By \eqref{covariances}, this is equal to $P_{0s'}$, so the initial state of the transformed system is identically distributed to that of $\mathcal{L}_{\Theta_{s'}}$.

Next, we examine the state dynamics. Starting from $x_{t+1} = A_s x_t + B_s u_t + w_t$ and substituting $x_t = T_c x'_t$, we find
$$
T_c x'_{t+1} = A_s (T_c x'_t) + B_s u_t + w_t \quad \implies \quad x'_{t+1} = (T_c^{-1}A_s T_c) x'_t + (T_c^{-1}B_s) u_t + T_c^{-1}w_t.
$$
Using the relations from \eqref{eq:similarity_transformation}, this becomes $x'_{t+1} = A_{s'} x'_t + B_{s'} u_t + w'_t$, where the new process noise $w'_t = T_c^{-1}w_t$ has covariance $\operatorname{Cov}(w'_t) = T_c^{-1} \Sigma_s T_c^{-\top} = \Sigma_{s'}$. The dynamics of $x'_t$ are therefore statistically identical to those of the state in $\mathcal{L}_{\Theta_{s'}}$.

Finally, for the output equation $y_t = C_s x_t + D_s u_t + z_t$, substitution of $x_t = T_c x'_t$ yields
$$
y_t = (C_s T_c) x'_t + D_s u_t + z_t = C_{s'} x'_t + D_{s'} u_t + z_t.
$$
Since $\Gamma_{s'} = \Gamma_s$, the measurement noise is also identical. Because the transformed system driven by $x'_t$ and the system $\mathcal{L}_{\Theta_{s'}}$ have identically distributed initial states, identical stochastic dynamics, and identical output maps, they must induce the same output distribution $p(y_{[T]} \mid u_{[T]})$ for any input $u_{[T]}$. Thus, they are statistically isomorphic.

($\Rightarrow$) Conversely, assume $\mathcal{L}_{\Theta_s}$ and $\mathcal{L}_{\Theta_{s'}}$ are minimal and statistically isomorphic. We aim to show that this behavioral equivalence implies the existence of a unique similarity transformation $T_c$ that structurally links their parameter sets. The equality of output distributions for any input $u_{[T]}$ implies that all moments of the output must match. For a Gaussian system, this means the mean and auto-covariance of the output are identical for both systems. The expected value of the output is determined by the system's Markov parameters, $H_k = C A^k B$. The equality $\mathbb{E}_{\Theta_s}[y_t \mid u_{[T]}] = \mathbb{E}_{\Theta_{s'}}[y_t \mid u_{[T]}]$ (where $\mathbb{E}_\Theta[\cdot|\cdot]$ denotes an expectation under the parametrization given by $\Theta$) for all $u_{[T]}$ implies
$$
D_s = D_{s'}, \quad \text{and} \quad C_s A_s^k B_s = C_{s'} A_{s'}^k B_{s'} \quad \forall k \ge 0.
$$
The Realization Theorem states that for minimal (controllable and observable) systems, identical Markov parameters imply that the systems are related by a unique similarity transformation $T_c \in \operatorname{GL}(d_x)$\cite{chen1999}. This establishes the relations: $A_{s'} = T_c^{-1}A_s T_c$, $B_{s'} = T_c^{-1}B_s$, and $C_{s'} = C_s T_c$.

Now we show the covariance matrices transform accordingly. Statistical isomorphism requires the output auto-covariance functions to be identical, which for zero input implies $\operatorname{Cov}(y_t)_{\Theta_s} = \operatorname{Cov}(y_t)_{\Theta_{s'}}$ for all $t$. The output covariance is $\operatorname{Cov}(y_t) = C P_t C^\top + \Gamma$, where $P_t = \mathbb{E}[x_t x_t^\top]$ is the state covariance, evolving via the Lyapunov equation $P_{t+1} = A P_t A^\top + \Sigma$. Equality of output distributions implies the instantaneous innovation covariance is identical, which means the feedthrough measurement noise covariance must be the same, so $\Gamma_{s'} = \Gamma_s$. Equating the remaining parts of the output covariance gives $C_s P_{ts} C_s^\top = C_{s'} P_{ts'} C_{s'}^\top$. Substituting $C_{s'} = C_s T_c$ gives
$$
C_s P_{ts} C_s^\top = (C_s T_c) P_{ts'} (C_s T_c)^\top = C_s (T_c P_{ts'} T_c^\top) C_s^\top.
$$
For an observable system, this implies $P_{ts} = T_c P_{ts'} T_c^\top$, or $P_{ts'} = T_c^{-1} P_{ts} T_c^{-\top}$. At $t=0$, this gives the transformation for the initial covariance: $P_{0s'} = T_c^{-1} P_{0s} T_c^{-\top}$.

Finally, substituting $P_{ts'} = T_c^{-1} P_{ts} T_c^{-\top}$ into the Lyapunov equation for system $\mathcal{L}_{\Theta_{s'}}$ yields
$$
T_c^{-1}P_{t+1,s}T_c^{-\top} = A_{s'} (T_c^{-1} P_{ts} T_c^{-\top}) A_{s'}^\top + \Sigma_{s'}.
$$
Using $A_{s'} = T_c^{-1} A_s T_c$ and simplifying leads to
$$
T_c^{-1}P_{t+1,s}T_c^{-\top} = (T_c^{-1}A_sT_c) (T_c^{-1} P_{ts} T_c^{-\top}) (T_c^\top A_s^\top T_c^{-\top}) + \Sigma_{s'} = T_c^{-1}(A_s P_{ts} A_s^\top)T_c^{-\top} + \Sigma_{s'}.
$$
Multiplying by $T_c$ and $T_c^\top$ gives $P_{t+1,s} = A_s P_{ts} A_s^\top + T_c \Sigma_{s'} T_c^\top$. Comparing this to the original Lyapunov equation for $\mathcal{L}_{\Theta_s}$, we must have $\Sigma_s = T_c \Sigma_{s'} T_c^\top$, which rearranges to $\Sigma_{s'} = T_c^{-1} \Sigma_s T_c^{-\top}$. This completes the proof. \hfill $\blacksquare$

\subsection{Proof of \Cref{thm:companion_form_invertible_matrix}}
\label{appendix_thm:companion_form_invertible_matrix}
 Let
 \begin{equation} 
 P_A(\lambda)
 \;=\;
 \lambda^{d_x} \;+\; a_{d_x-1} \lambda^{d_x-1} \;+\;\cdots\;+\;a_1\lambda \;+\; a_0 
 \end{equation} 
 be the characteristic polynomial of \(A\). Define the auxiliary polynomials 
 \begin{equation} 
 P_i(\lambda)
 \;=\;
 \sum_{k=0}^{d_x-i} a_{k+i}\,\lambda^k, 
 \quad 
 i = 0, 1, \ldots, d_x, 
 \end{equation} 
 with \(P_0 = P_A\) and \(P_{d_x} = 1\) (setting \(a_{d_x} = 1\)). These polynomials satisfy 
 \begin{equation} 
 \lambda\,P_i(\lambda) \;=\;
 P_{i-1}(\lambda) \;-\; a_{i-1}\,P_{d_x}(\lambda). 
 \end{equation} 
 Next, define the vectors \(f_i = P_i(A)\,b\) for \(i = 0, 1, \ldots, d_x\). By the Cayley--Hamilton theorem, we have \(f_0 = P_A(A)\,b = 0\). 

 Because \((A, b)\) is controllable, the vectors
 \(\{ A^{\,d_x-1}\,b,\,A^{\,d_x-2}\,b,\,\ldots,\,b\}\) form a basis of \(\mathbb{R}^{d_x}\).
 Moreover, \(\{\,f_1,\,\ldots,\,f_{d_x}\}\) are related to this basis via the invertible transformation 
 \begin{equation} 
 \begin{pmatrix} f_1 & \cdots & f_{d_x} \end{pmatrix}
 \;=\; 
 \begin{pmatrix} A^{d_x-1}b & \cdots & b \end{pmatrix} 
 \begin{pmatrix} 
 1 & 0 & 0 & \cdots & 0 \\ 
 a_{d_x-1} & 1 & 0 & \cdots & 0 \\ 
 a_{d_x-2} & a_{d_x-1} & 1 & \cdots & 0 \\ 
 \vdots & \vdots & \vdots & \ddots & \vdots \\ 
 a_1 & a_2 & a_3 & \cdots & 1 
 \end{pmatrix}. 
 \end{equation} 
 Hence, \(\{\,f_1,\,\ldots,\,f_{d_x}\}\) is itself a basis.
 Defining \(T_c^{-1} = [\,f_1 \;\cdots\; f_{d_x}\,]\), the polynomial relation implies 
 \begin{equation} 
 A f_i
 = f_{i-1} -a_{i-1}f_{d_x}, 
 \end{equation} 
 for \(i = 1, \ldots, d_x\). From this recursion, one directly obtains the canonical (companion) form of \(T_c\,A\,T_c^{-1}\). Finally, since \(f_{d_x} = b\), we have \(T_c b = (0, \ldots, 0,\,1)^\top\).  \hfill $\blacksquare$
\subsection{Proof of \Cref{lem:induced_prior}}
\label{appendix_lem:induced_prior}
Let the prior density $p(\Theta_s)$ on the standard space $\mathbf{\Theta}_s$ define a measure $\mu_s$ such that for any measurable set $A \subseteq \mathbf{\Theta}_s$, its measure is $\mu_s(A) = \int_A p(\Theta_s) \, d\Theta_s$. Using the canonical projection map $\tau: \mathbf{\Theta}_s \to \mathbf{\Theta}_c$, we define a pushforward measure, $\mu_c$, on the canonical space $\mathbf{\Theta}_c$. By definition, the measure of any set $E \subseteq \mathbf{\Theta}_c$ under $\mu_c$ is the measure of its pre-image, $\mu_c(E) := \mu_s(\tau^{-1}(E))$.

By the Radon--Nikodym theorem, since the measure $\mu_c$ is absolutely continuous with respect to the Lebesgue measure on $\mathbf{\Theta}_c$, there exists a unique (up to a set of measure zero) density function $p(\Theta_c)$ such that the measure $\mu_c(E)$ can be computed as an integral of this density: $\mu_c(E) = \int_E p(\Theta_c) \, d\Theta_c$. This function $p(\Theta_c)$ is the induced prior density.

By equating the definition of the pushforward measure with its representation via the Radon--Nikodym derivative we uniquely set up the relation:
\begin{equation}
\int_E p(\Theta_c) \, d\Theta_c = \int_{\tau^{-1}(E)} p(\Theta_s) \, d\Theta_s
\end{equation}
for any measurable set $E \subseteq \mathbf{\Theta}_c$, where $\tau^{-1}(E) = \{\Theta_s \in \mathbf{\Theta}_s \mid \tau(\Theta_s) \in E\}$ is the pre-image of $E$. \hfill $\blacksquare$

\subsection{Proof of \Cref{equivalence_of_pushforwards}}
First we establish some preliminary results, which will be useful to proof \Cref{equivalence_of_pushforwards}.
\label{appendix_equivalence_of_pushforwards}
\begin{proposition}[Invariance under similarity transformations]
\label{invariance_and_computational_advantage_canonical_forms}
Let $\Theta = (A, B, C, D)$ and $\Theta' = (A', B', C', D')$ be parameterizations of two minimal LTI systems related by a similarity transformation $T$ such that $A' = TAT^{-1}$, $B' = TB$, $C' = CT^{-1}$, and $D' = D$. Then the following key system properties are identical for both parameterizations:
\begin{enumerate}
    \item The sequence of \textbf{Markov parameters}, $M_t = CA^{t-1}B = C'(A')^{t-1}B'$ for $t \ge 1$, and the feedthrough term, $D = D'$.
    \item The \textbf{eigenvalue spectrum} of the dynamics matrix, $\Lambda(A) = \Lambda(A')$.
\end{enumerate}
These properties imply that the Hankel matrices $H$ and transfer functions $G$ of the two systems are also identical.
\end{proposition}

\begin{proof}[Proof of \Cref{invariance_and_computational_advantage_canonical_forms}]
\label{proof_invariance_and_computational_advantage_canonical_forms}
We begin by proving part (i). The Markov parameters transform under $T$ as:
\begin{align*}
M'_t &= CT T^{-1}A^{t-1}T T^{-1}B\\ 
&= CA^{t-1}B\\
&= M_t
\end{align*}
Since both the Hankel matrix and transfer function are constructed entirely from Markov parameters:
\begin{align*}
H &= [M_i]_{i,j}, \quad 1 \leq i \leq p, \, 1 \leq j \leq q\\
G(z) &= D + \sum_{t=1}^{\infty} M_t z^{-t}
\end{align*}
their invariance follows directly from the invariance of $M_t$. The eigenvalue invariance follows from:
\begin{align*}
\det(A' - \lambda I) &= \det(T^{-1}AT - \lambda I)\\
&= \det(T^{-1})\det(A - \lambda I)\det(T)\\
&= \det(A - \lambda I)
\end{align*}
\end{proof}

Furthermore, computing eigenvalues from the companion matrix $A_c$ of a canonical form (e.g., \eqref{eq:controller_form}) offers computational advantages compared to general matrices. The computational advantage of using a canonical form for eigenvalue analysis is substantial. For a general matrix $A$ in a standard parameterization, the eigenvalues are found numerically using a direct eigensolver, such as the QR algorithm. The computational cost of this operation is \textbf{$O({d_x}^3)$} floating-point operations.

In contrast, for a matrix in a canonical representation (e.g., the controller or observer form), the coefficients of its characteristic polynomial are directly available from its entries. The task is then reduced to finding the roots of this ${d_x}$-degree polynomial. Numerically solving for the roots using robust iterative methods (like the Jenkins-Traub algorithm) has a typical complexity of \textbf{$O({d_x}^2)$}.

\begin{proof}[Proof of \Cref{equivalence_of_pushforwards}]
\label{proof_of_generalized_equivalence_of_pushforwards}
$\Theta_s^M$ and $\Theta_c^M$ are compact subsets of Euclidean space equipped with the Borel $\sigma$-algebra and Lebesgue measure. We write $(A^1, B^1, C^1, D^1) \sim (A^2, B^2, C^2, D^2)$ if there exists $T \in GL(d_x)$ such that:
\begin{align}
A^2 &= T^{-1}A^1T,\quad
B^2 = T^{-1}B^1,\quad
C^2 = C^1T,\quad
D^2 = D^1 .
\end{align}
This relation partitions $\Theta_s^M$ into equivalence classes $[\Theta_s]$. The measurability of this partition follows from the fact that the map $(A, B, C, D, T) \mapsto (TAT^{-1}, TB, CT^{-1}, D)$ is continuous. For each equivalence class $[\Theta_s]$, we select a unique canonical form representative $\Theta_c \in \Theta_c^M$. The mapping $\mathcal{T}: \Theta_s^M \to \Theta_c^M$ that assigns each standard form to its canonical form is explicitly constructable (e.g. for the controller form \Cref{thm:companion_form_invertible_matrix}). Define the bounded general linear group:
\begin{equation}
GL_K(d_x) = \{T \in GL(d_x) : \|T\|_F \leq K, \|T^{-1}\|_F \leq K\}
\end{equation}

For sufficiently large $K$, and for any $\Theta_s \in \Theta_s^M$, there exists $T \in GL_K(d_x)$ such that $\mathcal{T}(\Theta_s) = (TAT^{-1}, TB, CT^{-1}, D)$. This follows from the normalization properties of canonical forms. The set $GL_K(d_x)$ is compact in the Euclidean topology on $\mathbb{R}^{{d_x}^2}$, as it is the intersection of the closed set $\{T : \|T\|_F \leq K, \|T^{-1}\|_F \leq K\}$ with the open set $GL(d_x)$. It carries a well-defined Borel $\sigma$-algebra and can be equipped with a finite measure $\mu$. For any transformation matrix $T \in GL_K(d_x)$:
\begin{equation}
p(y_{[k]} \mid TAT^{-1}, TB, CT^{-1}, D, u_{[k]}) = p(y_{[k]} \mid A, B, C, D, u_{[k]})
\end{equation} 
as we have established in \Cref{thm:isomorphism_general}. As prior can vary within equivalence classes:
\begin{equation}
p(TAT^{-1}, TB, CT^{-1}, D) \neq p(A, B, C, D) \text{ in general}
\end{equation}
we need to carefully keep track of the differently distributed masses. By Bayes' theorem, the posterior will also generally vary within the respective equivalence classes. For any canonical form $\Theta_c = (A', B', C', D') \in \Theta_c^M$, its preimage under $\mathcal{T}$ is:
\begin{equation}
\mathcal{T}^{-1}(\Theta_c) = \{(TA'T^{-1}, TB', C'T^{-1}, D') : T \in GL_K(d_x)\}
\end{equation}
However, this parameterization is redundant. For minimal systems with distinct eigenvalues, the stabilizer is given by:
\begin{equation}
\text{Stab}(\Theta_c) = \{\lambda I : \lambda \neq 0, \lambda I \in GL_K(d_x)\} = \{\lambda I : 0 < |\lambda| \leq K, |\lambda^{-1}| \leq K\}
\end{equation}
The equivalence class is thus precisely characterized as:
\begin{equation}
\mathcal{T}^{-1}(\Theta_c) = \{(TA'T^{-1}, TB', C'T^{-1}, D') : [T] \in GL_K(d_x)/\text{Stab}(\Theta_c)\}
\end{equation}
where $[T]$ denotes the equivalence class of $T$ in the quotient space. The quotient space $GL_K({d_x})/\text{Stab}(\Theta_c)$ carries a well-defined Borel $\sigma$-algebra and measure. Specifically, let $\mu$ be the Haar measure on $GL_K(d_x)$, which exists and is unique up to a scaling factor because $GL_K(d_x)$ is a compact topological group. 

The quotient measure $\mu_{/}$ on $GL_K(d_x)/\text{Stab}(\Theta_c)$ is defined for any measurable set $E$ in the quotient space as:
\begin{equation}
\mu_{/}(E) = \frac{\mu(\{T \in GL_K(d_x) : [T] \in E\})}{\mu(\text{Stab}(\Theta_c))}
\end{equation}

For any measurable set $E \subset \Theta_c^M$, we define the posterior probability as:
\begin{equation}
P(E \mid u_{[k]}, y_{[k]}) = \int_{\mathcal{T}^{-1}(E)} p(\Theta_s \mid u_{[k]}, y_{[k]}) \, d\Theta_s
\end{equation}

By the disintegration theorem, this integral can be broken down into integrals over individual fibers $\mathcal{T}^{-1}(\{\Theta_c\})$ for each $\Theta_c \in E$:
\begin{equation}
\int_{\mathcal{T}^{-1}(E)} p(\Theta_s \mid u_{[k]}, y_{[k]}) \, d\Theta_s = \int_{E} \left( \int_{\mathcal{T}^{-1}(\{\Theta_c\})} p(\Theta_s \mid u_{[k]}, y_{[k]}) \, d\kappa_{\Theta_c}(\Theta_s) \right) \, d\Theta_c
\end{equation}
where $\kappa_{\Theta_c}$ is the conditional measure on the fiber $\mathcal{T}^{-1}(\{\Theta_c\})$.

The explicit form of the conditional measure $\kappa_{\Theta_c}$ can be expressed using the quotient space structure:
\begin{equation}
\int_{\mathcal{T}^{-1}(\{\Theta_c\})} p(\Theta_s) \, d\kappa_{\Theta_c}(\Theta_s) = \int_{GL_K(d_x)/\text{Stab}(\Theta_c)} p((TA'T^{-1}, TB', C'T^{-1}, D')) \cdot J(T) \, d\mu_{/}([T]),
\end{equation}
where $J(T) = |\det(T)|^{{d_x}-1}$ is the Jacobian determinant of the parameterization map. Using Bayes' theorem and the invariance of the likelihood:
\begin{align}
p(\Theta_s \mid u_{[k]}, y_{[k]}) &\propto p(y_{[k]} \mid \Theta_s, u_{[k]}) \cdot p(\Theta_s)= p(y_{[k]} \mid \mathcal{T}(\Theta_s), u_{[k]}) \cdot p(\Theta_s)
\end{align}

For any $\Theta_c \in \Theta_c^M$, applying the conditional measure formula:
\begin{align}
\int_{\mathcal{T}^{-1}(\{\Theta_c\})} p(\Theta_s \mid u_{[k]}, y_{[k]}) \, d\kappa_{\Theta_c}(\Theta_s) &\propto p(y_{[k]} \mid \Theta_c, u_{[k]}) \cdot \int_{\mathcal{T}^{-1}(\{\Theta_c\})} p(\Theta_s) \, d\kappa_{\Theta_c}(\Theta_s)
\end{align}
We define the induced prior on the canonical form as:
\begin{equation}
p(\Theta_c)  = \int_{GL_K(d_x)/\text{Stab}(\Theta_c)} p((TA'T^{-1}, TB', C'T^{-1}, D')) \cdot |\det(T)|^{{d_x}-1} \, d\mu_{/}([T])
\end{equation}

This defines a proper probability measure on $\Theta_c^M$ that accounts for the variation of the prior density across the equivalence class corresponding to $\Theta_c$. For any measurable set $\mathcal{B}_Q$ in the space of system-level quantities, the set $\{\Theta_s^M : S^{Q}(\Theta_s) \in \mathcal{B}_Q\}$ consists of entire equivalence classes. This follows from the invariance of $S^Q$ under similarity transformations (see \Cref{invariance_and_computational_advantage_canonical_forms}). Therefore:
\begin{equation}
\begin{aligned}
\int_{\{\Theta_s^M : S^{Q}(\Theta_s) \in \mathcal{B}_Q\}} p(\Theta_s \mid u_{[k]}, y_{[k]}) \, d\Theta_s &= \int_{\mathcal{T}^{-1}(\{\Theta_c^M : S^Q(\Theta_c) \in \mathcal{B}_Q\})} p(\Theta_s \mid u_{[k]}, y_{[k]}) \, d\Theta_s \\
&= \int_{\{\Theta_c^M : S^Q(\Theta_c) \in \mathcal{B}_Q\}} \left( \int_{\mathcal{T}^{-1}(\{\Theta_c\})} p(\Theta_s \mid u_{[k]}, y_{[k]}) \, d\kappa_{\Theta_c}(\Theta_s) \right) \, d\Theta_c \\
&= \int_{\{\Theta_c^M : S^Q(\Theta_c) \in \mathcal{B}_Q\}} p(\Theta_c \mid u_{[k]}, y_{[k]}) \, d\Theta_c
\end{aligned}
\end{equation}
This establishes the equivalence of pushforward posteriors for standard and canonical forms with an arbitrary prior distribution, not necessarily constant over equivalence classes.
\end{proof}
\subsection{Remarks and challenges for priors}


\begin{remark}[Challenges in prior specification with standard parametrization]
\label{remark_challenges_setting_priors_over_standard_form}
Consider a discrete-time LTI system \eqref{eq:general_lti_state_space_model_noise} of dimension $n$. Consider the case, if we relax, that the priors are the same for the equivalence class, let us see that it still leads to several challenges. We examine how Gaussian priors on state-space matrices transform to canonical parameterizations and the resulting implications for Bayesian inference. A common approach places independent Gaussian priors $A_{ij} \sim \mathcal{N}(\mu_A, \sigma^2_A)$, $B_i \sim \mathcal{N}(\mu_B, \sigma^2_B)$, $C_j \sim \mathcal{N}(\mu_C, \sigma^2_C)$ on each matrix entry. While this specification appears intuitive, it introduces significant complications when considering system properties that are invariant to the choice of state-space realization. 

For controllable pairs $(A,B)$, there exists an invertible similarity transformation that yields the controller canonical form. The transformation matrix $T^{-1}$ is constructed from coefficients of the characteristic polynomial of $A$ (see \Cref{thm:companion_form_invertible_matrix}).
The characteristic polynomial coefficients can be expressed using exterior algebra as:
\begin{align}
p_A(t)=\sum_{k=0}^n t^{n-k}(-1)^k \operatorname{tr}\left(\bigwedge^k A\right)
\end{align}
where $\operatorname{tr}\left(\bigwedge^k A\right)$ is the trace of the $k$-th exterior power of $A$, computable via the determinant formula:
\begin{align}
\operatorname{tr}\left(\bigwedge^k A\right)=\frac{1}{k!}\begin{vmatrix}
\operatorname{tr} A & k-1 & \cdots & 0 \\
\vdots & \vdots & \ddots & \vdots \\
\operatorname{tr} A^k & \operatorname{tr} A^{k-1} & \cdots & \operatorname{tr} A
\end{vmatrix}
\end{align}
This transformation is fundamentally nonlinear. For a simple $2 \times 2$ matrix case, the characteristic polynomial $p_A(\lambda) = \lambda^2 - (a_{11} + a_{22})\lambda + (a_{11}a_{22} - a_{12}a_{21})$ demonstrates that while the trace term is linear in the entries of $A$, the determinant term exhibits quadratic nonlinearity. To quantify how probability measures transform under this nonlinear mapping, we examine the Jacobian matrix. For our $2 \times 2$ example, key entries include $\partial a_1/\partial a_{11} = -1$, $\partial a_0/\partial a_{11} = a_{22}$, and $\partial a_0/\partial a_{12} = -a_{21}$, further illustrating the mixture of linear and nonlinear dependencies.

The transformed probability density follows $p_{\theta'}(\theta') = p_{\theta}(f^{-1}(\theta')) \cdot |\det(J)|^{-1}$, and as the inverse is dense, leads to increased computational cost. When considering eigenvalues instead of canonical form coefficients, additionally it introduces several complications. First, the transformation from matrix entries to eigenvalues involves solving polynomial equations, which is inherently nonlinear. Second, eigenvalues are unordered, leading to permutation invariance that creates multimodality in the transformed distribution.
\end{remark}
\subsection{Proof of \Cref{thm:pushforward_polynomial_roots}}
\label{appendix_thm:pushforward_polynomial_roots}
Write
\begin{equation}
p(x)=x^{d_x}+ a_{{d_x}-1} x^{{d_x}-1} \ldots a_1 x +a_0=\prod_{k=1}^{d_x}\left(x-\lambda_k\right) .
\end{equation}
Using Vieta's formula we identify the coefficients with the products and sums of eigenvalues. We can write the Jacobian as:
\begin{equation}
    D\Phi=\left(\begin{array}{cccc}\frac{\partial a_{{d_x}-1}}{\partial \lambda_1} & \frac{\partial a_{{d_x}-1}}{\partial \lambda_2} & \cdots & \frac{\partial a_{{d_x}-1}}{\partial \lambda_{d_x}} \\ \frac{\partial a_{{d_x}-2}}{\partial \lambda_1} & \frac{\partial a_{{d_x}-2}}{\partial \lambda_2} & \cdots & \frac{\partial a_{{d_x}-2}}{\partial \lambda_{d_x}} \\ \vdots & \vdots & \ddots & \vdots \\ \frac{\partial a_0}{\partial \lambda_1} & \frac{\partial a_0}{\partial \lambda_2} & \cdots & \frac{\partial a_0}{\partial \lambda_{d_x}}\end{array}\right)=\left(\begin{array}{cccc}1 & 1 & \cdots & 1 \\ \lambda_1 & \lambda_2 & \cdots & \lambda_{d_x} \\ \vdots & \vdots & \ddots & \vdots \\ \lambda_2 \lambda_3 \cdots \lambda_{d_x} & \lambda_1 \lambda_3 \cdots \lambda_{d_x} & \cdots & \lambda_1 \lambda_2 \cdots \lambda_{{d_x}-1}\end{array}\right)
\end{equation}
Using induction one can show, that the matrix is equivalent to the Vandermonde matrix, which has an explicit formula given by:
\begin{equation}
|\operatorname{det}D\Phi|=\left|\prod_{i<j}\left(\lambda_i-\lambda_j\right)\right| . \qquad \blacksquare
\end{equation}

\subsection{Proof of \Cref{lemma:uniform_prior_2d}}
\label{appendix_lemma:uniform_prior_2d}
The characteristic polynomial of any 2×2 real matrix $A$ has the form $p(z) = z^2 - \text{tr}(A)z + \det(A)$. For a stable system whose eigenvalues lie within the unit disk, the coefficients $(\text{tr}(A), \det(A))$ must satisfy the Schur-Cohn stability conditions: $|\det(A)| < 1$ and $|\text{tr}(A)| < 1 + \det(A)$. These inequalities define a triangular region $\mathcal{R}$ in the $(\text{tr}(A), \det(A))$ plane with total area 4 square units.

The discriminant $\Delta = \text{tr}(A)^2 - 4\det(A)$ determines the nature of the eigenvalues. The curve $\Delta = 0$ (a parabola given by $\text{tr}(A)^2 = 4\det(A)$) divides $\mathcal{R}$ into two subregions: $\mathcal{R}_1$ where eigenvalues are real, and $\mathcal{R}_2$ where eigenvalues are complex.

Computing the area of $\mathcal{R}_1$ requires integrating over the region where $\text{tr}(A)^2 \geq 4\det(A)$ within the stability constraints:
\begin{align}
\text{Area}(\mathcal{R}_1) &= \int_{-2}^{2} \int_{\max(-1, \frac{\text{tr}(A)^2}{4})}^{\min(1, 1-|\text{tr}(A)|)} d\det(A) \, d\text{tr}(A)= \frac{8}{3}
\end{align}

The area of $\mathcal{R}_2$ is then $\text{Area}(\mathcal{R}_2) = \text{Area}(\mathcal{R}) - \text{Area}(\mathcal{R}_1) = 4 - \frac{8}{3} = \frac{4}{3}$.

Therefore, under a uniform distribution on $(\text{tr}(A), \det(A))$ in $\mathcal{R}$, the probability of real eigenvalues is $\frac{\text{Area}(\mathcal{R}_1)}{\text{Area}(\mathcal{R})} = \frac{2}{3}$, and the probability of complex eigenvalues is $\frac{\text{Area}(\mathcal{R}_2)}{\text{Area}(\mathcal{R})} = \frac{1}{3}$.

For purely imaginary eigenvalues, we require $\text{tr}(A) = 0$ and $\det(A) > 0$, which corresponds to a line segment in $\mathcal{R}$ with Lebesgue measure zero under the two-dimensional uniform distribution. To derive the specific density functions in each region, we apply the Jacobian transformation from the coefficient space to the eigenvalue space. For real eigenvalues the uniform density in $\mathcal{R}_1$ transforms to a constant density $\frac{1}{2}$ on $(-1,1)$. For complex eigenvalues the uniform density in $\mathcal{R}_2$ transforms to a constant density $\frac{1}{\pi}$ over the unit disk in the complex plane. \hfill $\blacksquare$

\subsection{Proof of \Cref{cor:mixture_stable_eigs}}
\label{appendix_cor:mixture_stable_eigs}
Let \(r\) be the number of real eigenvalues and \(c\) be the number of complex conjugate pairs. The total number of eigenvalues is \(d = r + 2c\). This immediately implies that \(d-r = 2c\) must be an even number, which is equivalent to the condition \(r \equiv d \pmod{2}\). For a stable system, all eigenvalues must lie within the open unit disk \(\mathbb{D} = \{z \in \mathbb{C} : |z| < 1\}\). Thus, the \(r\) real eigenvalues lie in \((-1, 1)\) and the \(c = (d-r)/2\) complex conjugate pairs lie in \(\mathbb{D} \setminus \mathbb{R}\).We first condition on the number of real eigenvalues, \(r\), which is an uncertain quantity governed by the prior probability mass function \(p_r(r)\). The set of admissible values for \(r\) is \(\{k \in \{0, \dots, d\} : d-k \text{ is even}\}\).
\begin{equation*}
p(\lambda) = \sum_{\substack{0 \le r \le d \\ d-r \text{ is even}}} p(\lambda \mid r) \, p_r(r).
\end{equation*}
Next, we condition on whether the eigenvalue \(\lambda\) is real or complex, yielding
\begin{equation}
p(\lambda \mid r) = p(\lambda \mid \text{real}, r) \, P(\text{real} \mid r) + p(\lambda \mid \text{complex}, r) \, P(\text{complex} \mid r).
\end{equation}
Here, the conditional probability of selecting a real eigenvalue from the spectrum, given \(r\) real roots, is \(P(\text{real} \mid r) = r/d\), while the probability of selecting a complex one is \(P(\text{complex} \mid r) = (d-r)/d\). Substituting these definitions back into the mixture expansion yields the final expression:
\begin{equation}
p(\lambda) = \sum_{\substack{0 \le r \le d \\ d-r \text{ is even}}} p_r(r) \left[ \frac{r}{d} \, p_{\mathrm{real}}(\lambda; r) + \frac{d-r}{d} \, p_{\mathrm{complex}}(\lambda; r) \right],
\end{equation}
For $d=2$, this recovers \Cref{lemma:uniform_prior_2d} with $p(0)=\frac{1}{3}$ and $p(2)=\frac{2}{3}$. \hfill $\blacksquare$

\subsection{BvM and Fisher information}
\label{appendix_thm:BvM_LTI_failure} 
\begin{proof}[Proof of \Cref{thm:BvM_LTI_failure}]
The BvM theorem fails for standard form parameterization $\Theta_s$, where all entries of matrices $A$, $B$, $C$, and $D$ are treated as free parameters. The fundamental issue is the existence of equivalence classes under similarity transformations. This creates equivalence classes in parameter space:
\begin{equation}
[\Theta_s] = \{\Theta'_s : \exists \text{ invertible } T \text{ such that } \Theta'_s = \{T^{-1}AT, T^{-1}B, CT, D\}\}
\end{equation}
Each equivalence class forms an $d_x^2$-dimensional manifold in parameter space (where $d_x$ is the state dimension), since $T$ has $d_x^2$ free elements. For any tangent vector $v$ to the manifold $[\Theta_s]$ at $\Theta_s$:
\begin{equation}
v^T I_T(\Theta_s, u_{[T]})v = \mathbb{E}_{\Theta_s}\left[\left(v^T \frac{\partial}{\partial\Theta_s}\log p(y_{[T]}|u_{[T]},\Theta_s)\right)^2\right] = 0
\end{equation}

This holds because moving along direction $v$ doesn't change the system's input-output behavior, hence doesn't affect the likelihood. Consequently:
\begin{equation}
\text{rank}(I_T(\Theta_s, u_{[T]})) \leq \dim(\Theta_s) - d_x^2
\end{equation}
making the FIM necessarily singular for standard parameterization. Due to this singularity, the posterior distribution cannot converge to a non-degenerate normal distribution in all directions of the parameter space.
\end{proof}
The geometric interpretation is that the posterior distribution in standard form becomes concentrated on an $d_x^2$-dimensional manifold rather than at a point, with its shape on this manifold determined by the prior distribution even in the asymptotic limit. This violates the core premise of the BvM theorem, which requires the posterior to converge to a normal distribution entirely determined by the data through the MLE and FIM.
\begin{proposition}[Fisher information for LTI systems in canonical forms]
\label{prop:fisher_LTI}
Consider an LTI system with deterministic state dynamics given by \eqref{eq:general_lti_state_space_model_noise}, where \( z_{[T]} \) is an i.i.d.\ sequence of Gaussian noise with \( z_t \sim \mathcal{N}(0,\sigma^2) \) for \( t=1,\dots,T \) and $w_t=0\;\forall t$. Assume that the system is in either the controller or observability canonical form. Let \(\theta\) denote the vector of system parameters (which includes either the entries of \(A_{c}\) or those of \(C_c\); in the observer form the the derivative would be with respect to $B_c$). Then, the Fisher information matrix corresponding to the likelihood function \(p(Y \mid \theta)\) is given by
\begin{equation}
\label{eq:fisher_info}
I_T(\theta, u_{[T]}) \;=\; \frac{1}{\sigma^2}\sum_{t=1}^T \left( \frac{\partial (C_c x_t)}{\partial \theta} \right)
\left( \frac{\partial (C_c x_t)}{\partial \theta} \right)^\top,
\end{equation}
where we have used the fact that $\mathbb{E}\Bigl[(y_t - C_c x_t)^2\Bigr] \;=\; \sigma^2$.
Moreover, the derivatives in \eqref{eq:fisher_info} can be computed recursively via the product rule. In particular, for the parameters \(a_i\) affecting the state matrix \(A_c\) (which appear in the controller form) the recursion is
\begin{equation}
\label{eq:deriv_a}
\frac{\partial y_{t+1}}{\partial a_i} \;=\; C_c\left(\sum_{m=0}^{t} A_c^{t-m}\frac{\partial A_c}{\partial a_i}A_c^m x_0 + \sum_{k=0}^{t}\sum_{n=0}^{t-k-1} A_c^{t-k-1-n}\frac{\partial A_c}{\partial a_i}A_c^n B_c u_k\right),
\end{equation}
and for the parameters \(c_i\) of \(C_c\) (when \(C_c\) is free, as in the controller form),
\begin{equation}
\label{eq:deriv_c}
\frac{\partial y_{t+1}}{\partial c_i} \;=\; \frac{\partial C_c}{\partial c_i}x_{t+1} \;=\; \frac{\partial C_c}{\partial c_i}\left(A_c^{t+1}x_0 + \sum_{k=0}^{t} A_c^{t-k}B_c u_k\right).
\end{equation}
In the observer canonical form, since \(C_c\) is fixed, the differentiation is carried out with respect to the parameters in the \(B_c\) matrix, and analogous recursive formulas hold.
\end{proposition}
\begin{proof}[Proof of \Cref{prop:fisher_LTI}]
The log-likelihood of the output sequence $Y = \{y_1, \dots, y_T\}$ given the parameter vector $\theta$ is derived from the Gaussian noise assumption. For each time step $t$, the observation $y_t$ is given by $y_t = C_c x_t + z_t$, where $z_t \sim \mathcal{N}(0, \sigma^2)$. The probability density function of a single observation $y_t$ is therefore $p(y_t | \theta) = \frac{1}{\sqrt{2\pi\sigma^2}} \exp\left(-\frac{(y_t - C_c x_t)^2}{2\sigma^2}\right)$. Since the noise samples $z_t$ are independent and identically distributed, the likelihood of the entire sequence $Y$ is the product of the individual probabilities, $p(Y|\theta) = \prod_{t=1}^T p(y_t|\theta)$. The log-likelihood function, $\mathcal{L}(\theta) = \log p(Y|\theta)$, is then the sum of the individual log-likelihoods:
$$ \mathcal{L}(\theta) = \sum_{t=1}^T \log p(y_t | \theta) = -\frac{T}{2}\log(2\pi\sigma^2) - \frac{1}{2\sigma^2}\sum_{t=1}^T (y_t - C_c x_t)^2. $$
The Fisher information matrix $I_L(\theta)$ is defined as the negative expectation of the Hessian of the log-likelihood function, i.e., $I_L(\theta) = -\mathbb{E}\left[\frac{\partial^2 \mathcal{L}(\theta)}{\partial \theta \partial \theta^\top}\right]$. Let's compute the first partial derivative of $\mathcal{L}(\theta)$ with respect to a generic parameter $\theta_i$:
$$ \frac{\partial \mathcal{L}(\theta)}{\partial \theta_i} = -\frac{1}{2\sigma^2}\sum_{t=1}^T 2(y_t - C_c x_t)\left(-\frac{\partial (C_c x_t)}{\partial \theta_i}\right) = \frac{1}{\sigma^2}\sum_{t=1}^T (y_t - C_c x_t) \frac{\partial (C_c x_t)}{\partial \theta_i}. $$
The second partial derivative with respect to $\theta_j$ is:
$$ \frac{\partial^2 \mathcal{L}(\theta)}{\partial \theta_j \partial \theta_i} = \frac{1}{\sigma^2}\sum_{t=1}^T \left[ -\frac{\partial (C_c x_t)}{\partial \theta_j} \frac{\partial (C_c x_t)}{\partial \theta_i} + (y_t - C_c x_t) \frac{\partial^2 (C_c x_t)}{\partial \theta_j \partial \theta_i} \right]. $$
Taking the expectation, we note that $\mathbb{E}[y_t - C_c x_t] = \mathbb{E}[z_t] = 0$. Thus, the second term vanishes. This leaves:
$$ I_L(\theta)_{ij} = -\mathbb{E}\left[\frac{\partial^2 \mathcal{L}(\theta)}{\partial \theta_j \partial \theta_i}\right] = \frac{1}{\sigma^2}\sum_{t=1}^T \mathbb{E}\left[ \frac{\partial (C_c x_t)}{\partial \theta_j} \frac{\partial (C_c x_t)}{\partial \theta_i} \right]. $$
Since the state evolution $x_t$ is deterministic given the inputs and initial state, the derivative $\frac{\partial (C_c x_t)}{\partial \theta}$ is non-random. Therefore, the expectation operator can be removed, and we can express the Fisher information matrix in vector form as stated in \eqref{eq:fisher_info}:
$$ I_L(\theta) = \frac{1}{\sigma^2}\sum_{t=1}^T \left( \frac{\partial (C_c x_t)}{\partial \theta} \right) \left( \frac{\partial (C_c x_t)}{\partial \theta} \right)^\top. $$
To derive the recursive formulas, we first express the state $x_t$ as a function of the initial state $x_0$ and the input sequence $u_k$: $x_t = A_c^t x_0 + \sum_{k=0}^{t-1} A_c^{t-1-k} B_c u_k$. The output is $y_t = C_c x_t$.
For a parameter $a_i$ in the state matrix $A_c$, we differentiate $y_{t+1} = C_c x_{t+1}$ with respect to $a_i$. Using the product rule and the formula for the derivative of a matrix exponential, we get:
$$ \frac{\partial x_{t+1}}{\partial a_i} = \frac{\partial A_c^{t+1}}{\partial a_i} x_0 + \sum_{k=0}^{t} \frac{\partial A_c^{t-k}}{\partial a_i} B_c u_k. $$
The derivative of the matrix power is $\frac{\partial A_c^p}{\partial a_i} = \sum_{m=0}^{p-1} A_c^{p-1-m} \frac{\partial A_c}{\partial a_i} A_c^m$. Substituting this into the expression for $\frac{\partial x_{t+1}}{\partial a_i}$ and then multiplying by $C_c$ yields the result in \eqref{eq:deriv_a}.
For a parameter $c_i$ in the output matrix $C_c$, the state $x_{t+1}$ does not depend on $c_i$. The derivative is simpler:
$$ \frac{\partial y_{t+1}}{\partial c_i} = \frac{\partial (C_c x_{t+1})}{\partial c_i} = \frac{\partial C_c}{\partial c_i} x_{t+1} = \frac{\partial C_c}{\partial c_i} \left(A_c^{t+1}x_0 + \sum_{k=0}^{t} A_c^{t-k}B_c u_k\right), $$
which gives \eqref{eq:deriv_c}. The reasoning for the observer canonical form, with differentiation with respect to parameters in $B_c$, follows an analogous procedure.
\end{proof}
The derivation, including the case with process noise (\(w_t \neq 0\)) requiring Kalman filter sensitivities, is discussed in the following proposition. We follow the derivation of \cite{CavanaughShumway1996EFIM} and adapt the equations to the canonical forms:

\begin{proposition}[Fisher information for LTI systems in canonical forms via Kalman filter]
\label{prop:fisher_kalman_canonical}
Consider an LTI system represented in a canonical form:
\begin{equation}
x_{t+1} = A x_t + B u_t + w_t, \quad y_t = C x_t + z_t,
\end{equation}
with initial state $x_0 \sim \mathcal{N}(\hat{x}_0, P_0)$, process noise $w_t \sim \mathcal{N}(0, \Sigma)$, and measurement noise $z_t \sim \mathcal{N}(0, \Gamma)$. Let the parameter vector $\theta$ consist of the variable entries in the system matrices. For the controller form the parameter vector is $\theta = [a_1, \dots, a_{d_x}, c_1, \dots, c_{d_x}]^\top$, where $\{a_i\}$ are the coefficients in the last row of $A$ and $\{c_i\}$ are the elements of $C$. The matrix $B$ is assumed known. The observer form's parameter vector contains the variable entries of $A$ and $B$, while $C$ is fixed.
The log-likelihood function, derived from the Kalman filter's prediction error decomposition, is:
\begin{align}
\mathcal{L}(\theta) = -\frac{T N_y}{2}\log(2\pi) - \frac{1}{2}\sum_{t=1}^T\left(\log|S_t| + \nu_t^\top S_t^{-1}\nu_t\right)
\end{align}
where $N_y$ is the dimension of the output $y_t$, $\nu_t = y_t - C\hat{x}_{t|t-1}$ is the innovation, and $S_t = CP_{t|t-1}C^\top + \Gamma$ is the innovation covariance. The Fisher information matrix (FIM) is given by:
\begin{align}
I(\theta) = \sum_{t=1}^T \mathbb{E}\left[ \left(\frac{\partial \nu_t}{\partial \theta}\right)^\top S_t^{-1} \left(\frac{\partial \nu_t}{\partial \theta}\right) + \frac{1}{2} \left(\frac{\partial \vec{S}_t}{\partial \theta}\right)^\top (S_t^{-1} \otimes S_t^{-1}) \left(\frac{\partial \vec{S}_t}{\partial \theta}\right) \right],
\end{align}
where $\vec{S}_t$ denotes the vectorization of the matrix $S_t$. The derivatives $\frac{\partial \nu_t}{\partial \theta}$ and $\frac{\partial \vec{S}_t}{\partial \theta}$ are computed by propagating the derivatives of the parameters through the Kalman filter equations, with the derivatives of fixed matrix entries being zero.
\end{proposition}
\begin{proof}[Proof of \Cref{prop:fisher_kalman_canonical}]
The foundation of this proof is the prediction error decomposition provided by the Kalman filter. The likelihood of the full observation sequence $y_{1:T}$ can be factored as $p(y_{1:T}|\theta) = \prod_{t=1}^T p(y_t|y_{1:t-1}, \theta)$. For a linear Gaussian system, the Kalman filter establishes that each one-step-ahead conditional distribution is Gaussian:
$$ y_t | y_{1:t-1}, \theta \sim \mathcal{N}(C\hat{x}_{t|t-1}, S_t) $$
where $\hat{x}_{t|t-1} = \mathbb{E}[x_t | y_{1:t-1}, \theta]$ is the predicted state and $S_t = \text{Cov}(y_t | y_{1:t-1}, \theta) = CP_{t|t-1}C^\top + \Gamma$ is the predicted observation covariance. The log-likelihood of a single observation $y_t$ is:
$$ \log p(y_t|y_{1:t-1}, \theta) = -\frac{N_y}{2}\log(2\pi) - \frac{1}{2}\log|S_t| - \frac{1}{2}(y_t - C\hat{x}_{t|t-1})^\top S_t^{-1}(y_t - C\hat{x}_{t|t-1}). $$
Summing over all $T$ time steps gives the total log-likelihood $\mathcal{L}(\theta)$ as stated in the proposition. The Fisher information matrix is defined as $I(\theta) = -\mathbb{E}\left[ \nabla_\theta^2 \mathcal{L}(\theta) \right]$. To compute this, we first find the score function $s(\theta) = \nabla_\theta \mathcal{L}(\theta)$. Differentiating $\mathcal{L}(\theta)$ with respect to a single parameter $\theta_i$ involves the chain rule and standard matrix derivative identities ($\nabla_x \log|A| = \text{Tr}(A^{-1}\nabla_x A)$ and $\nabla_x A^{-1} = -A^{-1}(\nabla_x A)A^{-1}$):
$$ \frac{\partial\mathcal{L}}{\partial\theta_i} = -\frac{1}{2}\sum_{t=1}^T\left( \text{Tr}(S_t^{-1}\frac{\partial S_t}{\partial\theta_i}) - \nu_t^\top S_t^{-1}\frac{\partial S_t}{\partial\theta_i}S_t^{-1}\nu_t + 2\left(\frac{\partial\nu_t}{\partial\theta_i}\right)^\top S_t^{-1}\nu_t \right). $$
To find the $(i,j)$-th entry of the Hessian, we differentiate again with respect to $\theta_j$ and take the negative expectation. The key statistical properties of the innovations are that they are zero-mean, $\mathbb{E}[\nu_t] = 0$, and are white, i.e., $\mathbb{E}[\nu_t \nu_k^\top] = S_t \delta_{tk}$. Consequently, when taking the expectation of the Hessian, all terms that are linear in $\nu_t$ vanish. The remaining non-zero terms arise from products of derivatives. Specifically, we use $\mathbb{E}[\nu_t^\top M \nu_t] = \text{Tr}(M \mathbb{E}[\nu_t \nu_t^\top]) = \text{Tr}(M S_t)$.
The resulting $(i,j)$-th element of the FIM is:
$$ I(\theta)_{ij} = \mathbb{E}\left[\sum_{t=1}^T \left(\frac{\partial\nu_t}{\partial\theta_i}\right)^\top S_t^{-1}\frac{\partial\nu_t}{\partial\theta_j}\right] + \frac{1}{2}\mathbb{E}\left[\sum_{t=1}^T \text{Tr}\left(S_t^{-1}\frac{\partial S_t}{\partial\theta_i}S_t^{-1}\frac{\partial S_t}{\partial\theta_j}\right)\right]. $$
The second term can be written using the Kronecker product ($\otimes$) and vectorization ($\vec{\cdot}$), yielding the compact matrix form presented in the proposition.

For a system in a canonical form, the derivatives $\frac{\partial A}{\partial\theta_i}$, $\frac{\partial B}{\partial\theta_i}$, and $\frac{\partial C}{\partial\theta_i}$ are sparse matrices containing only a '1' or '-1' at the position of the parameter $\theta_i$, and zeros elsewhere. For instance, in controller form, if $\theta_i = a_k$, then $\frac{\partial A}{\partial a_k}$ is zero everywhere except for a '-1' at position $({d_x}, {d_x}-k+1)$. If $\theta_i = c_k$, then $\frac{\partial C}{\partial c_k}$ is zero everywhere except for a '1' at position $(1, k)$. These sparse derivatives are propagated through the recursive derivative equations for the Kalman filter states ($\hat{x}_{t|t}, P_{t|t}$) and predictions ($\hat{x}_{t|t-1}, P_{t|t-1}$), to compute the required gradients $\frac{\partial \nu_t}{\partial \theta}$ and $\frac{\partial S_t}{\partial \theta}$ at each time step. The initialization is handled by the derivatives of the initial state, $\frac{\partial\hat{x}_{0}}{\partial\theta_i}$ and $\frac{\partial P_{0}}{\partial\theta_i}$, which are typically assumed to be zero if the initial conditions are not functions of $\theta$.
\end{proof}

\begin{proposition}[Efficient FIM calculation for canonical forms via Kalman filter sensitivities]
\label{prop:fim_efficient_canonical}
Consider an LTI system in a canonical form with parameters $\theta$. The Fisher information matrix (FIM) is given by:
\begin{equation}
I(\theta)=\sum_{t=1}^T \mathbb{E}\left[\left(\frac{\partial \nu_t}{\partial \theta}\right)^{\top} S_t^{-1}\left(\frac{\partial \nu_t}{\partial \theta}\right)+\frac{1}{2}\left(\frac{\partial S_t}{\partial \theta}\right)^{\top}\left(S_t^{-1} \otimes S_t^{-1}\right)\left(\frac{\partial S_t}{\partial \theta}\right)\right]
\end{equation}
The efficiency in this calculation stems from specializing the general derivative recursions to the sparse structure of the canonical form matrices. The required sensitivities are computed as follows:

\paragraph{1. Controller canonical form}
Here, the parameter vector is $\theta = [a_1, \dots, a_{d_x}, c_1, \dots, c_{d_x}]^\top$. The derivatives $\frac{\partial A_c}{\partial a_k}$ and $\frac{\partial C_c}{\partial c_k}$ are sparse matrices (containing only one non-zero element). This simplifies the general derivative recursions significantly. For each parameter $\theta_i \in \theta$, we compute:
\begin{align}
\label{eq:dxdt_c}
\frac{\partial \hat{x}_{t|t-1}}{\partial \theta_i} &= A_c \frac{\partial \hat{x}_{t-1|t-1}}{\partial \theta_i} + \frac{\partial A_c}{\partial \theta_i} \hat{x}_{t-1|t-1} \\
\label{eq:dPdt_c}
\frac{\partial P_{t|t-1}}{\partial \theta_i} &= A_c \frac{\partial P_{t-1|t-1}}{\partial \theta_i} A_c^\top + \left(\frac{\partial A_c}{\partial \theta_i} P_{t-1|t-1} A_c^\top + A_c P_{t-1|t-1} \left(\frac{\partial A_c}{\partial \theta_i}\right)^\top\right) \\
\label{eq:dSdt_c}
\frac{\partial S_t}{\partial \theta_i} &= C_c \frac{\partial P_{t|t-1}}{\partial \theta_i} C_c^\top + \left(\frac{\partial C_c}{\partial \theta_i} P_{t|t-1} C_c^\top + C_c P_{t|t-1} \left(\frac{\partial C_c}{\partial \theta_i}\right)^\top\right) \\
\label{eq:dnudt_c}
\frac{\partial \nu_t}{\partial \theta_i} &= -C_c \frac{\partial \hat{x}_{t|t-1}}{\partial \theta_i} - \frac{\partial C_c}{\partial \theta_i} \hat{x}_{t|t-1}
\end{align}
These sensitivities are then used to compute the derivatives of the Kalman gain $K_t$ and the updated state and covariance, which are needed for the next time step's recursion.

\paragraph{2. Observer canonical form}
Here, the parameter vector is $\theta = [a_1, \dots, a_{d_x}, b_1, \dots, b_{d_x}]^\top$. The derivatives $\frac{\partial A_o}{\partial a_k}$ and $\frac{\partial B_o}{\partial b_k}$ are sparse. The efficient recursions become:
\begin{align}
\label{eq:dxdt_o}
\frac{\partial \hat{x}_{t|t-1}}{\partial \theta_i} &= A_o \frac{\partial \hat{x}_{t-1|t-1}}{\partial \theta_i} + \frac{\partial A_o}{\partial \theta_i} \hat{x}_{t-1|t-1} + \frac{\partial B_o}{\partial \theta_i} u_{t-1} \\
\label{eq:dPdt_o}
\frac{\partial P_{t|t-1}}{\partial \theta_i} &= A_o \frac{\partial P_{t-1|t-1}}{\partial \theta_i} A_o^\top + \left(\frac{\partial A_o}{\partial \theta_i} P_{t-1|t-1} A_o^\top + A_o P_{t-1|t-1} \left(\frac{\partial A_o}{\partial \theta_i}\right)^\top\right) \\
\label{eq:dSdt_o}
\frac{\partial S_t}{\partial \theta_i} &= C_o \frac{\partial P_{t|t-1}}{\partial \theta_i} C_o^\top \\
\label{eq:dnudt_o}
\frac{\partial \nu_t}{\partial \theta_i} &= -C_o \frac{\partial \hat{x}_{t|t-1}}{\partial \theta_i}
\end{align}
Note that derivatives with respect to $C_o$ are zero as it is fixed in this form.
\end{proposition}
\begin{proof}[Proof of \Cref{prop:fim_efficient_canonical}]
The derivation provides a computational procedure. The FIM formula is standard for Gaussian models and is derived from the negative expectation of the Hessian of the log-likelihood. The efficiency comes from exploiting the structure of the canonical forms when computing the required derivatives.

The general sensitivity equations for the Kalman filter are obtained by direct differentiation of the filter's time- and measurement-update equations. For any parameter $\theta_i$, we have:
\begin{align*}
\frac{\partial\hat{x}_{t|t-1}}{\partial\theta_i} &= \frac{\partial A}{\partial\theta_i}\hat{x}_{t-1|t-1} + A\frac{\partial\hat{x}_{t-1|t-1}}{\partial\theta_i} + \frac{\partial B}{\partial\theta_i}u_{t-1} \\
\frac{\partial P_{t|t-1}}{\partial\theta_i} &= \frac{\partial A}{\partial\theta_i}P_{t-1|t-1}A^\top + A\frac{\partial P_{t-1|t-1}}{\partial\theta_i}A^\top + AP_{t-1|t-1}\left(\frac{\partial A}{\partial\theta_i}\right)^\top + \frac{\partial \Sigma}{\partial\theta_i}
\end{align*}
And similarly for the remaining filter quantities ($\nu_t, S_t, K_t, \hat{x}_{t|t}, P_{t|t}$). The matrix derivatives (e.g., $\frac{\partial A_c}{\partial a_k}$) are extremely sparse, typically containing only a single non-zero entry.
\paragraph{Controller form:} For a parameter $\theta_i = a_k$, we have $\frac{\partial C_c}{\partial a_k} = \mathbf{0}$ and $\frac{\partial A_c}{\partial a_k}$ is a matrix of zeros with a '-1' at position $({d_x}, {d_x}-k+1)$. For a parameter $\theta_i = c_k$, we have $\frac{\partial A_c}{\partial c_k} = \mathbf{0}$ and $\frac{\partial C_c}{\partial c_k}$ is a row vector of zeros with a '1' at position $k$. Substituting these sparse matrices into the general sensitivity equations simplifies them to \eqref{eq:dxdt_c}-\eqref{eq:dnudt_c}. For example, in \eqref{eq:dxdt_c}, if $\theta_i = c_k$, the term $\frac{\partial A_c}{\partial c_k}$ is zero, so the equation simplifies to $\frac{\partial \hat{x}_{t|t-1}}{\partial c_k} = A_c \frac{\partial \hat{x}_{t-1|t-1}}{\partial c_k}$. This sparsity propagates through the recursions, making the computation highly efficient as many terms become zero.
\paragraph{Observer form:} The logic is identical. Here, $C_o$ is fixed, so $\frac{\partial C_o}{\partial \theta_i} = \mathbf{0}$ for all parameters. For $\theta_i = a_k$, the derivative $\frac{\partial A_o}{\partial a_k}$ is sparse, and for $\theta_i = b_k$, the derivative $\frac{\partial B_o}{\partial b_k}$ is sparse. Substituting these into the general equations yields the simplified set \eqref{eq:dxdt_o}-\eqref{eq:dnudt_o}. For example, the term involving $\frac{\partial C_o}{\partial \theta_i}$ vanishes from the derivative of $S_t$, leaving the much simpler \eqref{eq:dSdt_o}.

This specialization avoids unnecessary matrix operations with zero matrices, providing a clear and efficient algorithm. At each time step $t=1, \dots, T$, one runs the Kalman filter and, in parallel, the set of $N$ recursive derivative equations (where $N$ is the number of parameters), using the results to update the total FIM sum.
\end{proof}

\subsection{Expected posterior curvature}
\label{expected_posterior_curvature_section}
Distinct from the likelihood-based Fisher information matrix (FIM) $I_T(\theta)$, which depends solely on the likelihood and involves expectation over data $y_{[T]}$, one can assess the posterior's average concentration using the Expected Posterior Curvature matrix, $\mathcal{J}_T$. This metric captures the negative expected Hessian of the log-posterior (thus incorporating both likelihood and prior influence), averaged with respect to the posterior itself, offering a finite-sample perspective on uncertainty (see \Cref{def:expected_posterior_curvature}). While $\mathcal{J}_T$ reflects the average posterior shape for finite data $T$ and $I_T(\theta)$ the likelihood's curvature, they are related—the Bernstein-von Mises theorem (\Cref{thm:BvM_LTI_applicability}) links the asymptotic posterior shape to the limiting FIM $I(\theta_0)$. However, they measure distinct aspects of information and curvature, especially for finite $T$.
\begin{definition}[Expected posterior curvature matrix]
\label{def:expected_posterior_curvature} 
Consider the posterior distribution \(p(\Theta \mid y_{[T]}, u_{[T]})\) for parameters \(\Theta\). Assume the log-posterior is twice differentiable with respect to \(\Theta\) (treated as a flattened vector of free parameters), and that the expectations below exist. The \emph{Expected Posterior Curvature Matrix} is defined as the negative expectation of the Hessian of the log-posterior, where the expectation is taken with respect to the posterior distribution itself:
\[
  \mathcal{J}_T
  \;=\;
  -\,\E_{\Theta \mid Y_{[T]}, U_{[T]}} \! \left[
    \frac{\partial^2}{\partial\Theta \partial\Theta^\top}
    \log p(\Theta \mid y_{[T]}, u_{[T]})
  \right].
\]
Equivalently, under regularity conditions:
\[
  \mathcal{J}_T
  \;=\;
  \E_{\Theta \mid Y_{[T]}, U_{[T]}} \! \left[
    \left( \frac{\partial}{\partial\Theta} \log p(\Theta \mid y_{[T]}, u_{[T]}) \right) \!
    \left( \frac{\partial}{\partial\Theta} \log p(\Theta \mid y_{[T]}, u_{[T]}) \right)^{\!\top}
  \right].
\]
\(\mathcal{J}_T\) quantifies the average concentration of the posterior distribution \(p(\Theta \mid y_{[T]}, u_{[T]})\).
\end{definition}
\section{Canonical forms in LTI systems}
\label{appendix_a_canonical_forms}
In control theory, canonical forms provide standardized representations of LTI systems. Here we present all eight canonical forms divided into controllability and observability categories.
\subsection{Multi-input multi output (MIMO) systems}
\label{appendix_mimo_case}
\begin{definition}[MIMO canonical structure example]
\label{def:mimo_canonical_form}
A possible canonical structure for MIMO systems (\(d_u > 1\) or \(d_y > 1\)) uses a block companion form for \(A_c\), often parameterized by structural indices like the observability index \(r\). For instance, with state dimension \(d_x = r \cdot d_u\), one such identifiable structure is given by the matrices \(\{A_c, B_c, C_c, D_c\}\):
\begin{equation}
\label{eq:mimo_canonical_form}
A_{c}=\begin{pmatrix}
-\alpha_1 I_{d_u} & -\alpha_2 I_{d_u} & \cdots & -\alpha_{r-1} I_{d_u} & -\alpha_r I_{d_u} \\
I_{d_u} & 0 & \cdots & 0 & 0 \\
0 & I_{d_u} & \cdots & 0 & 0 \\
\vdots & \vdots & \ddots & \vdots & \vdots \\
0 & 0 & \cdots & I_{d_u} & 0
\end{pmatrix}, \;\; B_{c}=\begin{pmatrix}
I_{d_u} \\ 0 \\ 0 \\ \vdots \\ 0
\end{pmatrix},\;\; C_{c}=\begin{pmatrix} N_1 & \cdots & N_r \end{pmatrix}
\end{equation}
where \(I_{d_u}\) is the \(d_u \times d_u\) identity matrix, \(\alpha_i \in \R\) (\(i=1,\dots,r\)) are scalar parameters defining the characteristic polynomial blocks, the matrices \(N_k \in \R^{d_y \times d_u}\) (\(k=1,\dots,r\)) contain free parameters representing the system's numerator dynamics (Markov parameters), and the feedthrough matrix \(D_c = D \in \R^{d_y \times d_u}\) also consists of free parameters (unaffected by state basis transformations). The set of free parameters defining the system dynamics within this canonical form is explicitly:
\[
\Theta_{c,dyn} = \{\alpha_1, \dots, \alpha_r\} \cup \{ (N_k)_{ij} \}_{\substack{k=1..r \\ i=1..d_y \\ j=1..d_u}} \cup \{ (D_c)_{ij} \}_{\substack{i=1 \ldots d_y \\ j=1 \ldots d_u}}
\]
This specific dynamic structure contains \(N_{c,dyn} = r + (r \cdot d_y \cdot d_u) + (d_y \cdot d_u) = r + (r+1)d_y d_u\) free parameters. Using the relationship \(d_x = r \cdot d_u\), this count can be expressed as \(d_x/d_u + d_x d_y + d_y d_u\). The full canonical parameter set, \(\Theta_c\), encompasses both these dynamic parameters \(\Theta_{c,dyn}\) and the parameters describing the noise statistics. Assuming Gaussian noise (\Cref{ass:gaussian_noise}) as discussed in \Cref{sec:defs}, the noise parameters consist of the unique elements of the lower (or upper) triangular Cholesky factors. Therefore, the complete set of parameters to be inferred is $
\Theta_c = \Theta_{c,dyn} \cup \{ L_\Sigma, L_\Gamma \}$. The noise components \(L_\Sigma\) and \(L_\Gamma\) contribute an additional \(d_x(d_x+1)/2 + d_y(d_y+1)/2\) parameters to the total count in \(\Theta_c\).
\end{definition}

\subsection{Controllability SISO canonical forms}~\\
\begin{minipage}[t]{0.48\textwidth}
\paragraph{(1) Controller form}
\[
A_{\mathrm{ctrl}}
=
\begin{bmatrix}
0 & \cdots & 0 & -p_{0}\\
1 & \ddots &   & \vdots \\
\vdots & \ddots & \ddots &  \\
0 & \cdots & 1 & -p_{{d_x}-1}
\end{bmatrix}
\]
\[
b_{\mathrm{ctrl}} = [\,1 \;\; 0 \;\;\cdots\; 0\,].
\]
\end{minipage}
\hfill
\begin{minipage}[t]{0.48\textwidth}
\paragraph{(2) Dual controller form}
\[
A_{\mathrm{ctrl\text{-}dual}}
=
\begin{bmatrix}
0 & 1 & \cdots & 0\\
\vdots & \ddots & \ddots & \vdots\\
\vdots &        & \ddots & 1\\
-p_{0} & \cdots & \cdots & -p_{{d_x}-1}
\end{bmatrix}
\]
\[
b_{\mathrm{ctrl\text{-}dual}} = [\,0 \;\;\cdots\;\;0\;\; 1\,].
\]
\end{minipage}

\vspace{1em}

\noindent
\begin{minipage}[t]{0.48\textwidth}
\paragraph{(3) Observable-style controller}
\[
A_{\mathrm{ctrl\text{-}obs}}
=
\begin{bmatrix}
-p_{{d_x}-1} & 1 & \cdots & 0\\
\vdots & \ddots & \ddots & \vdots\\
\vdots & & & 1\\
-p_{0} & 0 & \cdots & 0
\end{bmatrix}
\]
\[
b_{\mathrm{ctrl\text{-}obs}} = [\,0 \;\;\cdots\;\;0\;\; 1\,].
\]
\end{minipage}
\hfill
\begin{minipage}[t]{0.48\textwidth}
\paragraph{(4) Dual observable-style controller}
\[
A_{\mathrm{ctrl\text{-}obs\text{-}dual}}
=
\begin{bmatrix}
-p_{{d_x}-1} & \cdots &\cdots &-p_{0}\\
1 &  & & 0\\
\vdots & \ddots   & \ddots& \vdots\\
0 & \cdots & 1 & 0
\end{bmatrix}
\]
\[
b_{\mathrm{ctrl\text{-}obs\text{-}dual}} = [\,1 \;\; 0 \;\;\cdots\; 0\,].
\]
\end{minipage}
\subsection{Observability SISO canonical forms}~\\
\begin{minipage}[t]{0.48\textwidth}
\paragraph{(5) Observer Form}
\[
A_{\mathrm{obs}}
=
\begin{bmatrix}
0 & 1 & \cdots & 0\\
\vdots & \ddots & \ddots & \vdots\\
0 &  & & 1\\
-p_{0} & \cdots & \cdots & -p_{{d_x}-1}
\end{bmatrix}
\]
\[
c_{\mathrm{obs}} = [\,1 \;\; 0 \;\;\cdots\; 0\,].
\]
\end{minipage}
\hfill
\begin{minipage}[t]{0.48\textwidth}
\paragraph{(6) Dual observer form}
\[
A_{\mathrm{obs\text{-}dual}}
=
\begin{bmatrix}
0 & \cdots & 0 & -p_{0}\\
1 & \ddots &   & \vdots\\
\vdots &     & \ddots & \vdots\\
0 & \cdots & 1 & -p_{{d_x}-1}
\end{bmatrix}
\]
\[
c_{\mathrm{obs\text{-}dual}} = [\,0 \;\;\cdots\;\;0\;\; 1\,].
\]
\end{minipage}
\hfill
\begin{minipage}[t]{0.48\textwidth}
\paragraph{(7) Controller-style observer}
\[
A_{\mathrm{obs\text{-}ctrl}}
=
\begin{bmatrix}
-p_{{d_x}-1} & 1 & \cdots & 0\\
\vdots & \ddots & \ddots & \vdots\\
\vdots& & & 1\\
-p_{0} & 0 & \cdots & 0
\end{bmatrix}
\]
\[
c_{\mathrm{obs\text{-}ctrl}} = [\,1 \;\; 0 \;\;\cdots\; 0\,].
\]
\end{minipage}
\hfill
\begin{minipage}[t]{0.48\textwidth}
\paragraph{(8) Dual controller-style observer}
\[
A_{\mathrm{obs\text{-}ctrl\text{-}dual}}
=
\begin{bmatrix}
-p_{{d_x}-1} & \cdots & & -p_{0}\\
1 & \ddots & & 0\\
\vdots &   & & \vdots\\
0 & \cdots & 1& 0
\end{bmatrix}
\]
\[
c_{\mathrm{obs\text{-}ctrl\text{-}dual}} = [\,0 \;\;\cdots\;\; 0\;\;1\,].
\]
\end{minipage}
\bigskip

\section{Procedure for generating well-conditioned LTI systems}
\label{appendix_e}
To generate stable, high-dimensional, and well-conditioned test systems, we employ a systematic procedure instead of sampling directly from our model's prior. While sampling from the prior is a principled Bayesian approach, our deterministic method avoids any appearance of ``cherry-picking'' or creating an ``inverse crime,'' ensuring the test systems provide a challenging and objective benchmark. The procedure for a given state dimension \({d_x}\) begins by defining an initial discrete-time LTI system pair \((A, B_0)\) to be stable and controllable. This is done by choosing the dynamics matrix \(A \in \mathbb{R}^{{d_x} \times {d_x}}\) to be diagonal, \(A = \operatorname{diag}(\lambda_1, \dots, \lambda_{d_x})\), with \({d_x}\) distinct eigenvalues \(\{\lambda_i\}_{i=1}^{d_x}\) linearly spaced on the real interval \([-0.98, 0.9]\), and setting the input matrix to a column vector of ones, \(B_0 = [1, \dots, 1]^\top \in \mathbb{R}^{{d_x} \times 1}\). Next, to normalize the system's input-to-state mapping, we compute the controllability Gramian, \(W_{c0}\), as the unique symmetric positive definite solution to the discrete-time algebraic Lyapunov equation \( W_{c0} = A W_{c0} A^\top + B_0 B_0^\top \). A state transformation matrix \(T = W_{c0}^{1/2}\) is then defined as the unique symmetric positive definite square root of this Gramian. Applying this similarity transformation yields a new system realization \((\tilde{A}, \tilde{B})\), where \( \tilde{A} = T^{-1} A T \) and \( \tilde{B} = T^{-1} B_0 \), whose controllability Gramian is the identity matrix \(\tilde{W}_c = I_{d_x}\). Finally, to make the overall system nearly balanced, an output matrix \(C \in \mathbb{R}^{1 \times {d_x}}\) is constructed such that its observability Gramian, \(W_o\), which solves \(W_o = \tilde{A}^\top W_o \tilde{A} + C^\top C\), is as close to the identity matrix as possible. This is achieved using the ``uniform proportions method,'' where \(C\) is constructed from the left eigenvectors of \(\tilde{A}\). Specifically, given the eigendecomposition \(\tilde{A} = V \Lambda V^{-1}\), the output matrix is set proportional to the sum of the rows of \(V^{-1}\), i.e., \(C \propto \mathbf{1}^\top V^{-1}\). This construction ensures that \(C\) has a significant projection onto every system mode, making observability uniform. The resulting system \((\tilde{A}, \tilde{B}, C)\) is thus stable, fully controllable, observable, and approximately balanced, making it an excellent benchmark for numerical methods.
\end{document}